\documentclass[10pt]{article} 
\usepackage[preprint]{tmlr}

\usepackage{pifont} 
\usepackage{amsmath} 

\usepackage{amsmath,amsfonts,bm}

\def\eqref#1{equation~\ref{#1}}

\def\1{\bm{1}}

\DeclareMathAlphabet{\mathsfit}{\encodingdefault}{\sfdefault}{m}{sl}
\SetMathAlphabet{\mathsfit}{bold}{\encodingdefault}{\sfdefault}{bx}{n}

\usepackage{url}

\usepackage[utf8]{inputenc} 
\usepackage[T1]{fontenc}    
\usepackage{hyperref}       
\usepackage{url}            
\usepackage{booktabs}       
\usepackage{amsfonts}       
\usepackage{svg}           
\usepackage{nicefrac}       

\usepackage{microtype}      
\usepackage{xcolor}         
\usepackage{subcaption}

\definecolor{lightgraybg}{RGB}{245,245,245}
\usepackage[most]{tcolorbox}

\usepackage{multirow}
\usepackage{comment} 
\usepackage{enumitem}
\usepackage[frozencache,cachedir=minted-cache]{minted}
\usepackage{amsthm}
\newtheorem{theorem}{Theorem}

\definecolor{Gray}{gray}{0.9}

\usepackage{listings}
\usepackage{xcolor}

\usepackage{mathtools}
\usepackage{commath}
\usepackage[ruled,vlined]{algorithm2e}

\newcommand{\Heaviside}[1]{\ensuremath{u\left(#1\right)}}  

\definecolor{battleshipgrey}{rgb}{0.52, 0.52, 0.51}

\usepackage{xspace}
\newcommand{\method}{{\fontfamily{lmtt}\selectfont\emph{\textsc{URDP}}}\xspace}

\title{Uncertainty-aware Reward Design Process}

\author{\name Yang Yang \email yangyang2025@ia.ac.cn \\
      \addr $C^2DL$ , Institute of Automation, Chinese Academy of Science
      \AND
      \name Xiaolu Zhou \email 202321130108@mail.bnu.edu.cn \\
      \addr Beijing Normal University
      \AND
      \name Bosong Ding  \email B.Ding\_3@tilburguniversity.edu \\
      \addr Tilburg University 
      \AND
      \name Miao Xin\thanks{Corresponding authors.} \email miao.xin@ia.ac.cn\\
      \addr $C^2DL$ , Institute of Automation, Chinese Academy of Science }

\def\month{MM}  
\def\year{YYYY} 
\def\openreview{\url{https://openreview.net/forum?id=XXXX}} 

\begin{document}

\maketitle

\begin{abstract}

Designing effective reward functions is a cornerstone of reinforcement learning (RL), yet it remains a challenging process due to the inefficiencies and inconsistencies inherent in conventional reward engineering methodologies.  
Recent advances have explored leveraging large language models (LLMs) to automate reward function design. However, their suboptimal performance in numerical optimization often yields unsatisfactory reward quality, while the evolutionary search paradigm demonstrates inefficient utilization of simulation resources, resulting in prohibitively lengthy design cycles with disproportionate computational overhead. 
To address these challenges, we propose the Uncertainty-aware Reward Design Process (\method), a novel framework that integrates large language models to streamline reward function design and evaluation in RL environments. 
\method quantifies candidate reward function uncertainty based on the self-consistency analysis, enabling simulation-free identification of ineffective reward components while discovering novel reward components. 
Furthermore, we introduce uncertainty-aware Bayesian optimization (UABO), which incorporates uncertainty estimation to significantly enhance hyperparameter configuration efficiency. 
Finally, we construct a bi-level optimization architecture by decoupling the reward component optimization and the hyperparameter tuning. \method orchestrates synergistic collaboration between the reward logic reasoning of the LLMs and the numerical optimization strengths of the Bayesian Optimization. 
We conduct a comprehensive evaluation of \method across 35 diverse tasks spanning three benchmark environments: IsaacGym, Bidexterous Manipulation, and ManiSkill2. Our experimental results demonstrate that \method not only generates higher-quality reward functions but also achieves significant improvements in the efficiency of automated reward design compared to existing approaches. 

\end{abstract}

\section{Introduction}
\label{sect:introduction}

\textcolor{black}{
In reinforcement learning (RL), the design of reward functions serves as a pivotal determinant for successfully training agents in sequential decision-making tasks. These rewards guide the learning process by shaping agent behaviors to accomplish complex objectives across diverse environments. While conventional approaches such as reward engineering and inverse reinforcement learning (IRL)~\cite{arora2021survey} established early research paradigms, they remain fundamentally constrained by their reliance on human expertise and the availability of high-quality demonstration data, particularly in domains like robotic skill acquisition~\cite{zitkovich2023rt}. Recent advancements in large language models (LLMs)~\cite{Radford2019language,brown2020language, liu2024deepseek} have demonstrated remarkable capabilities in natural language understanding, code generation, and contextual optimization, thereby introducing a novel paradigm for automated reward function design. }

However, current automated reward design methodologies based on LLMs present two fundamental challenges~\cite{cao2024survey}. 
First, the \textbf{efficiency} of reward function design remains suboptimal. Existing approaches rely heavily on simulation-based training processes~\cite{ma2024eureka} to evaluate reward function efficacy, which involves extensive and often redundant evaluations, leading to significant computational overhead without commensurate benefits.
Second, the \textbf{performance} of LLM-generated reward functions frequently falls short of expectations. During the optimization process, LLMs fail to fully leverage their reasoning capabilities, resulting in reward functions that inadequately capture the intended task objectives.
Given that the design efficiency and generation quality of reward functions are closely tied to the speed and effectiveness of policy learning, a key research question emerges: \textit{How can we enhance the efficiency of obtaining high-quality reward functions?} 

In this paper, we introduce the \textbf{Uncertainty-aware Reward Design Process} (\method), a novel framework for automated reward function generation. 
First, we propose a method to quantify the uncertainty of generated samples, which enables the selective elimination of redundant reward function sampling and simulation. This approach is grounded in a key observation regarding self-consistency~\cite{wang2022COT-SC}: LLMs exhibit higher output consistency when handling well-defined tasks, allowing for more efficient sampling strategies. 
Second, we identify a critical limitation in current LLM-based evolutionary search approaches, i.e., their suboptimal performance in numerical optimization. To address this, we decouple reward component formulation from reward intensity optimization, delegating the latter to a dedicated numerical optimization module. 
Specifically, we propose a novel Bayesian optimization~\cite{snoek2012practical} approach incorporating uncertainty distribution priors, which significantly accelerates convergence in this black-box optimization task. 
Experimental results demonstrate that \method surpasses state-of-the-art methods in both reward quality and computational efficiency, establishing a new benchmark for automated reward design. 

Overall, our contributions are summarized as follows: (1) We propose \method, a novel framework for automated reward function design in reinforcement learning. The framework employs an alternating bi-level optimization process that decouples reward component design from hyperparameter optimization, thereby significantly enhancing reward function performance. (2) We introduce a self-consistency-based reward uncertainty quantification method for the reward component design process. This approach not only dramatically improves the efficiency of reward component validation but also facilitates the discovery of novel reward logic. (3) We present an Uncertainty-aware Bayesian optimization algorithm that substantially increases the efficiency of hyperparameter search. (4) Our comprehensive evaluation across 35 tasks spanning 3 distinct benchmarks demonstrates that \method consistently outperforms existing methods in both reward function generation efficiency and final policy performance, as evidenced by rigorous quantitative analysis.

\section{Related work}
\label{related_work}

\textcolor{black}{
\textbf{Reward code generation.}
The automation of reward code generation has been a critical area of research, aiming to simplify and improve the process of defining task-specific reward functions for RL. 
As the dual formulation of RL problems, inverse reinforcement learning (IRL) methods~\cite{arora2021survey} have been extensively investigated for reward function acquisition. However, these approaches are fundamentally limited by their dependence on demonstration data, which severely constrains their scalability in practical applications. 
Recently, LLM-based reward design methodologies~\cite{kwon2023reward, yu2023language, ma2024eureka, xie2024text2reward} have demonstrated promising potential, offering a paradigm shift in automated reward function development. 
L2R~\cite{yu2023language} introduced a two-stage LLM-prompting framework to generate templated rewards, bridging high-level language instructions with low-level robot actions. 
Eureka~\cite{ma2024eureka} leveraged the zero-shot and in-context learning capabilities of advanced LLMs to perform evolutionary optimization over reward code. This method demonstrated the potential of LLMs to generate rewards without task-specific prompting or predefined templates, enabling agents to acquire complex skills via RL.
Text2Reward~\cite{xie2024text2reward} extended this line of work by generating shaped, dense reward functions as executable programs grounded in compact environment representations. Unlike sparse reward codes or constant reward functions, Text2Reward produces interpretable, dense reward codes capable of iterative refinement with human feedback. 
Despite their success, these methods overlook two key limitations: (1) insufficient reasoning capability for reward logic derivation, and (2) inadequate exploration of novel reward components. Our proposed framework decouples confounding factors in reward design to reduce cognitive load and incorporates uncertainty quantification to guide LLMs toward more focused analysis and refinement of reward logic relevance. }

\textbf{Hybrid optimization.} Recent advances in large language models (LLMs)~\cite{openai2023gpt4, liu2024deepseek} have demonstrated remarkable progress in text-based complex reasoning tasks~\cite{li2025system}. Through techniques such as self-improvement~\cite{song2023self}, multi-path reasoning~\cite{wan2024alphazero}, and reward modeling~\cite{zhong2025comprehensive}, LLMs exhibit substantial potential in contextual comprehension~\cite{openai2023gpt4}, code generation~\cite{yu2024outcome}, and task planning~\cite{hao2023reasoning}. However, their capabilities in deep logical reasoning~\cite{cheng2025empowering}, particularly in mathematical and numerical optimization domains~\cite{yan2025phd}, remain underexplored, with significant performance gaps persisting. Several studies attempt to employ LLMs as meta-optimizers for diverse optimization problems~\cite{yanglarge}. Yet, due to their inherently discrete representation nature, LLMs' effectiveness in high-dimensional, continuous numerical reasoning tasks requires further investigation~\cite{assran2025v}. In reinforcement learning, the reward design inherently involves multiple types of optimization problems. Unlike completely LLM-based evolutionary search approaches~\cite{ma2024dreureka, xie2024text2reward}, our work introduces black-box numerical optimization tools to compensate for LLMs' limitations in continuous numerical optimization.

\noindent\textbf{Uncertainty quantification.}
Uncertainty quantification (UQ) serves as a fundamental component for reliable automated decision-making and has been extensively studied in domains such as Bayesian inference~\cite{gal2015dropout,foong2020on}. Recent advances have investigated uncertainty quantification in black-box language models \cite{liu2025uncertainty, geng2024survey, kuhn2023semantic, lingenerating}, yielding various approaches including token-level entropy methods~\cite{kadavath2022language}, conformal prediction-based techniques~\cite{su2024api}, and consistency-based frameworks~\cite{lingenerating}. While existing methods primarily leverage uncertainty estimation to enhance LLM interpretability~\cite{ahdritz2024distinguishing} and mitigate hallucination risks \cite{shorinwa2024survey, mohri2024language}, our methodology not only actively quantifies prediction uncertainty in LLMs but also strategically leverages this uncertainty as the foundational mechanism for both simulation-free reward function design and efficient numerical optimization. Furthermore, we identify and characterize a significant correlation between reward component uncertainty and the discovery of novel reward formulations. This targeted uncertainty quantification framework represents a substantive methodological contribution that enables more effective and reliable reward design for reinforcement learning tasks.

\section{Preliminary: problem setup and notations}
\label{sect:preliminary}

\textcolor{black}{
\textbf{Reinforcement learning (RL)} tasks can be modeled as Markov Decision Processes (MDPs) defined by the tuple $\langle \mathcal{S}, \mathcal{A}, P, R, \gamma \rangle$, where $\mathcal{S}$ is the state space, $\mathcal{A}$ is the action space, $P(s'|s, a)$ is the transition probability, $R(s, a)$ is the reward function, and $\gamma$ is the discount factor. }  

\textcolor{black}{
\textbf{Reward design problem (RDP)}~\cite{singh2009rewards}. Despite its superior interpretability, designing rewards \( R(s, a) \) represented in code form is critical to aligning agent behavior with task objectives but involves challenges such as managing redundancies, inconsistencies, and uncertainties during reward function generation. 
Following the previous works~\cite{song2023self, ma2024eureka, xie2024text2reward}, the reward ($R$) is represented as a function of the reward components ($\mathbf{r}$) and the reward intensity ($\mathbf{\theta}$)
\begin{equation}
    R = f(\mathbf{r}, \mathbf{\theta}).  
\end{equation} 
Conventionally, evaluating the quality of a reward function candidate necessitates computationally expensive simulations (RL training). In this work, we aim to simultaneously maximize reward function performance while minimizing the associated simulation cost throughout the automatic reward design process. }

\textcolor{black}{
\textbf{Bayesian optimization (BO)} is a sequential design strategy for global optimization of black-box functions~\cite{shahriari2015taking, snoek2012practical}. 
Given a black-box function $f: \mathbb{X} \rightarrow \mathbb{R} $, Bayesian optimization aims to find an input $\mathbf{x}^* \in argmin_{x \in \mathbb{X}}f(\mathbf{x})$  that globally minimizes $f$.  
It places a prior $p(f)$ over the objective function $f$ to form a surrogate model (usually a Gaussian process~\cite{wang2024pre}). 
An \textit{acquisition function} (e.g., the Expected Improvement (EI)~\cite{ament2023unexpected}) $a_{p(f)} : \mathbb{X} \rightarrow \mathbb{R }$ strategically determines the direction of the search for sampling points. 
The algorithm iteration proceeds in the following three steps: (1) find the most promising $\mathbf{x}_{n+1} \sim argmax_{a_p(\mathbf{x})}$; (2) evaluate the function $y_{n+1} \sim f(\mathbf{x}_{n+1})+\mathcal{N}(0,\sigma^2)$ and update the set of historical observations $\mathcal{D}_n = (x_j , y_j)_{j=1 \dots n}$ by adding the point $(\mathbf{x}_{n+1}, y_{n+1})$, and (3)  update $p(f|\mathcal{D}_{n+1})$ and $a_{p(f|\mathcal{D}_{n+1})}$. } 

\section{Methods}
\label{methods}

\textcolor{black}{
The Uncertainty-aware Reward Design Process (\method) framework incorporates three fundamental elements: (1) Decoupling of the reward component and reward intensity design processes, (2) Reward component generation based on uncertainty quantification, and (3) Uncertainty-aware Bayesian optimization. 
Collectively, \method enhances reward design efficiency by minimizing redundant simulations while improving the quality of generated reward functions through the integration of numerical optimization techniques within the decoupled optimization framework. }

\begin{figure}[htb]
    \centering
    \includegraphics[width=1\linewidth]{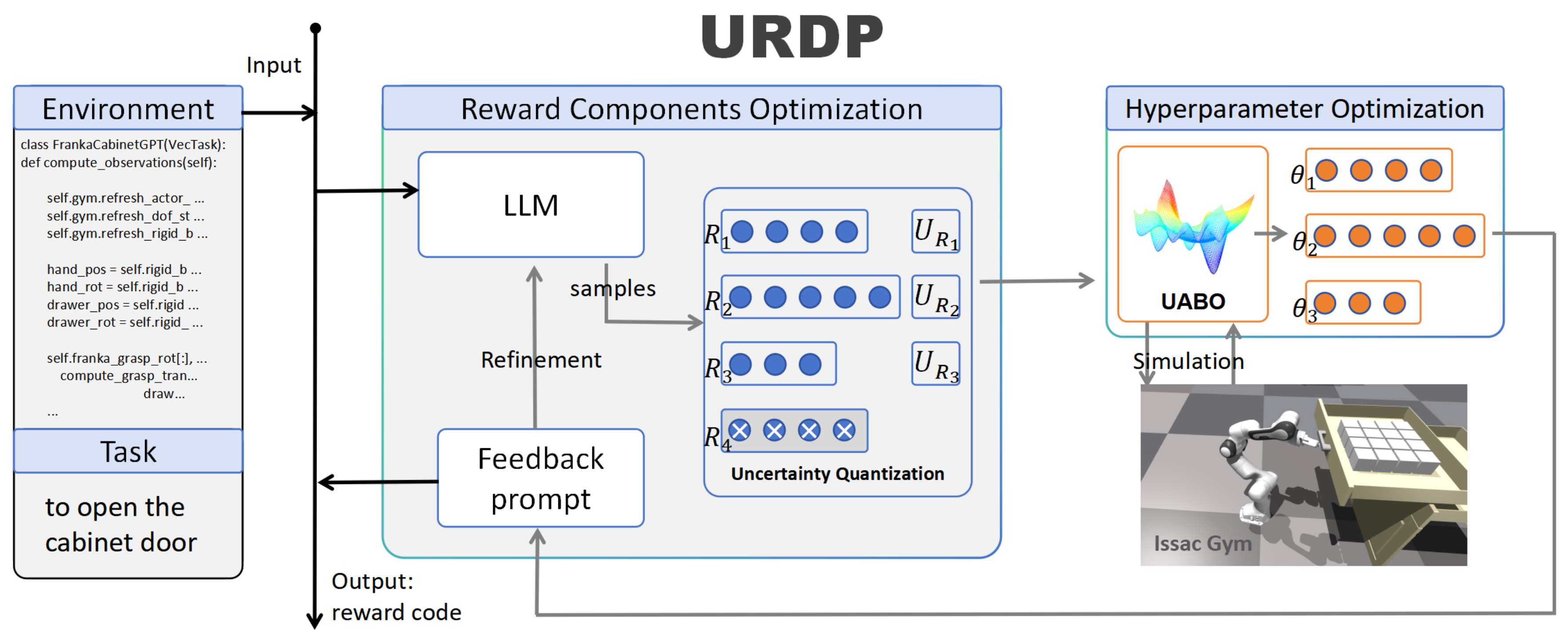}
    \caption{\textcolor{black}{\method implements an alternating bi-level iterative optimization framework for automated reward design problems (RDP). The outer-loop optimization employs LLMs to refine reward components, where uncertainty quantification significantly enhances sampling efficiency. Concurrently, the inner-loop optimization utilizes Uncertainty-Aware Bayesian Optimization (UABO) to determine optimal hyperparameter configurations for the reward components. The decoupled architecture strategically leverages the complementary strengths of LLMs in conceptual reward design and numerical optimization tools in precise parameter tuning, achieving synergistic improvements in both final policy performance and computational efficiency.}
    }
    \label{fig:method}
\end{figure}

\subsection{Decoupled Reward Generation and Hyperparameter Optimization}
\label{sect:decoupling}

Large language models (LLMs) possess extensive commonsense knowledge about task rewards, enabling them to surpass human-level performance in designing reward components~\cite{yu2023language, ma2024eureka, xie2024text2reward}. However, their capability in black-box numerical optimization remains inferior to specialized numerical optimization algorithms (see Section \ref{sect:ablation} Abl-3 for an ablation study). Consequently, existing methods \cite{ma2024eureka} that conflate the optimization of reward components with reward intensities not only yield suboptimal numerical optimization results but also lead to insufficient attention being paid to the optimization of reward components in agent learning. 

\textbf{\method framework}. Building upon these observations, we proposes a decoupled reward function design process. As illustrated in Figure~\ref{fig:method}, \method implements a bi-level iterative optimization procedure where, given environment specifications and task descriptions, the agent first samples multiple reward functions from the LLM in the outer loop, ranks them using uncertainty quantification metrics, and filters out redundant and potentially unreliable rewards through a process termed \textit{Reward Code Sampling with Uncertainty Screening}. Subsequently, the agent invokes numerical optimization tools in the inner loop to determine optimal reward intensity hyperparameters for the current reward configuration through black-box optimization and simulation-based evaluation. Finally, the agent evaluates the feasibility of current reward components and provides improvement feedback to refine the reward components. In essence, the outer loop optimizes the reward components and the reward logic while the inner loop tunes reward intensity hyperparameters, with their alternating optimization progressively converging toward optimal reward functions. See Alg.~\ref{alg: method} for pseudocode.

\begin{algorithm}[h]
\caption{Uncertainty-aware Reward Design Process}\label{alg: method}
\KwIn{Task description \( T \), Environment code \( E \), LLM \( \mathcal{L} \).}
\KwOut{Optimized reward function \( R^* \).}
\ForEach{iteration $n \in N_{outer}$ }{
    Generate $K$ reward component samples \( \{r_{i,1}, r_{i,2}, \dots, r_{i,m}\}_{i \in K} \) using \( \mathcal{L}(T, E, prompt) \) \\
    Uncertainty Quantization: \( \{U(r_{i,1}), \dots, U(r_{i,m}) \} \) and \( U(R_i) \) \\
    Filter out redundancies and reserve $R_{i \in K^*}$ \\
    \ForEach{\( R_i \), \( t=1,2, \dots N_{inner}/U(R_i) \) }{
        Fit probabilistic model for $f( \{\theta_{i}\})$ on data $D_{i,t-1}$ \\ 
        Choose $\{\theta_{i}\}_t$ by maximizing the acquisition function \( uEI(\{\theta_{i}\}, \{U(r_i)\} ) \) \\ 
        Evaluate by simulation training \( y_{i,t} = f(\{\theta_{i}\}_t) \)   \\
        Augment the data $D_{i,t} = D_{i,t-1} \cup (\{\theta_{i}\}_t, y_{i,t}) $ \\ 
        Choose incumbent \( \{\theta_i^*\} \gets  argmax\{ y_{i,1}, \dots, y_{i,t} \} \) and \( y_i^*  \gets max\{ y_{i,1}, \dots, y_{i,t} \} \)
    }
    Refine \( prompt \) for the reward components
}
Choose optimal \( \{(r^*, \theta^*)\}  \gets argmax\{ y_1^*, \dots,  y_{K^*}^*\} \) \\
Recombine \( \{(r^*, \theta^*)\} \) into \( R^*(s, a) \) \\
\Return \( R^* \).
\end{algorithm}

\subsection{Reward Code Sampling with Uncertainty Screening}
\label{sect:code_sampling}

In LLM-based RDP, automated reward function design can be achieved through iterative sampling and simulation. However, existing approaches indiscriminately conduct simulation-based evaluation on all LLM-generated reward function samples, despite the frequent presence of redundant or infeasible candidates within these samples. Our analysis identifies this as a critical factor contributing to computational inefficiency (see results in Section~\ref{sect:ablation} Abl-1). Consequently, a pivotal question emerges: \textit{how to effectively filter out potentially problematic reward function samples prior to simulations}, thereby avoiding computationally expensive yet futile simulation training.

\textbf{Uncertainty priority}. Our sampling approach is grounded in the principle of \textit{self-consistency}~\cite{wang2022COT-SC} in LLMs. When explicitly prompted to generate diverse outputs, LLMs that produce highly consistent responses demonstrate well-internalized, task-specific knowledge. Such outputs exhibit high reliability and typically require minimal refinement. Conversely, if the generated results show substantial diversity, this indicates uncertainty in the LLMs' understanding of the task context and underlying concepts. These divergent results exhibit lower reliability and consequently demand more refinement.

\textbf{Sampling and filtering}. 
Based on this principle, the agent prompts the LLM to generate diverse reward components for a given RL task. The uncertainty score of each reward component $r_i$, denoted as $U(r_i)$, is quantified by its occurrence frequency across all sampled candidates. 
To quantify the uncertainty of the reward component, \method identifies and resolves ambiguities in reward components using LLMs. It combines textual similarity and semantic similarity analyses to evaluate the relevance and clarity of reward components, assigning $U(r_i)$ to each reward component $r_i$ as
\begin{equation}
    U(r_i) = 1 - \sum\limits_{i\in [1,K]}(\Heaviside{ max(S_{\text{text}}(r_i), S_{\text{semantic}}(r_i)) - \omega} )/K,
\end{equation}
where $\Heaviside{\,\cdot\,}$ is a step function, $K$ denotes the quantity of the reward samples, $\omega$ is a decision parameter regarding the maximum similarity ($\omega = 0.95$), $S_{\text{text}} \in (0,1]$ and $ S_{\text{semantic}} \in  (0,1]$ are the textual and semantic similarity scores, respectively. 
Furthermore, the normalized sample uncertainty score $U(R_k)$ is computed to evaluate the overall uncertainty of each reward function sample $R_k$. 
Using $U(R_k)$, the agent filters out samples containing identical reward components, thereby eliminating redundant inner-loop optimization processes that would otherwise incur unnecessary computational overhead. 
Implementation details are elaborated in App.~\ref{app:impl_sampling}. 
Moreover, our analysis reveals that highly uncertain reward components may contain unexplored components capable of facilitating effective reward shaping (see Section~\ref{sect:discussion} Disc-2). 
Consequently, the agent implements an adaptive exploration-exploitation strategy. 
For high-uncertainty samples, it allocates additional inner-loop iterations to prioritize exploration of optimal hyperparameter configurations for potentially novel reward components.
For low-uncertainty samples, it emphasizes exploitation to minimize unnecessary simulations. This approach balances the trade-off between the exploration and utilization of uncertain reward components.  

\subsection{Uncertainty-aware Bayesian Optimization}
\label{sect:UABO}

\textcolor{black}{
While large language models demonstrate significant potential for reward component design, they exhibit suboptimal performance in numerical optimization tasks (see results in Section~\ref{sect:ablation} Abl-3). This limitation leads to non-optimal reward intensity configurations in LLM-generated reward functions, representing a key factor in the poor policy learning performance observed in prior approaches. In contrast to existing methods, our \method framework does not rely on LLM-based agents for direct numerical optimization. Instead, the agent serves as a controller that orchestrates specialized numerical optimization tools. Specifically, \method delegates the inner-loop optimization of reward intensity parameters to Bayesian optimization (BO) algorithms. Benefiting from BO's superiority in black-box global optimization, this approach achieves significantly better performance than LLM-based optimization.
}

\textcolor{black}{
Although the classical Bayesian optimization algorithm, \textit{i.e.}, Gaussian Process with Expected Improvement (EI)~\cite{ament2023unexpected, snoek2012practical}, demonstrates theoretical advantages, its practical efficiency remains unsatisfactory, particularly due to the substantial computational overhead incurred by acquisition functions during sampling (simulation training). This inefficiency frequently prevents convergence to globally optimal solutions within an acceptable number of samplings. Notably, the uncertainties $U(r_i)$ of individual reward components imply valuable prior knowledge for enhancing BO's sampling efficiency. Given a reward function comprising $m$ components, we model the $m$ reward intensities as a joint probability distribution. Crucially, higher uncertainty in $r_i$ corresponds to a more uniform marginal distribution along that dimension. This observation suggests that sampling should prioritize exploitation over exploration in high-uncertainty dimensions. Building upon this smoothness assumption, we propose \textbf{Uncertainty-aware Bayesian optimization (UABO)} to address these limitations. }

\textbf{UABO} incorporates reward component uncertainty scores, $U(r)$, into both the kernel function and acquisition function of the standard Bayesian optimization. 
The Matern kernel in Gaussian process has the form
\begin{equation}\label{matern2.5}
    k(p,p^{\prime})= f_{\nu}(d) = \sigma^2 \cdot \frac{2^{1-\nu}}{\Gamma(\nu)} \left( \frac{\sqrt{2\nu} d}{\ell} \right)^\nu K_\nu \left( \frac{\sqrt{2\nu} d}{\ell} \right), 
\end{equation}
where $d$ is the Euclidean distance between $p$ and $p^{\prime}$, $\sigma^2$ is the variance, $\nu$ is the smoothness parameter, $\ell$ is the length scale parameter and $K_\nu$ is the modified Bessel function of the second kind.
We note that the kernel is isotropic, which means that all dimensions (i.e., the intensity parameters of reward components) share the same length scale parameter. To accommodate heterogeneous smoothness (uncertainty score $U(r_i)$) across different dimensions, we propose an anisotropic kernel function that incorporates uncertainty values as length scales within the distance metric. The distance is formulated as follows, 
\begin{equation}\label{dist-new}
    d_{\text{u}}(p,p^{\prime}) =  \sqrt{\left(\frac{x_1-x_1^{\prime}}{U(r_{i, 1})}\right)^2 +\cdots+ \left( \frac{x_n-x_n^{\prime}}{U(r_{i, m})} \right)^2}.
\end{equation}
Then the new kernel function is defined as
\begin{equation}\label{kernel-new}
    \Tilde{k}(p,p^{\prime})= f_{\nu}(d_{\text{u}}) = \sigma^2 \cdot \frac{2^{1-\nu}}{\Gamma(\nu)} \left( \frac{\sqrt{2\nu} d_{\text{u}}}{U(R_i)} \right)^\nu K_\nu \left( \frac{\sqrt{2\nu} d_{\text{u}}}{U(R_i)} \right).
\end{equation}

\textcolor{black}{
Furthermore, we leverage $U(r)$ to enhance the acquisition function's performance. 
The standard form of Expected Improvement (EI) acquisition function~\cite{ament2023unexpected} in Bayesian Optimization is as follows,
\begin{equation}
    \mathrm{EI}_{y^\star}(x)=\mathbb{E}_{f(x)\sim\mathcal{N}(\mu(x),\sigma^2(x))}\left[ [f(x)-y^\star]_{+} \right] = \sigma(x)h\left( \frac{\mu(x)-y^\star}{\sigma(x)} \right), 
\end{equation}
where $[\,\cdot\,]_{+}=\max(0\,,\,\cdot)$, $y^\star = \max_i y_i$ is the best observed value, and $h(z) = \phi(z)+z\Phi(z)$, $\phi$ is the standard normal distribution density and $\Phi$ is the distribution function. } 

To reduce inefficient exploration along directions with potentially insignificant influence on the function value, we introduce a penalty term that constrains the weighted distance between the candidate point and the current optimum, yielding an uncertainty-accelerated EI acquisition function ($\mathrm{uEI}$) defined as 
\begin{align}
    &\mathrm{uEI}(\theta) = \mathrm{EI}(\theta) \cdot w(\theta), \\
    &w(\theta) = \exp\left( -\sum_{j=1}^d U(r_j)(\theta_j-\theta_j^\star)^2 \right), 
\end{align}
where $\theta$ denotes the reward intensity hyperparameter, $w(\theta)$ is a penalty term to constrain the weighted distance between $\theta$ and $\theta^\star$. 
Specifically, an uncertainty value approaching zero for a particular dimension indicates no restriction on variations along that dimension. 
Conversely, a large uncertainty weight in a certain direction implies that extensive exploration in that direction is discouraged. 
UABO demonstrates significantly improved convergence efficiency (see Section~\ref{sect:ablation} Alb-3), reaching optimal values within limited hyperparameter search steps, thereby substantially enhancing the performance and effectiveness of reward intensity configuration. 
A formal proof of its convergence lower bound is provided in App.~\ref{app:proofs}.

\section{Experiments}
\label{sect:experiments}

\textcolor{black}{In this section, we evaluate the proposed \method framework through extensive experiments on a diverse set of environments and tasks, comparing its performance against human and baseline approaches. All experiments and comparative analyses presented in this paper utilize DeepSeek-v3-241226~\cite{liu2024deepseek} as the foundational model unless explicitly stated otherwise. }

\subsection{Baselines}

\textcolor{black}{
\textbf{Eureka}~\cite{ma2024eureka} (Baseline) provides a systematic approach for generating reward functions utilizing LLMs. It incorporates feedback from various evaluation results to refine the generation of the reward function in evolutionary iterations. This iterative process continues until an optimal reward function is achieved. } 

\textcolor{black}{
\textbf{Text2reward}~\cite{xie2024text2reward} is a reinforcement learning method that automatically generates dense reward functions from natural language task descriptions using LLMs, without relying on expert data or demonstrations, and is able to express human goals in the form of procedural rewards, given to iterations using human feedback. } 

\textcolor{black}{
\textbf{Human}. To maintain a fair comparison, we adopted the same Human data as reported in Eureka~\cite{ma2024eureka}. The original shaped reward functions provided in the benchmark tasks are developed by active reinforcement learning researchers who designed the tasks. These reward functions embody the outcomes of expert human reward engineering. }

\textcolor{black}{
\textbf{Sparse}.These functions correspond to the fitness measures $F$ employed to assess the quality of the generated reward signals. Analogous to human feedback, they are also provided as part of the benchmark suite. 
The detailed configurations for Dexterity tasks and Isaac tasks are consistent with those used in Eureka~\cite{ma2024eureka}. 
The fitness functions of all tasks in ManiSkill2 are specified in App.~\ref{app:maniskill2_details}.} 

\subsection{Experimental Setup}

\textbf{Benchmarks}. Our environments consist of three benchmarks: Isaac, Dexterity and Maniskill2, and comprises 35 different tasks. Nine of these tasks are from the original Isaac Gym environment~\cite{647d5ee0d68f896efa59676c} (Issac), twenty are complex bi-manual tasks~\cite{62afe5495aee126c0f668bd5} (Dexterity) and the remaining six are from the Maniskill2 environment~\cite{63dcdb422c26941cf00b661c}.
See App. \ref{app:benchmark_details} for more details. 

\textbf{Metrics}. 
We examine fore metrics: i. Success Rate (\textbf{SR}). 
We report the success rates of different reward functions on Dexterity and ManiSkill2 tasks. To ensure fair comparison with the baseline methods, the success rates for ManiSkill2 tasks are calculated using the last 50\% of test results from each evaluation, while full test results are used for Dexterity tasks. 
ii.Human Normalized Score \textbf{(HNS)}. For Isaac tasks, following the evaluation setup in the Eureka, we employ the Human Normalized Score, $\frac{\text{Method-Sparse}}{|\text{Human-Sparse}|}$,  as the evaluation metric. 
iii. The Number of Evaluations (\textbf{NOE}), quantified as the total number of simulations conducted across all samples during the optimization process. 
iiii. The Number of LLM Callings (\textbf{NLC}), representing the cumulative number of LLM calls made during both reward function generation and refinement. 
These metrics collectively provide comprehensive assessment, with SR and HNS evaluating the performance of the generated reward functions, and NOE and NLC quantifying design process efficiency. 

\textbf{Policy Learning}. 
The performance of reward functions generated by \method and comparative methods was rigorously validated through RL training. For both Isaac and Dexterity environments, we employ the same high-efficiency PPO~\cite{schulman2017proximal} implementation as used in Eureka, using identical task-specific hyperparameters without modification. In the ManiSkill2 environment, we utilized both SAC~\cite{haarnoja2018soft} and PPO algorithms to ensure fair comparison between \method, Text2Reward, and Eureka, strictly maintaining the original hyperparameter configurations across all methods. 
See App.~\ref{app:Hyper-parameter_Settings} for detailed parameter configurations. 

\subsection{Results}
\label{sect:results}

\textbf{\method improves the efficiency of the reward function design}. 
Table~\ref{tab:sota_efficiency} presents a comprehensive comparison of computational efficiency between \method and Eureka across three benchmarks. When optimizing for peak reward performance, \method requires only $52.4\%$ of the simulation episodes (NOE) and $46.6\%$ of the evolutionary iterations (NLC) compared to Eureka. 
This significant acceleration demonstrates \method's superior optimization efficiency in automated reward function design. 
See App.\ref{app:detailed_desults} for a per-task breakdown. 

\begin{table}[ht]
  \small
  \caption{\textcolor{black}{\method demonstrates superior efficiency across all benchmarks.} }
  \centering
  \begin{tabular}{lcccccc}
    \toprule
    & \multicolumn{2}{c}{Isaac} & \multicolumn{2}{c}{Dexterity}  & \multicolumn{2}{c}{ManiSkill2} \\
    \cmidrule(lr){2-3} \cmidrule(lr){4-5} \cmidrule(lr){6-7}
    Methods & NOE$\downarrow$ & NLC$\downarrow$ & NOE$\downarrow$ & NLC$\downarrow$ & NOE$\downarrow$ & NLC$\downarrow$  \\
    \midrule
    Txet2Reward &72.889   &5 & 84.45 & 6.05 & 106.667 & 6.667 \\
    Eureka & 68.667 & 4.556 & 80.05 & 5.5  & 98.667 & 6.167 \\

    \method & \textbf{39.501} & \textbf{2.495} & \textbf{57.8} & \textbf{3.4}  & \textbf{32.33} & \textbf{1.667} \\
    \bottomrule
  \end{tabular}
  \label{tab:sota_efficiency}
\end{table}

\textbf{\method demonstrates superior reward function performance}. 
Table~\ref{tab:sota_quality} presents a systematic comparison of reward functions generated by different approaches across benchmark tasks. Under identical simulation budgets (NOE), reinforcement learning agents trained with \method-derived reward functions achieve significantly higher success rates than competing methods. Notably, \method demonstrates substantial performance gains of 132\%, 45\% and 76\% over Eureka across the three experimental environments, representing substantial quality enhancements.  Furthermore, \method-designed reward functions outperform manually engineered counterparts by a considerable margin, providing compelling evidence for the efficacy of automated reinforcement learning frameworks. 

\begin{table}[ht]
  \small
  \caption{\textcolor{black}{\method exhibits higher reward quality across all benchmarks.} }
  \centering
  \begin{tabular}{lccc}
    \toprule
    Methods & Isaac (HNS$\uparrow$) & Dexterity (SR$\uparrow$) & ManiSkill2 (SR$\uparrow$) \\
    \midrule
    Sparse        & 0     & 0.054 & 0.101    \\        
    Human         & 1.000 & 0.459 & 0.434 \\
    Text2Reward   & 1.553 & 0.452 & 0.554 \\
    Eureka        & 1.607 & 0.466 & 0.449 \\
    \method       & \textbf{3.424} & \textbf{0.675} & \textbf{0.792} \\
    \bottomrule
  \end{tabular}
  \label{tab:sota_quality}
\end{table}

\textbf{\method achieves synergistic progress in both generation quality and generation efficiency}. 
In Figure~\ref{fig:multi}, it can be intuitively seen that the superiority of \method over Eureka, text2reward and Human, both in terms of the success rate and the reduction in the number of simulations, has been well improved on these typical tasks. 
Each data point in the line plots represents the effectiveness of the reward function obtained after a single iteration. The results show that \method surpasses human-designed rewards after just one LLM refinement cycle and achieves optimal performance with significantly fewer iterations than both Eureka and Text2Reward, which means that it consumes fewer tokens. 
These results collectively demonstrate that \method achieves simultaneous improvements in both reward generation efficiency and design quality.

\begin{figure}[htbp]
  \centering
  \captionsetup[subfigure]{aboveskip=0pt, belowskip=-6pt}

  \includegraphics[width=0.7\linewidth]{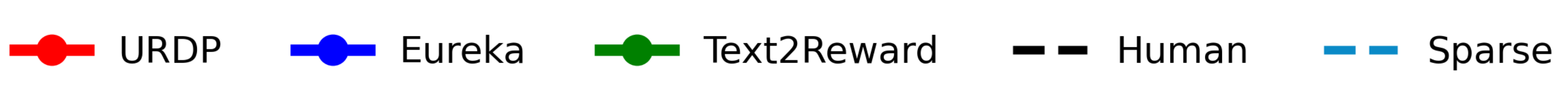}
  \vspace{1pt}  

  \begin{subfigure}[b]{0.31\textwidth}
    \includegraphics[width=\linewidth]{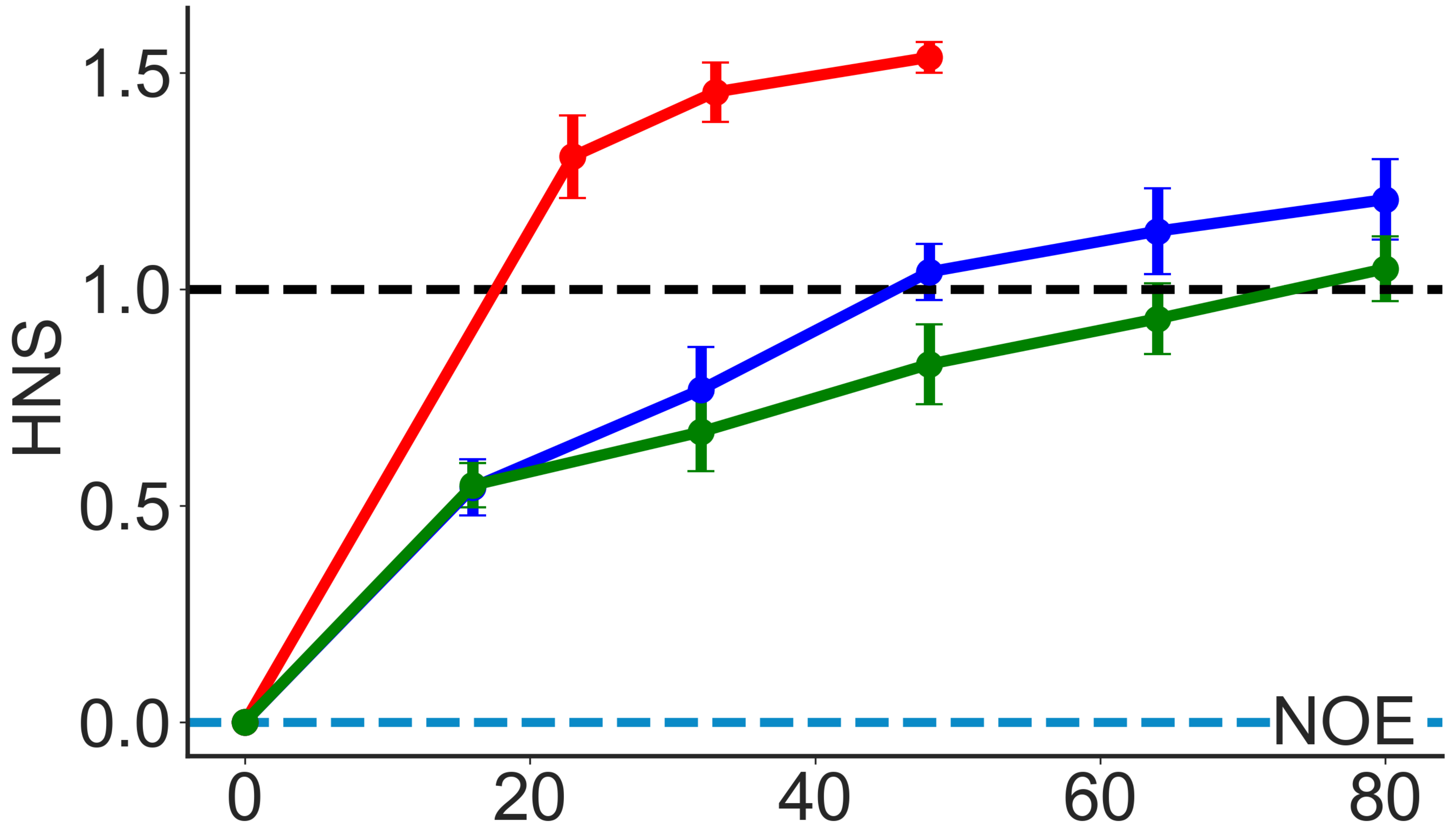}
    \caption{Ant}
    \label{fig:sub1}
  \end{subfigure}
  \hspace{0.01\textwidth}
  \begin{subfigure}[b]{0.31\textwidth}
    \includegraphics[width=\linewidth]{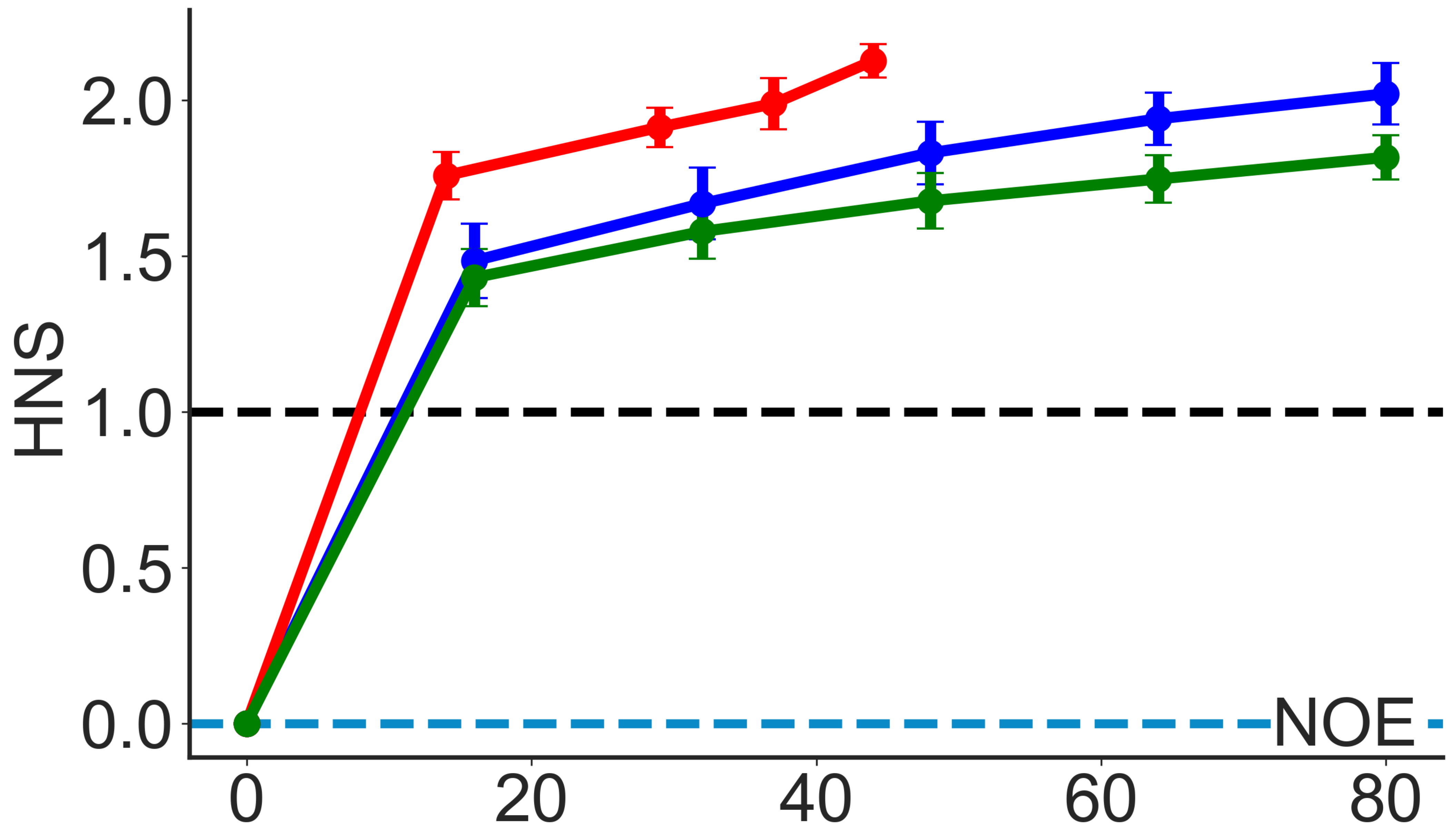}
    \caption{AllegroHand}
    \label{fig:sub2}
  \end{subfigure}
  \hspace{0.01\textwidth}
  \begin{subfigure}[b]{0.31\textwidth}
    \includegraphics[width=\linewidth]{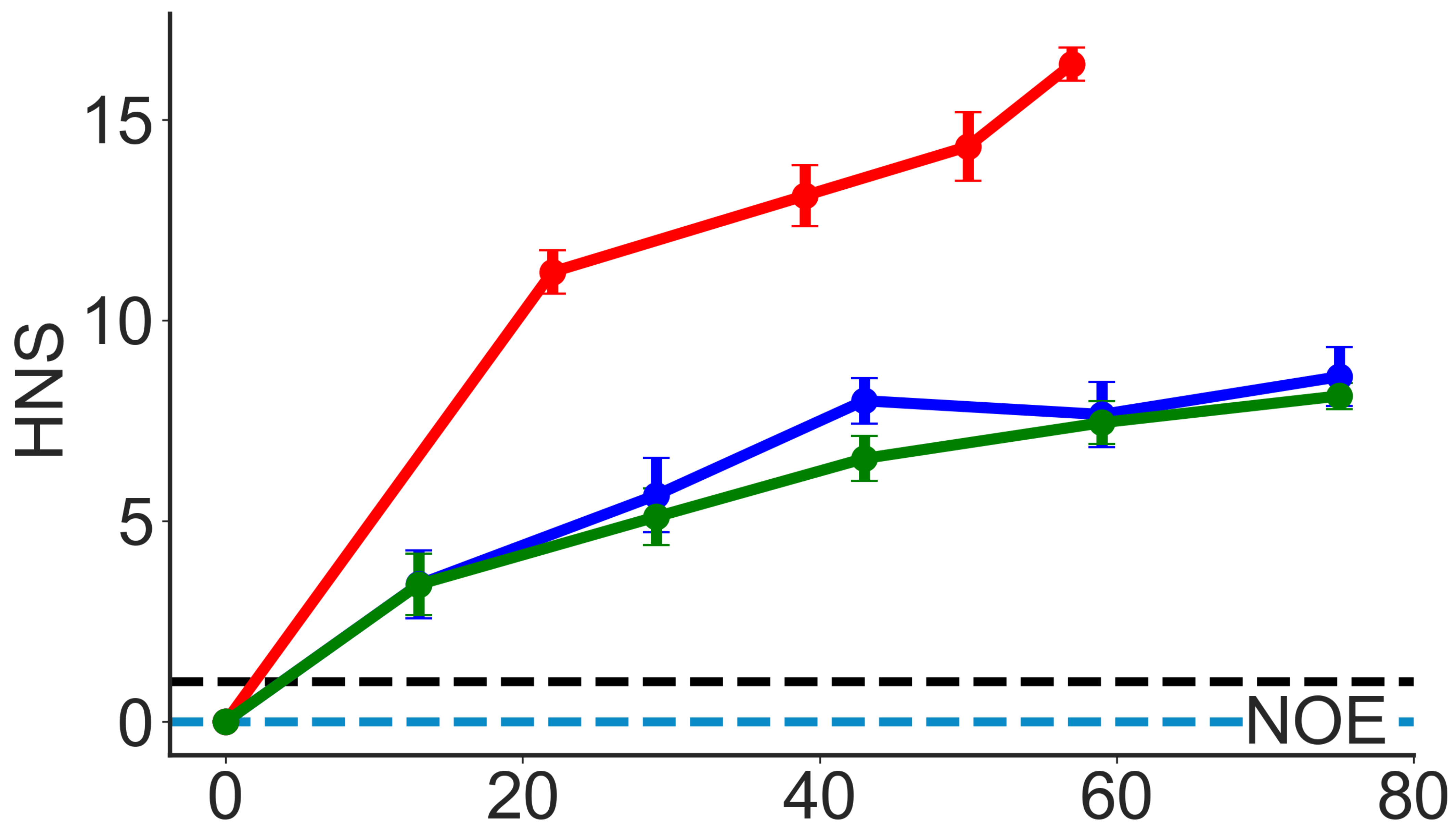}
    \caption{FrankaCabinet}
    \label{fig:sub3}
  \end{subfigure}

  \vspace{5pt}

  \begin{subfigure}[b]{0.31\textwidth}
    \includegraphics[width=\linewidth]{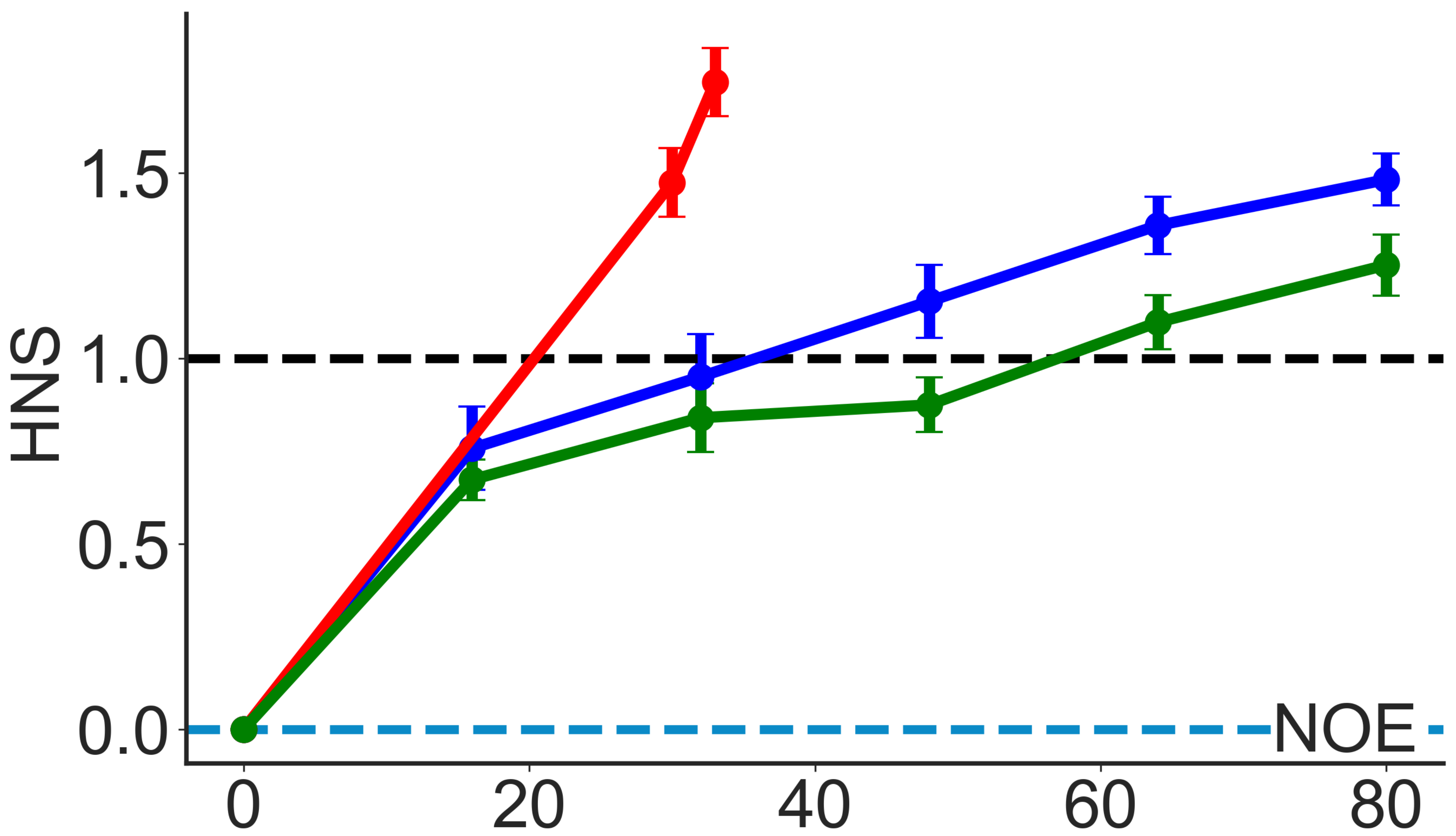}
    \caption{ShadowHand}
    \label{fig:sub4}
  \end{subfigure}
  \hspace{0.01\textwidth}
  \begin{subfigure}[b]{0.31\textwidth}
    \includegraphics[width=\linewidth]{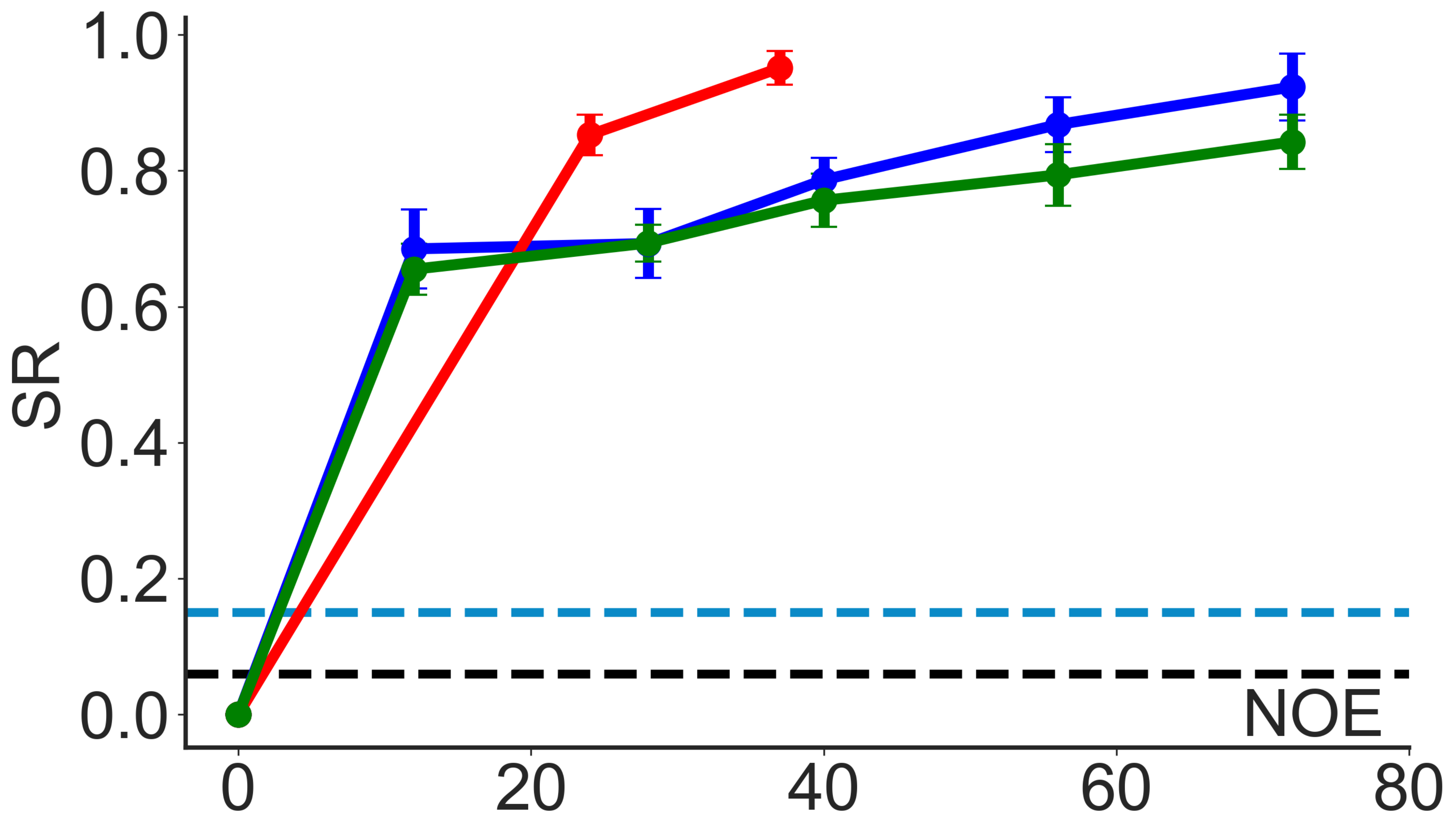}
    \caption{DoorCloseOutward}
    \label{fig:sub5}
  \end{subfigure}
  \hspace{0.01\textwidth}
  \begin{subfigure}[b]{0.31\textwidth}
    \includegraphics[width=\linewidth]{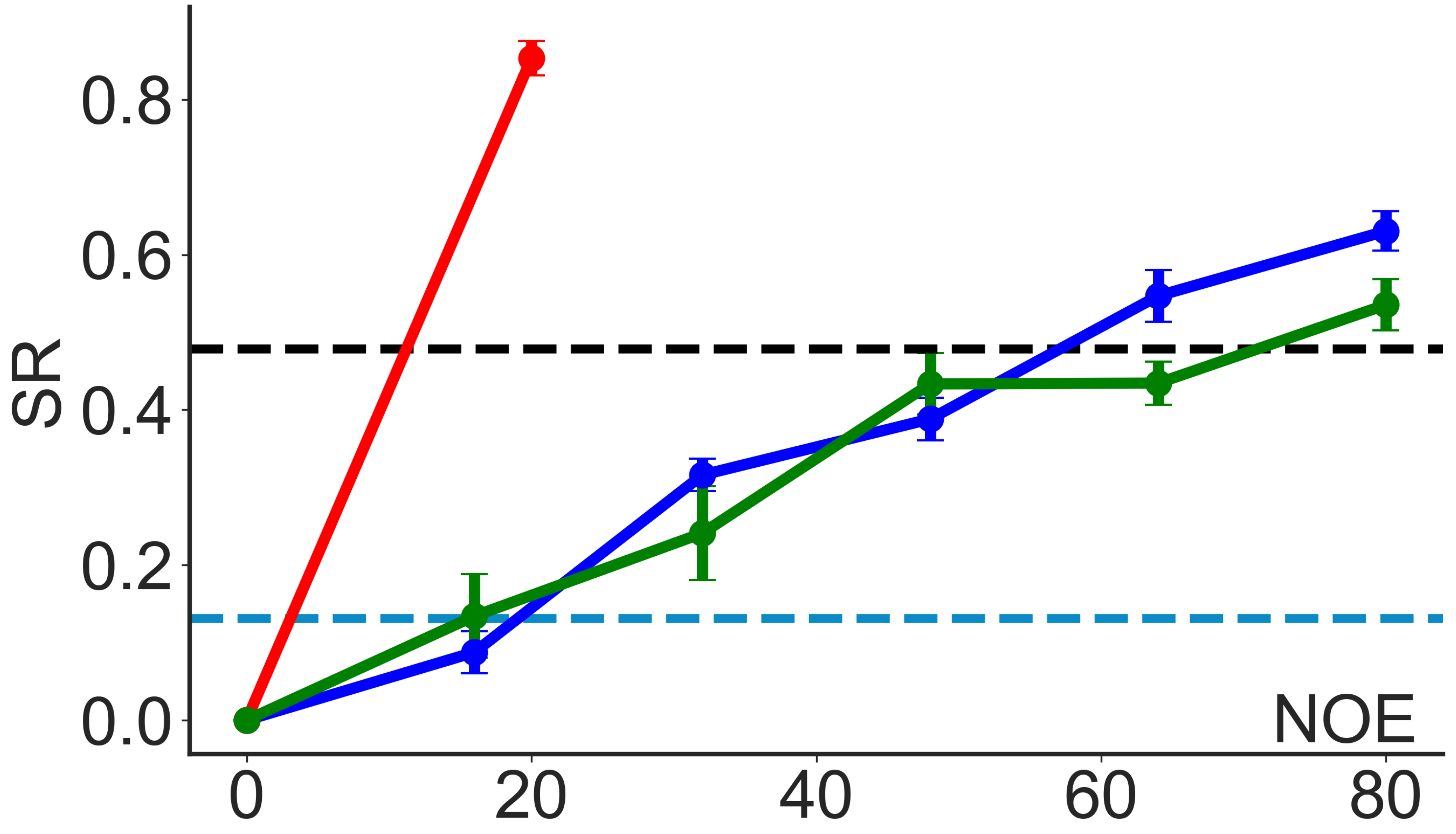}
    \caption{PickCube}
    \label{fig:sub6}
  \end{subfigure}

  \caption{\textcolor{black}{Comparisons of \method with other methods in Isaac (a-d), Dexterity (e), and ManiSkill2 (f).}  }
  \label{fig:multi}
\end{figure}

\subsection{Ablation Experiments}
\label{sect:ablation}

Furthermore, we explore the role of each core content in \method in achieving the above results.

\textcolor{black}{
\textbf{Abl-1: Uncertainty quantification improves the efficiency of reward design}. 
To evaluate the role of uncertainty sampling, we conduct ablation studies by removing the uncertainty sampling and filtering module from \method (denoted as \textbf{URDP w.o. Uncertainty}). Experimental results in Figure~\ref{fig:multi3} demonstrate that \method achieves comparable success rates while requiring significantly fewer optimization episodes (NOE) than URDP w.o. Uncertainty, quantitatively validating the efficiency improvement brought by uncertainty-aware sampling. Interestingly, our analysis also reveals the positive role of the uncertainty in obtaining novel reward functions, with detailed mechanistic explanations to be discussed in Section~\ref{sect:discussion} (Disc-2).}

\begin{figure}[htbp]
  \centering
  \captionsetup[subfigure]{aboveskip=0pt, belowskip=-6pt}

  \includegraphics[width=0.7\linewidth]{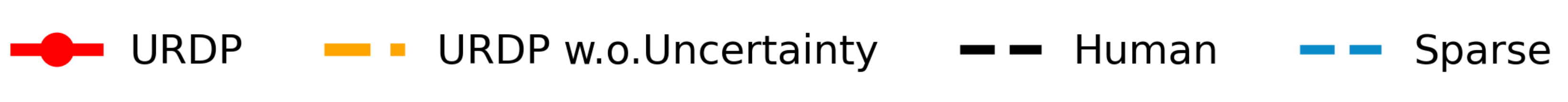}
  \vspace{1pt}  
  
  \begin{subfigure}[b]{0.31\textwidth}
    \includegraphics[width=\linewidth]{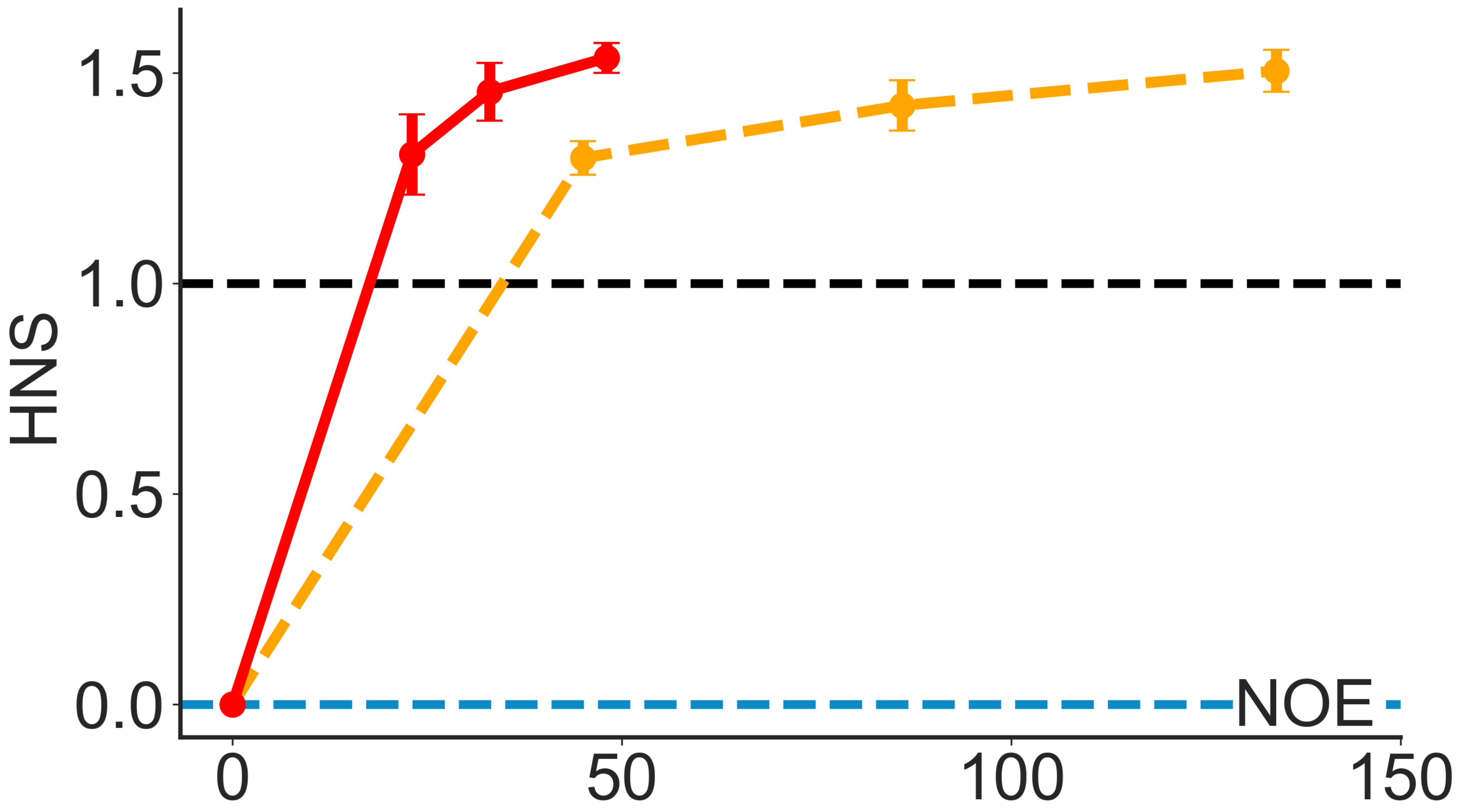}
    \caption{Ant}
    \label{fig:sub13}
  \end{subfigure}
  \hspace{0.01\textwidth}
  \begin{subfigure}[b]{0.31\textwidth}
    \includegraphics[width=\linewidth]{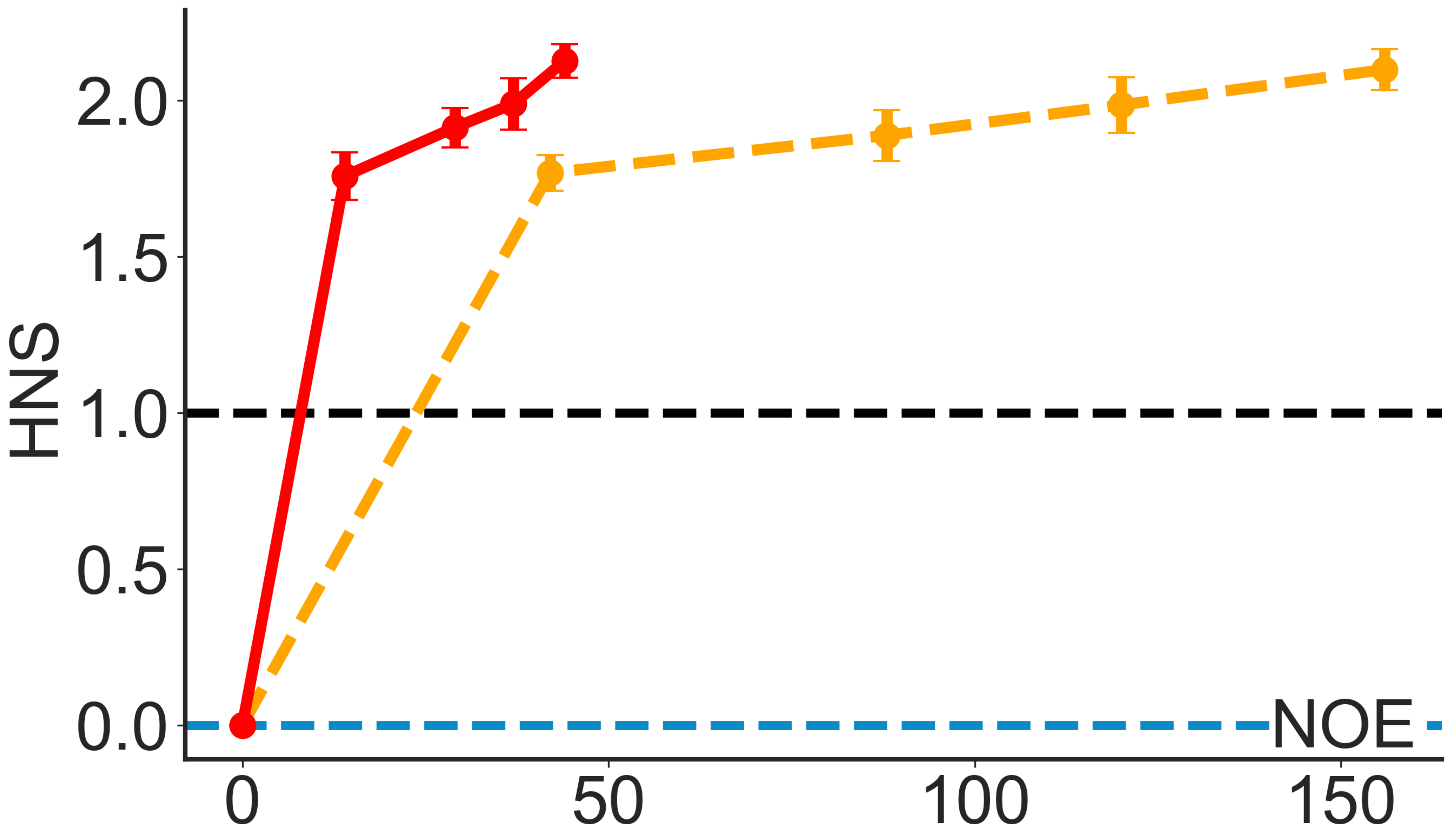}
    \caption{AllegroHand}
    \label{fig:sub14}
  \end{subfigure}
  \hspace{0.01\textwidth}
  \begin{subfigure}[b]{0.31\textwidth}
    \includegraphics[width=\linewidth]{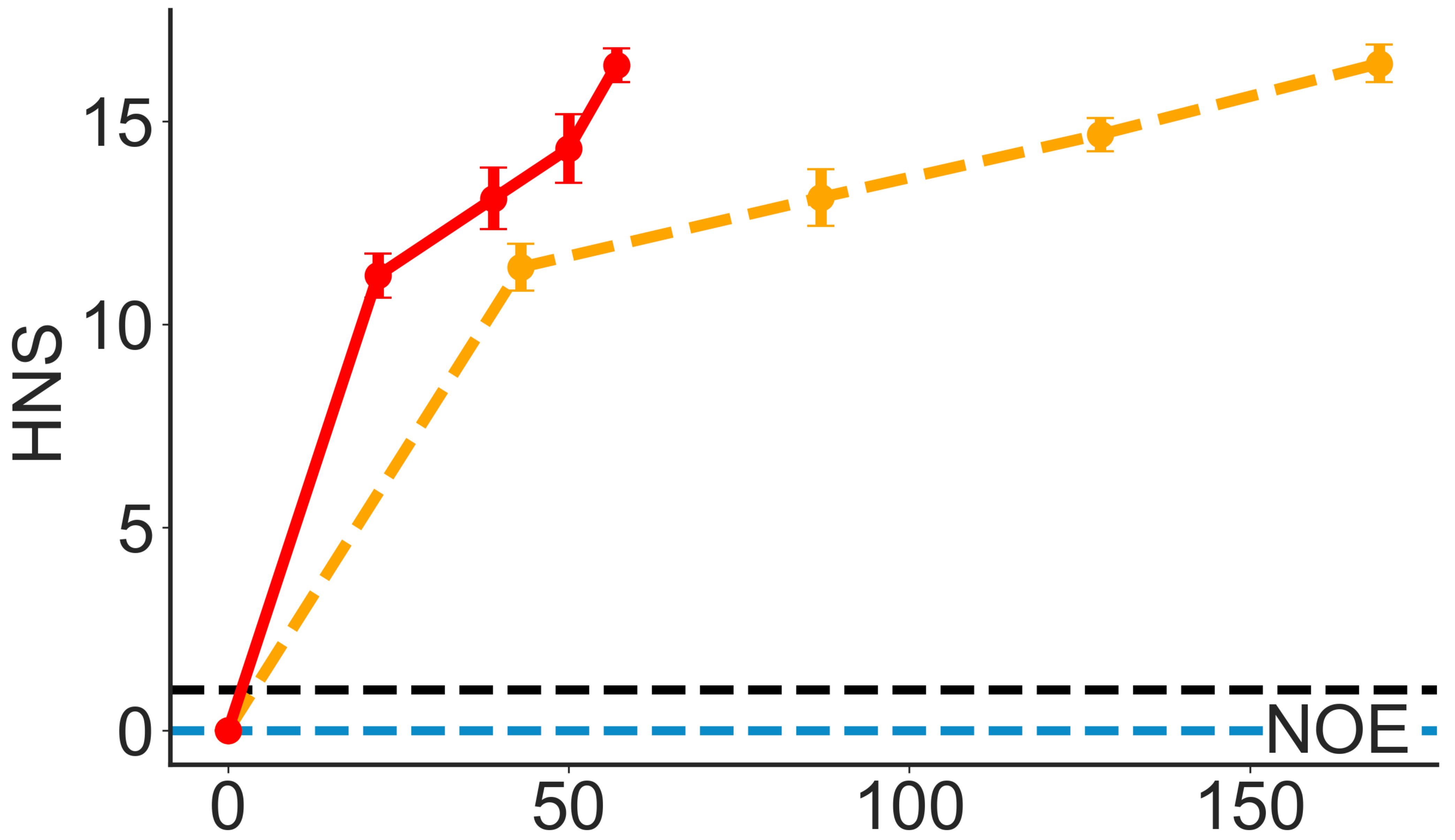}
    \caption{FrankaCabinet}
    \label{fig:sub15}
  \end{subfigure}

  \vspace{5pt}

  \begin{subfigure}[b]{0.31\textwidth}
    \includegraphics[width=\linewidth]{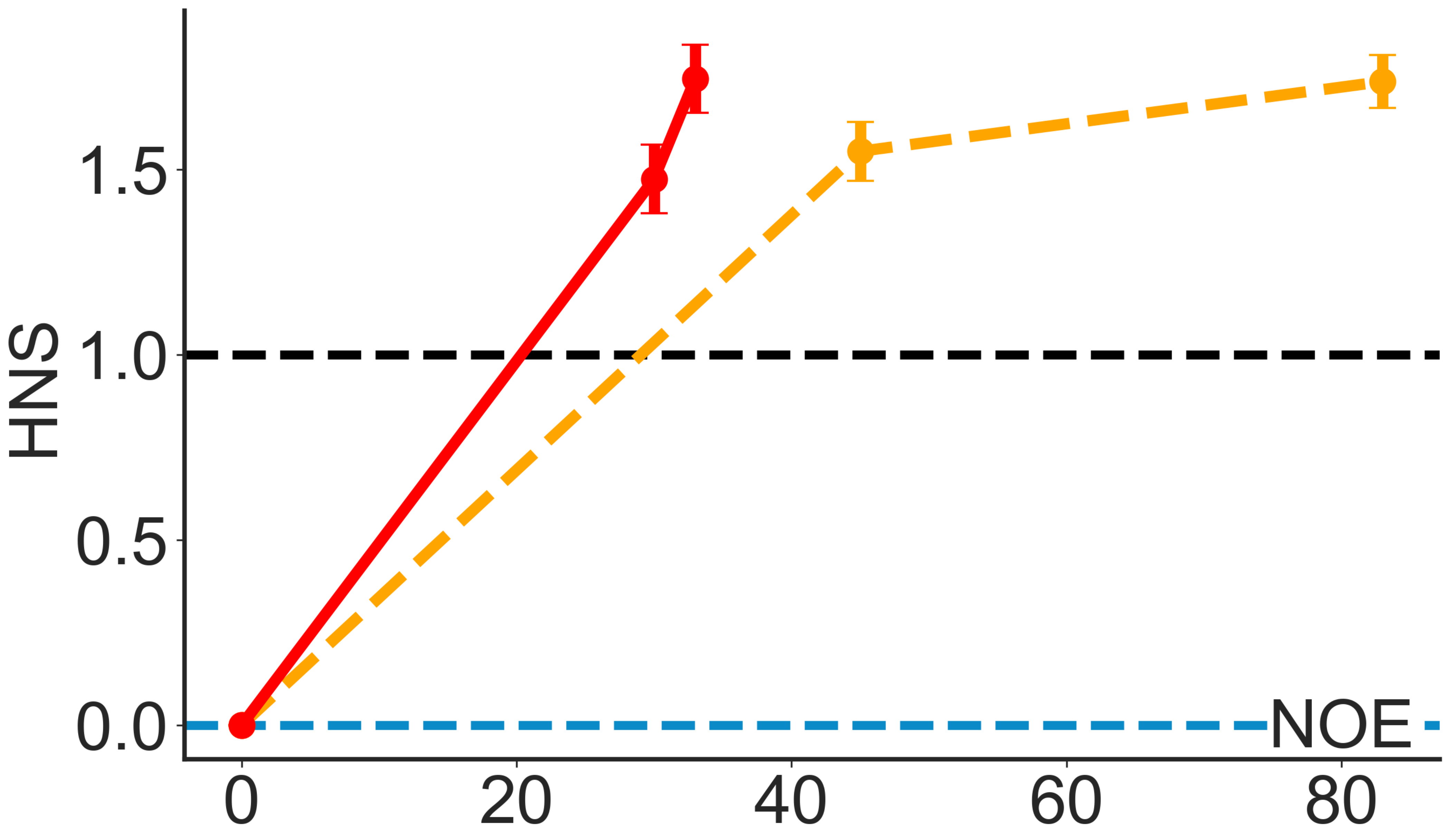}
    \caption{ShadowHand}
    \label{fig:sub16}
  \end{subfigure}
  \hspace{0.01\textwidth}
  \begin{subfigure}[b]{0.31\textwidth}
    \includegraphics[width=\linewidth]{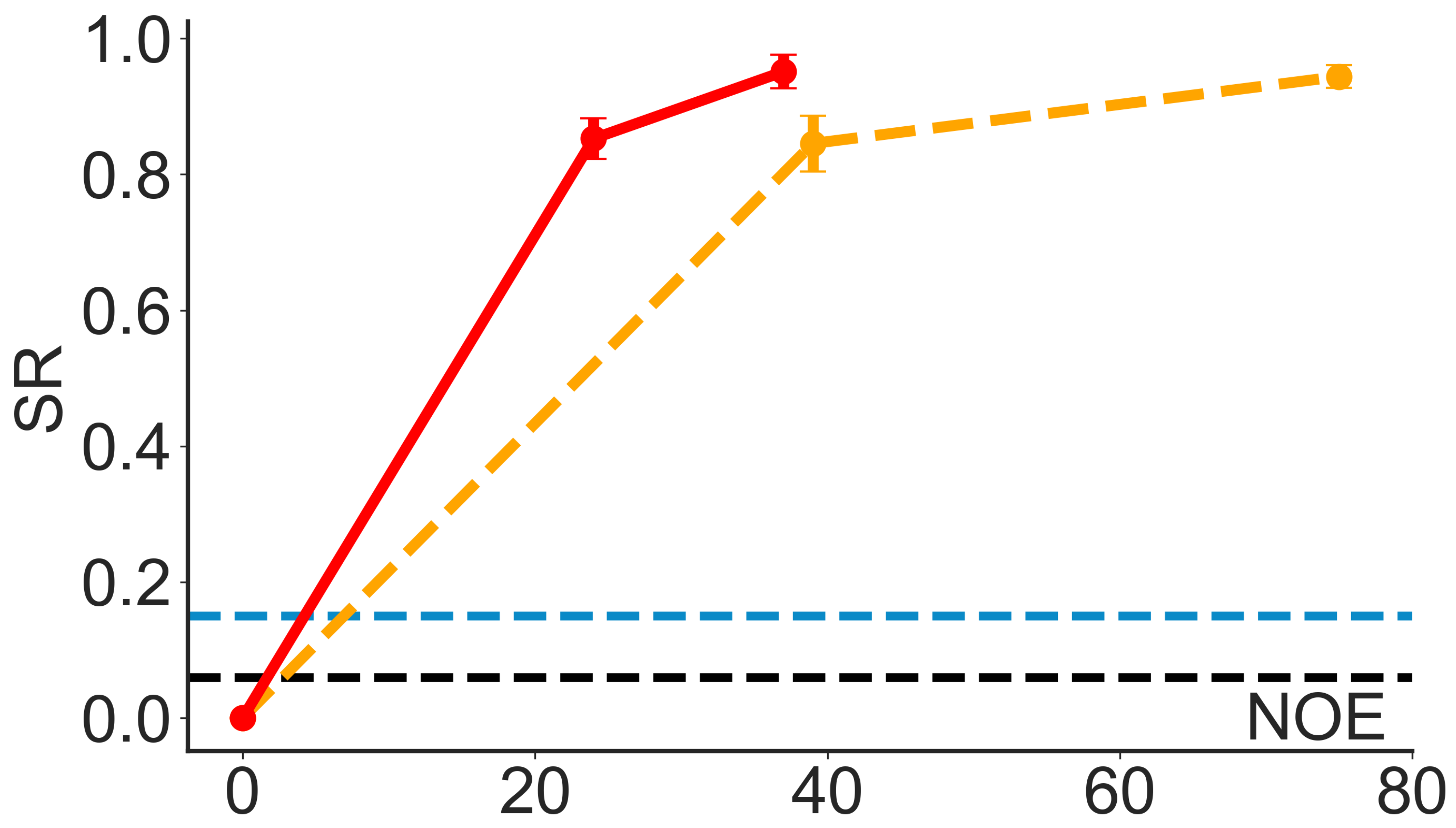}
    \caption{DoorCloseOutward}
    \label{fig:sub17}
  \end{subfigure}
  \hspace{0.01\textwidth}
  \begin{subfigure}[b]{0.31\textwidth}
    \includegraphics[width=\linewidth]{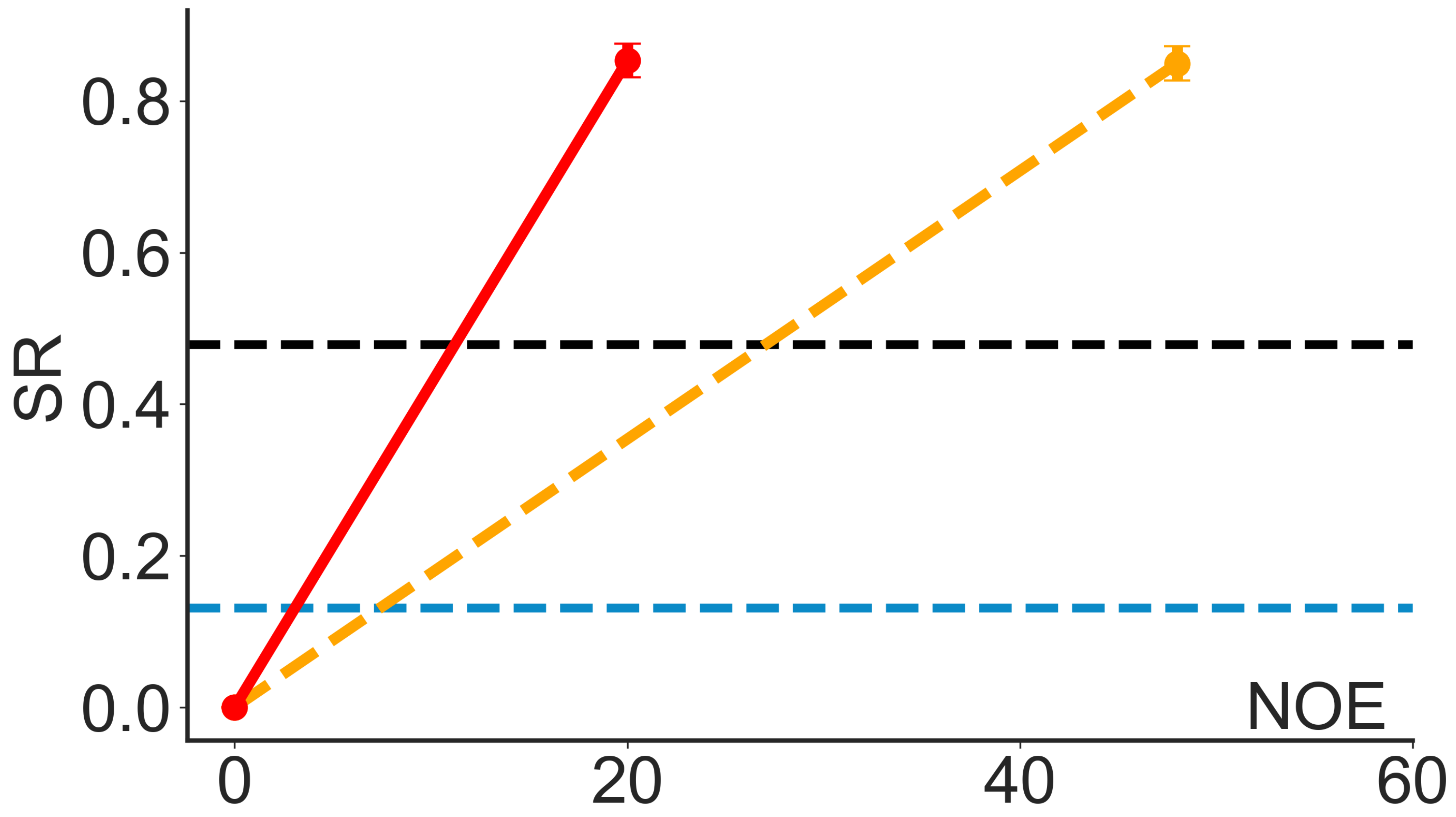}
    \caption{PickCube}
    \label{fig:sub18}
  \end{subfigure}
  \caption{\textcolor{black}{When generating reward functions of comparable quality, \method requires significantly fewer simulation training episodes, attributable to its effective uncertainty-based filtering mechanism.} }
  \label{fig:multi3}
\end{figure}

\textcolor{black}{
\textbf{Abl-2: Decoupled optimization is the cornerstone for the collaborative improvement of performance and efficiency}. 
This experimental investigation examines the role of decoupling in \method, where reward components and their associated intensities are optimized separately. To evaluate this mechanism, we ablate UABO from \method (denoted as \textbf{URDP w.o. UABO}) while maintaining identical configurations otherwise, resulting in a system where both reward components and intensities are jointly configured by the LLM without alternating optimization. Under unrestricted NLC, we compare the HNS or SR achieved by URDP w.o. UABO using equivalent NOE to the standard \method implementation. As shown in Figure~\ref{fig:multi2} (dashed lines), consistent reductions in both HNS and SR are observed across all three benchmark tasks, demonstrating the substantial impact of decoupled optimization on improving reward function design quality. Furthermore, \method exhibites faster convergence (requiring fewer NLC) to optimal solutions, suggesting that decoupling also enhances the efficiency of evolutionary search. A comprehensive discussion of this phenomena and the underlying mechanisms is presented in Section~\ref{sect:discussion} (Disc-1). }

\begin{figure}[h]
  \centering
  \captionsetup[subfigure]{aboveskip=0pt, belowskip=-6pt}

  \includegraphics[width=\linewidth]{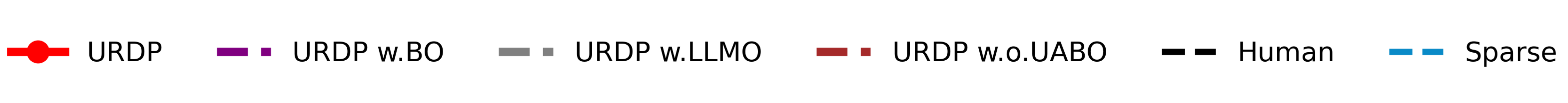}
  \vspace{-15pt}  
  
  \begin{subfigure}[b]{0.31\textwidth}
    \includegraphics[width=\linewidth]{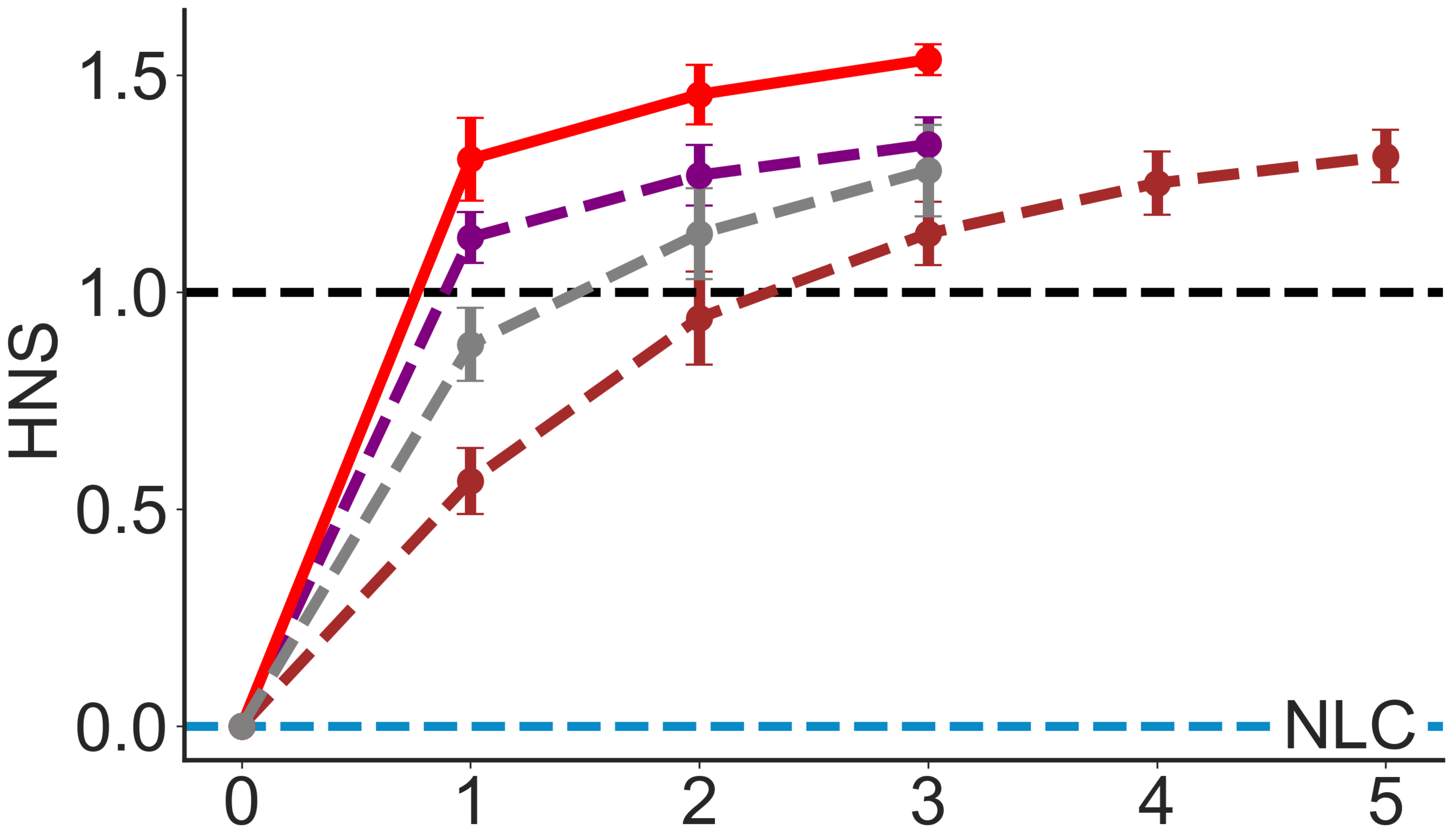}
    \caption{Ant}
    \label{fig:sub7}
  \end{subfigure}
  \hspace{0.01\textwidth}
  \begin{subfigure}[b]{0.31\textwidth}
    \includegraphics[width=\linewidth]{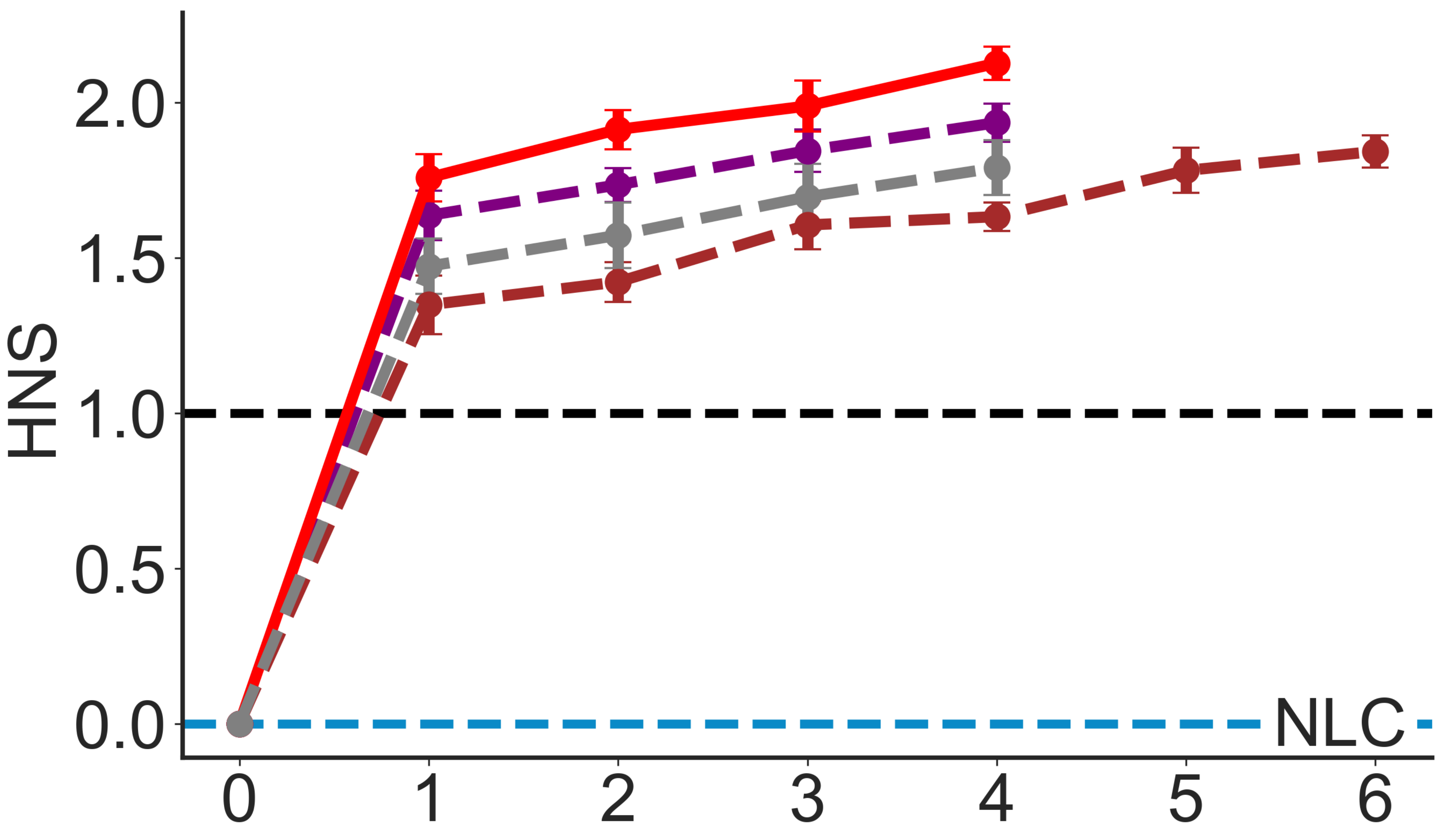}
    \caption{AllegroHand}
    \label{fig:sub8}
  \end{subfigure}
  \hspace{0.01\textwidth}
  \begin{subfigure}[b]{0.31\textwidth}
    \includegraphics[width=\linewidth]{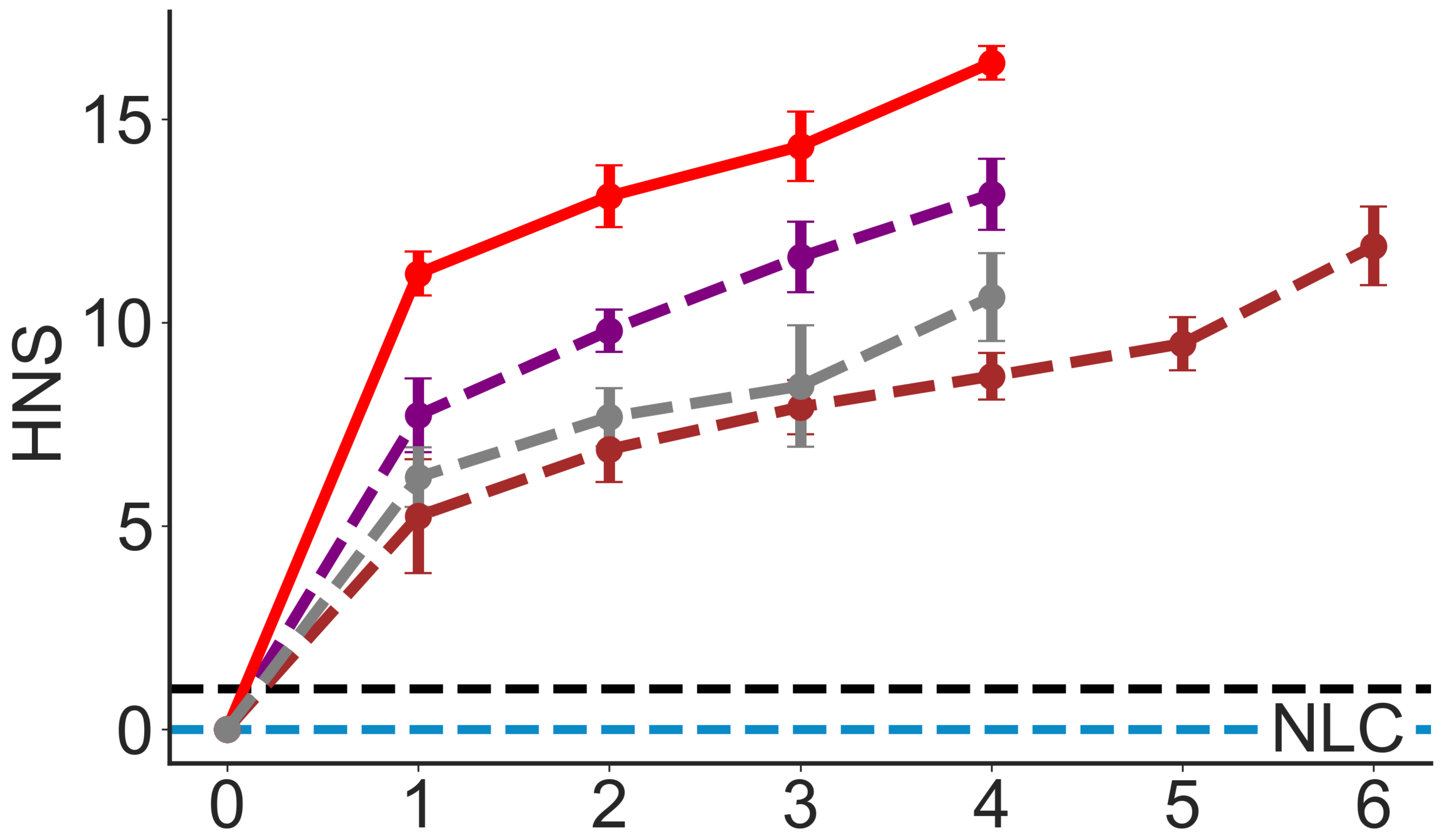}
    \caption{FrankaCabinet}
    \label{fig:sub9}
  \end{subfigure}

  \vspace{5pt}

  \begin{subfigure}[b]{0.31\textwidth}
    \includegraphics[width=\linewidth]{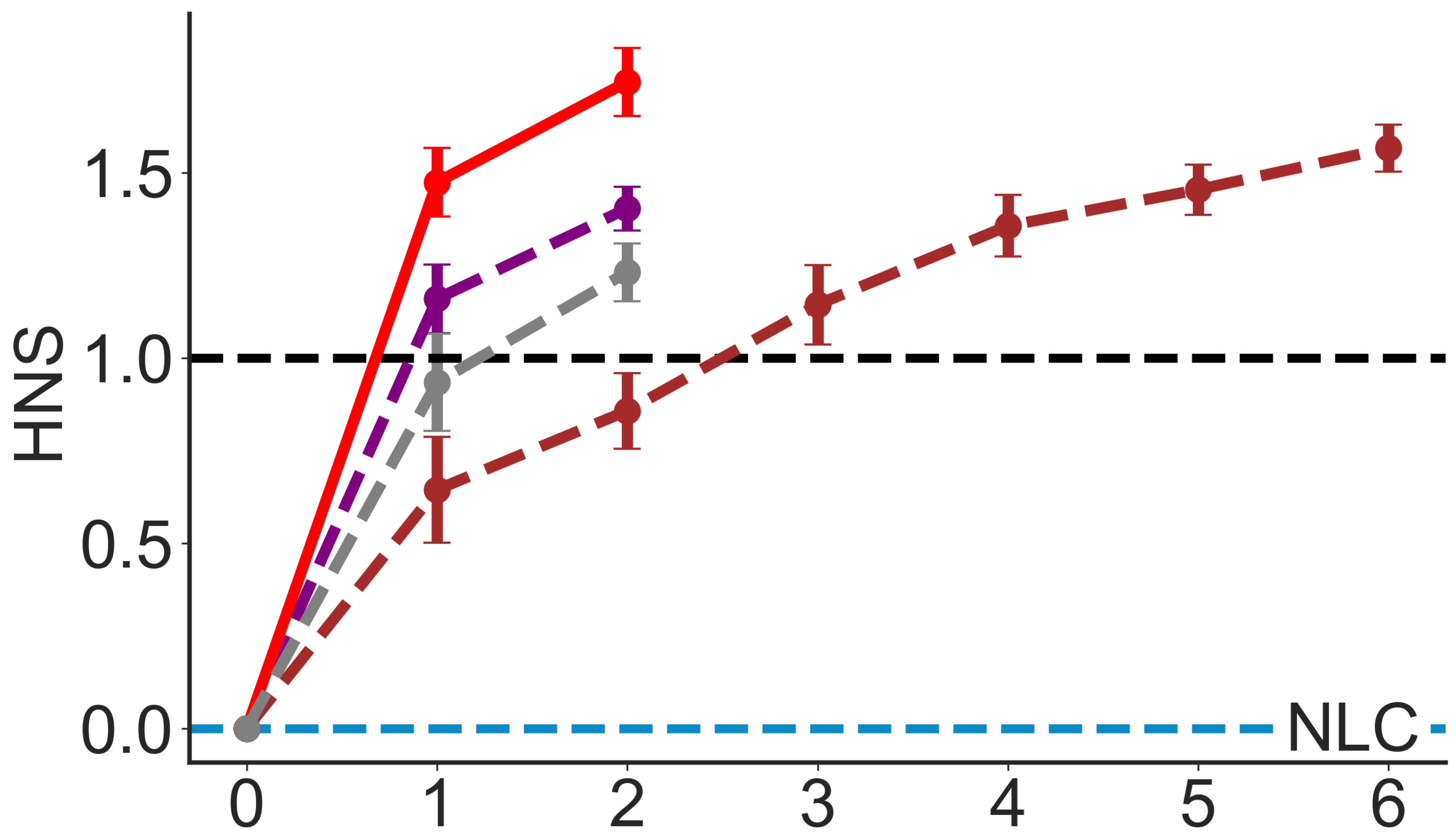}
    \caption{ShadowHand}
    \label{fig:sub10}
  \end{subfigure}
  \hspace{0.01\textwidth}
  \begin{subfigure}[b]{0.31\textwidth}
    \includegraphics[width=\linewidth]{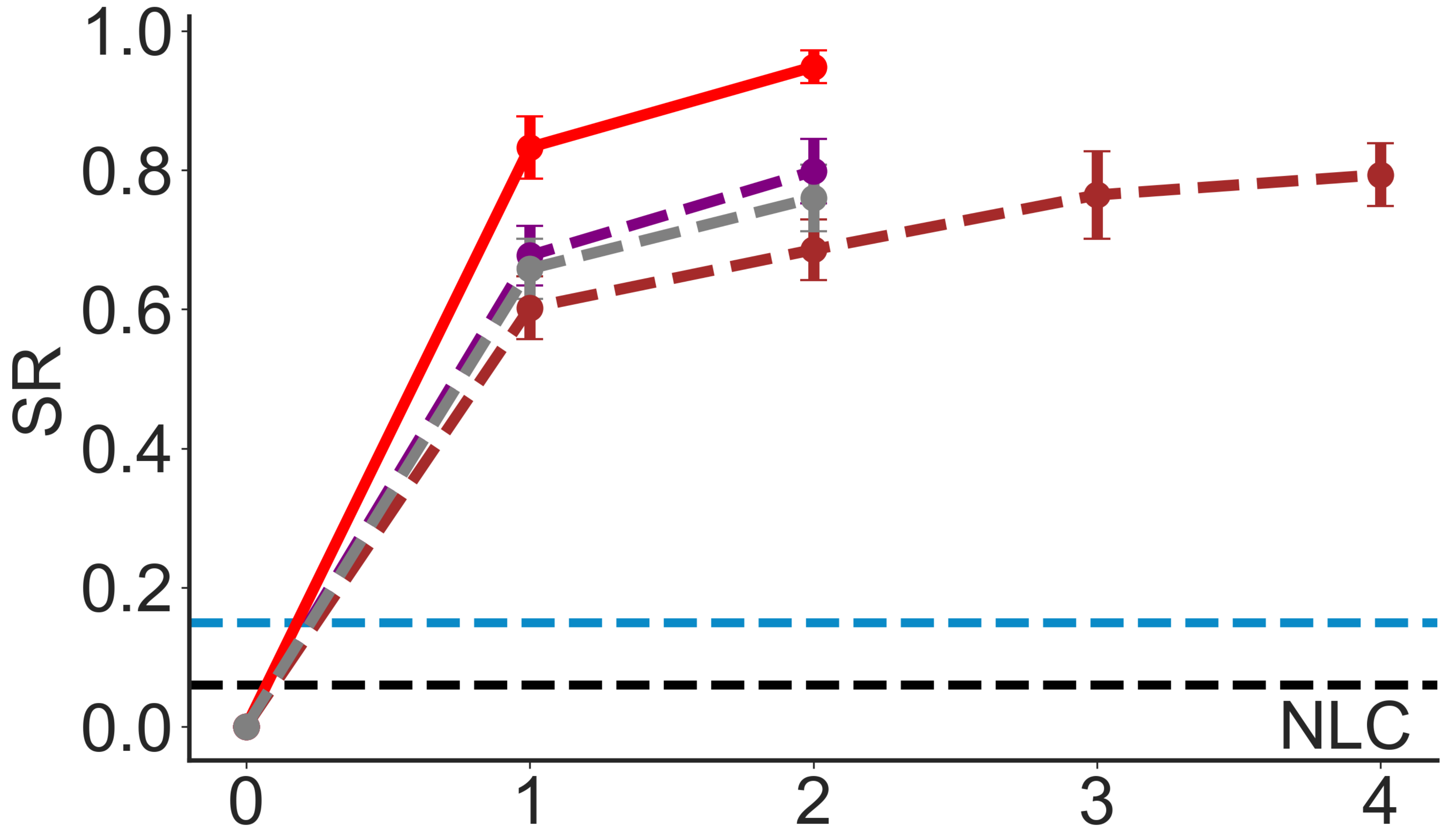}
    \caption{DoorCloseOutward}
    \label{fig:sub11}
  \end{subfigure}
  \hspace{0.01\textwidth}
  \begin{subfigure}[b]{0.31\textwidth}
    \includegraphics[width=\linewidth]{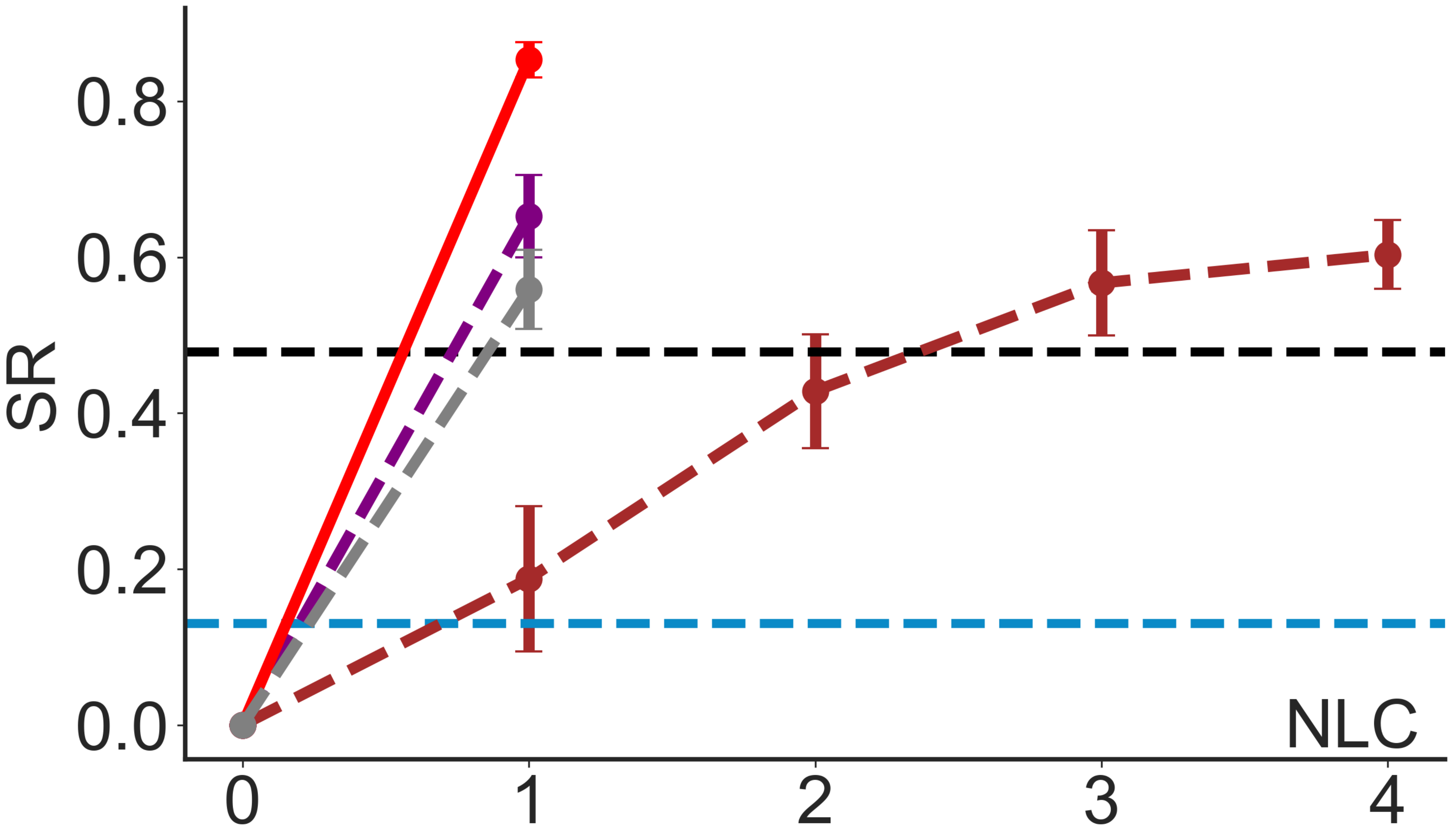}
    \caption{PickCube}
    \label{fig:sub12}
  \end{subfigure}
  \caption{\textcolor{black}{A comparison of the \method, URDP w. BO, URDP w. LLMO and URDP w.o. UABO when all methods utilize identical simulation budgets (NOE).} }
  \label{fig:multi2}
\end{figure}

\textbf{Abl-3: The UABO play a key role in performance improvement}. 
This experimental study systematically compares the numerical optimization capabilities between large language models (LLMs) and Bayesian optimization approaches within our framework. In the first ablation, we replaced \method's UABO module with LLM reflection (denoted as \textbf{URDP w. LLMO}), employing identical prompting strategies to Eureka, thereby configuring both outer-loop (reward components) and inner-loop (reward intensities) optimization entirely through LLMs. Figure~\ref{fig:multi2} demonstrates that even under decoupled optimization conditions, LLM-based numerical optimization underperforms Bayesian optimization, revealing fundamental limitations in LLMs' mathematical optimization capabilities while confirming the critical role of Bayesian methods in enhancing reward function quality. These results substantiate our core hypothesis that LLMs serve more effectively as controllers for numerical optimization tools rather than direct optimizers.  
Subsequent validation experiments replacing UABO with standard BO (\textbf{URDP w. BO}) reveal UABO's superior efficiency (see Figure~\ref{fig:bo_ibo}): \method (UABO) achieves comparable or better performance than URDP w. BO using only $80\%$ of sampling budget across all Isaac tasks, with consistently superior final reward function performance, demonstrating that uncertainty-aware priors accelerate optimal search in reward function design.

\begin{figure}
  \centering
    \includegraphics[width=0.55\linewidth]{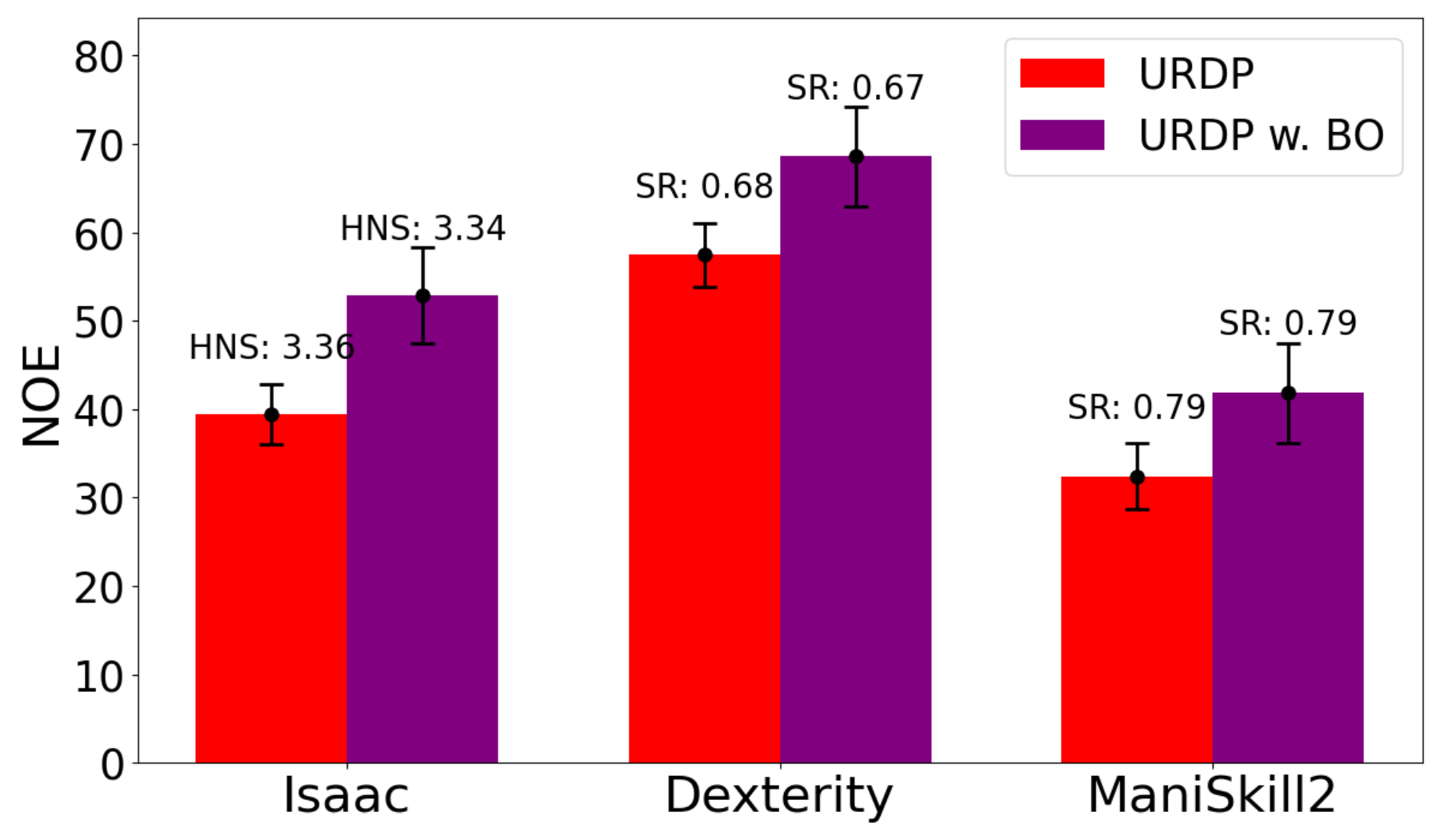}
  \caption{\textcolor{black}{\method achieves a significantly greater improvement in efficiency across all benchmarks compared to BO.} }
  \label{fig:bo_ibo}
\end{figure}

\subsection{Extended Discussion}
\label{sect:discussion}

\textbf{Disc-1: Decoupling and performance degradation in evolutionary search.} As illustrated in Figure~\ref{fig:degradation} (blue line), our experiments reveal performance degradation during evolutionary search in the baseline (Eureka) approach, with performance regression observed in $23\%$ of Isaac tasks. Analysis of the LLMs' decision-making in these cases demonstrates that in $75\%$ of instances, the models modified only the reward intensity hyperparameters while leaving the reward components unchanged, indicating a propensity for erroneous judgments in numerical optimization. 
More critically, we identify \textit{oscillatory phenomenon} in the baseline optimization process (see Figure~\ref{fig:sub23}), where LLMs entered persistent cycles of alternating between limited sets of hyperparameters during evolutionary search. This optimization instability resulted in complete convergence failure, with the baseline system trapped in ineffective, non-progressive iterations.

\begin{figure}[htbp]
  \centering
    \includegraphics[width=0.6\linewidth]{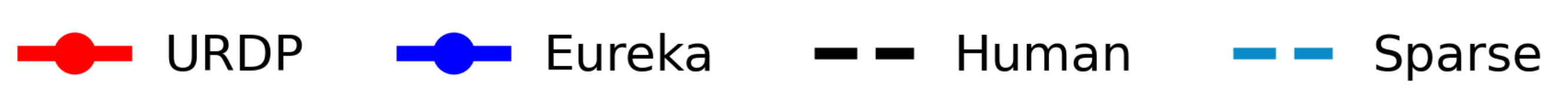}
    \vspace{1pt}  

    \begin{subfigure}[b]{0.31\textwidth}
    \includegraphics[width=\linewidth]{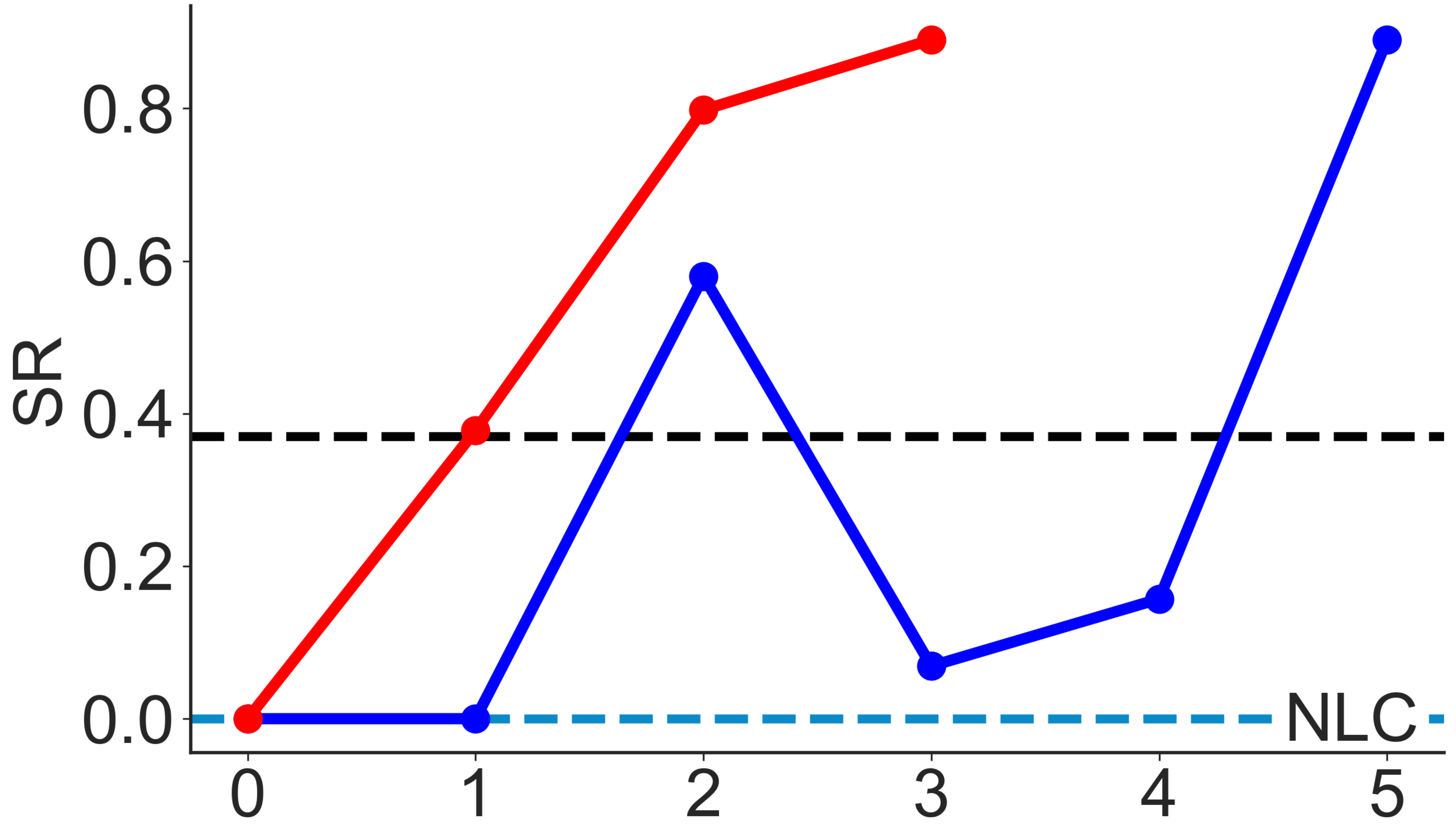}
    \caption{LiftUnderarm}
    \label{fig:sub21}
  \end{subfigure}
  \begin{subfigure}[b]{0.31\textwidth}
    \includegraphics[width=\linewidth]{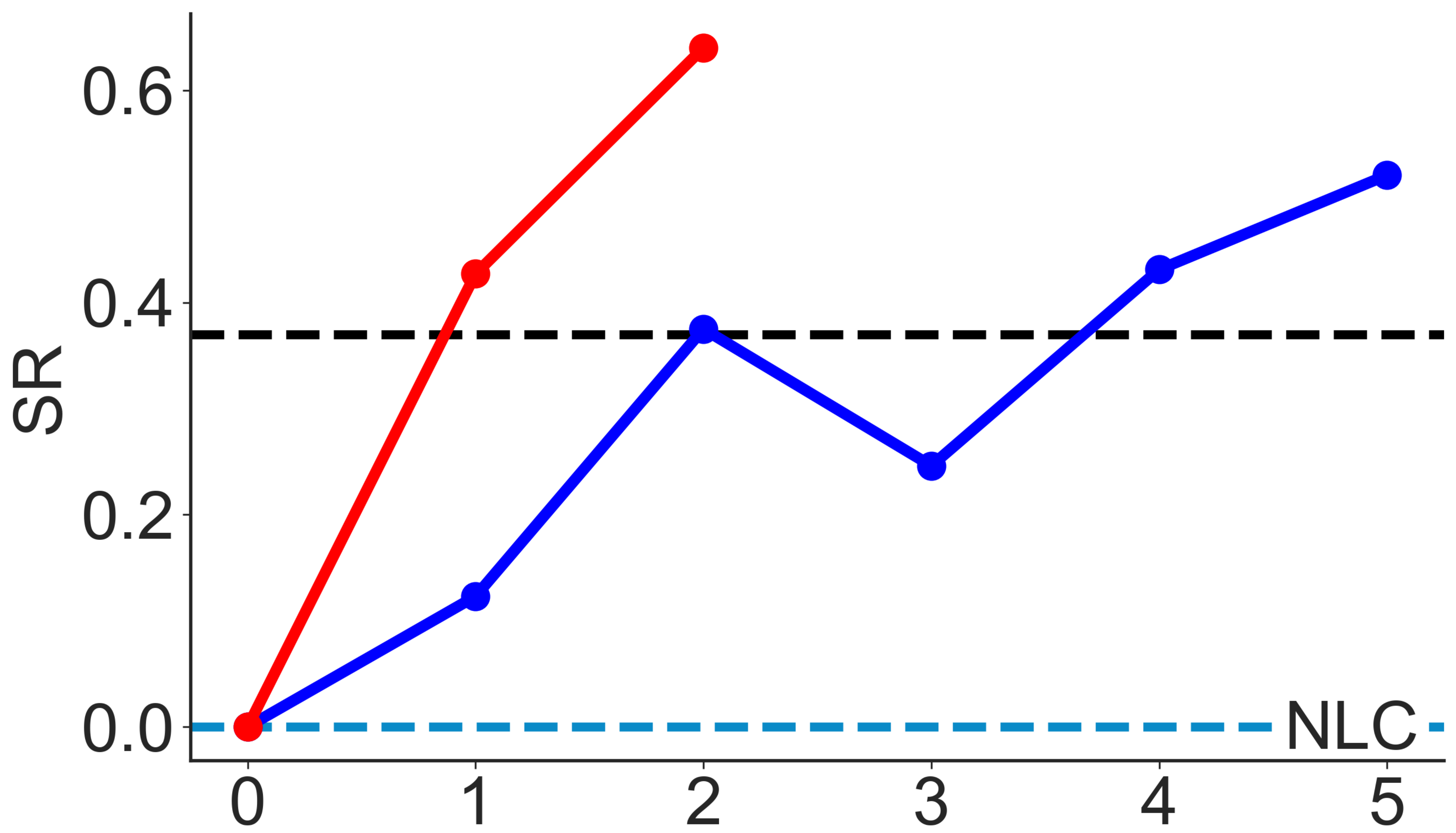}
    \caption{OpenCabinetDrawer}
    \label{fig:sub22}
  \end{subfigure}
  \begin{subfigure}[b]{0.31\textwidth}
    \includegraphics[width=\linewidth]{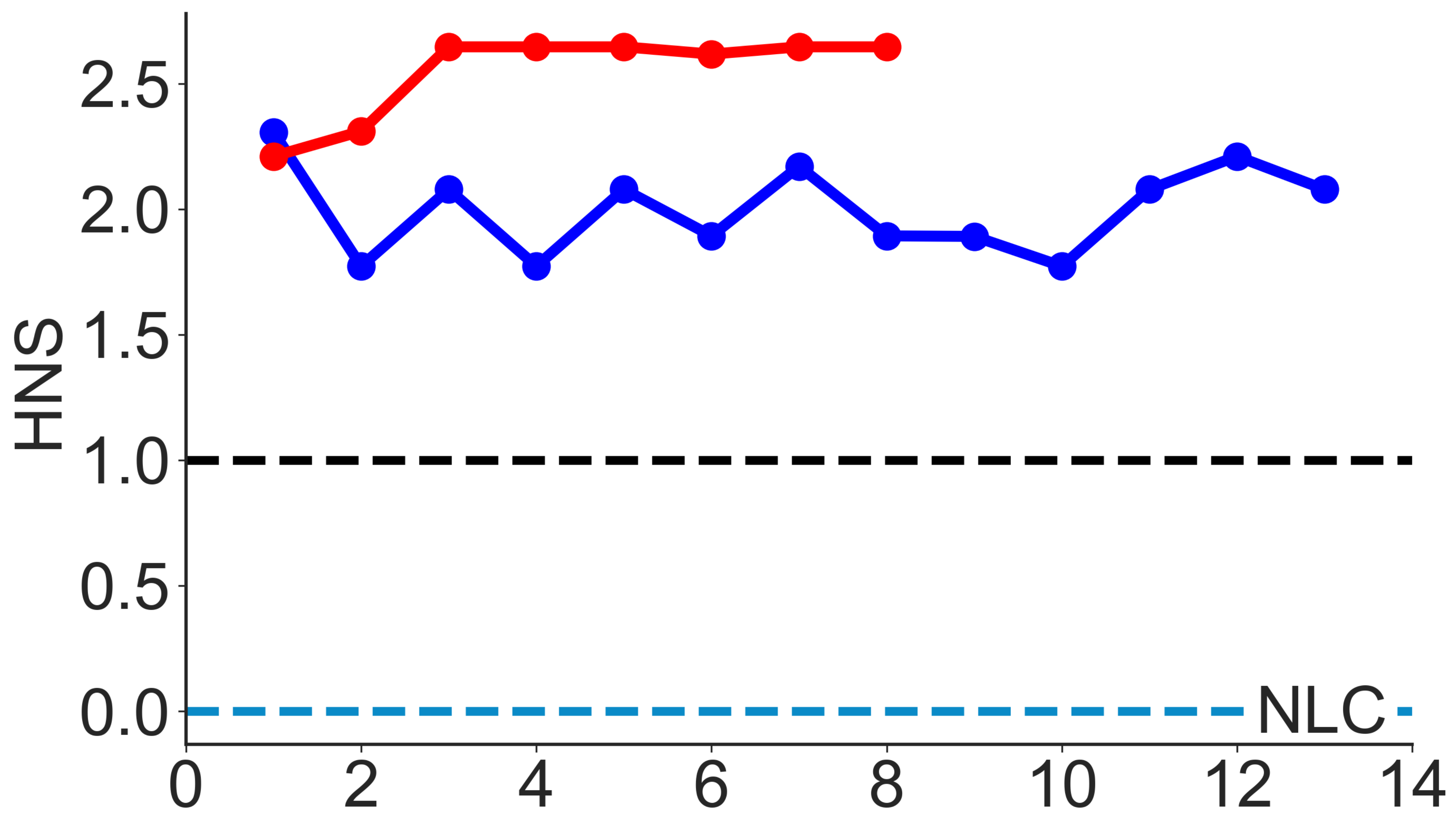}
    \caption{Humanoid}
    \label{fig:sub23}
  \end{subfigure}
  \caption{\textcolor{black}{The baseline (Eureka) exhibits undesirable performance degradation during evolutionary search on certain tasks (a-b). Notably, the oscillatory phenomenon is detected in the baseline method for Task (c), indicating substantial computational waste of the baseline method.} }
  \label{fig:degradation}
\end{figure}

In contrast, \method demonstrates superior optimization efficiency, requiring significantly fewer evolutionary iterations (NLC) while exhibiting markedly reduced instances of performance regression and oscillatory behavior during the optimization process. Our empirical results show performance regression in merely $1\%$ of Isaac tasks, with no observed cases of persistent optimization oscillations. These findings provide a strong empirical explanation for the effectiveness of the decoupled reward design approach implemented in \method. 

\textcolor{black}{
\textbf{Disc-2: The correlation between Reward Component Uncertainty and novel reward discovery}. 
Our investigation reveals a noteworthy phenomenon: the high-uncertainty reward components ($r_{u\uparrow}$) identified by \method frequently correspond to novel reward components not previously utilized in human-designed reward functions, with these components exhibiting significant reward shaping effects during RL training, thereby revealing uncertainty's dual role in enhancing both optimization efficiency and final policy performance. Here, we refer to a reward component as $r_{u\uparrow}$ when the uncertainty of the reward is greater than 0.9. Through systematic comparisons between reward components designed by the \method and conventional human-designed rewards, Figure~\ref{fig:novel_rewards} (a) illustrates a strong correlation between the component uncertainty level and the novelty, suggesting that higher uncertainty components represent more innovative reward formulations. Ablation studies conducted by removing $r_{u\uparrow}$ components from the reward function ($R$ w.o. $r_{u\uparrow}$) and retraining PPO agents reveal statistically significant performance degradation and poorer convergence characteristics in the resulting policies as shown in Figure~\ref{fig:novel_rewards} (b), whereas the complete reward function maintains substantially better optimization stability, collectively providing conclusive evidence that $r_{u\uparrow}$ components play an essential role in effective reward shaping and policy learning regularization. See more examples in App.~\ref{app:reward_shaping}. }

\begin{figure}
  \centering
    \begin{subfigure}[b]{0.45\textwidth}
    \includegraphics[width=\linewidth]{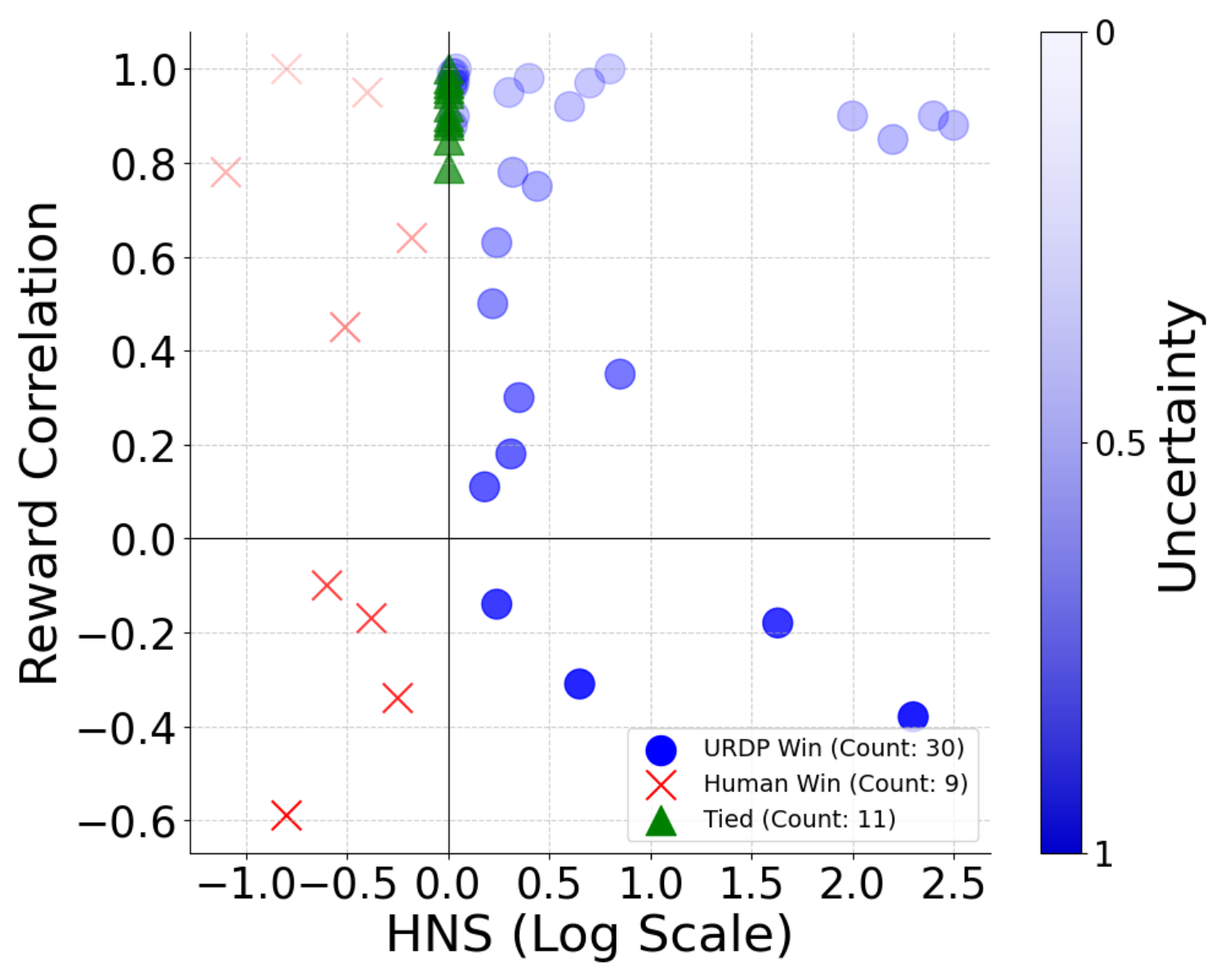}
    \caption{\textcolor{black}{\method Rewards vs. Human Rewards on IsaacGym.} }
    \end{subfigure}
    \begin{subfigure}[b]{0.47\textwidth}
    \includegraphics[width=\linewidth]{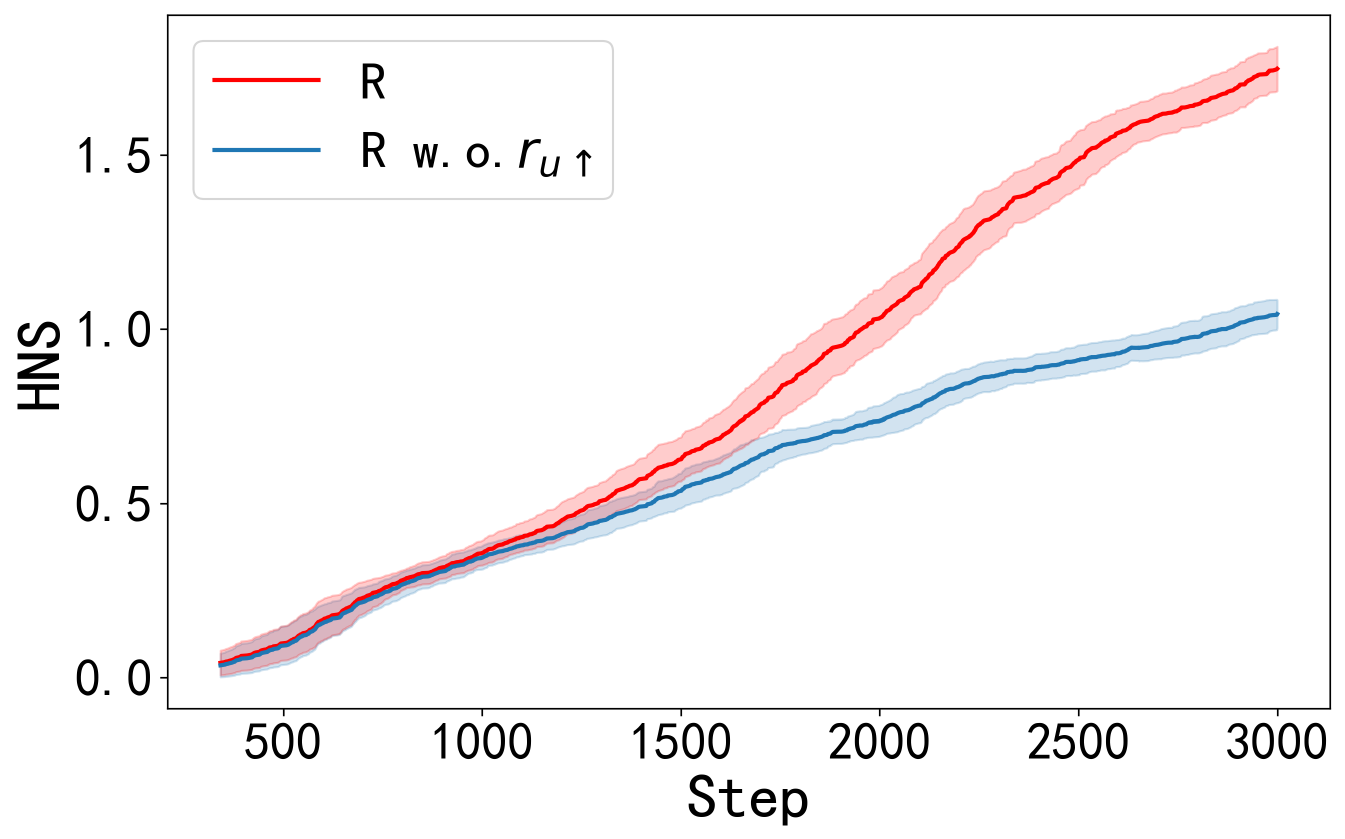}
    \caption{\textcolor{black}{RL training curves of reward functions R vs. R w.o. $r_{u\uparrow}$ on the ShadowHand.} }
    \end{subfigure}
  \caption{\textcolor{black}{(a) High-uncertainty reward components are likely novel reward components that humans have never explored before. (b) $r_{u\uparrow}$ is conducive to achieving higher returns.} }
  \label{fig:novel_rewards}
\end{figure}

\section{Conclusion}

In this work, we present a novel decoupled architecture for automated high-quality reward function design in reinforcement learning. Our framework introduces uncertainty quantification into reward component design, significantly improving the sampling efficiency of LLMs. Furthermore, we propose Uncertainty-Aware Bayesian Optimization to enable efficient hyperparameter search. Extensive experimental results demonstrate that our approach outperforms existing methods in both reward function quality and automated design efficiency.

\clearpage
\section*{Acknowledgments} 

This work has been supported by the program of National Natural Science Foundation of China (No. 61906195).

\bibliography{main}
\bibliographystyle{tmlr}

\appendix
\clearpage
\appendix
\usemintedstyle{colorful}

\section{Full Prompts}
\label{app:Full_Prompts}

In this section, we provide all prompts in the \method framework.

\begin{center}
Prompt 1: Initial system prompt
\end{center}
\vspace{-1.2em} 
{\color{gray!60!black}
\begin{minted}[
  fontsize=\scriptsize,
  bgcolor=lightgraybg,
  escapeinside=||,                     
  frame=single,           
  breaklines=true,
  breaksymbolleft={},
  tabsize=2
]{text}
You are a reward engineer trying to write reward functions to solve reinforcement learning tasks as effective as possible.
Your goal is to write a reward function for the environment that will help the agent learn the task described in text. 
Your reward function should use useful variables from the environment as inputs. As an example,
the reward function signature can be: 
@torch.jit.script
def compute_reward(object_pos: torch.Tensor, goal_pos: torch.Tensor) -> Tuple[torch.Tensor, Dict[str, torch.Tensor]]:
    ...
    return reward, {}
Since the reward function will be decorated with @torch.jit.script,
please make sure that the code is compatible with TorchScript (e.g., use torch tensor instead of numpy array). 
Make sure any new tensor or variable you introduce is on the same device as the input tensors. 
\end{minted}
}

\begin{center}
Prompt 2: Reward reflection and feedback
\end{center}
\vspace{-1.2em}

{\color{gray!60!black}
\begin{minted}[
  fontsize=\scriptsize,
  bgcolor=lightgraybg,
  escapeinside=||,                     
  frame=single,           
  breaklines=true,
  breaksymbolleft={},
  tabsize=2
]{text}
We trained a RL policy using the provided reward function code and tracked the values of the individual components in the reward function as well as global policy metrics such as success rates and episode lengths after every {epoch_freq} epochs and the maximum, mean, minimum values encountered:
<REWARD REFLECTION HERE1>

|\textbf{\color{black}We calculated a score for each sample based on the uncertainty of the reward term. We then calculated the standard and extreme deviations of all the sample scores in this iteration, which were as follows:}|
|\textbf{\color{black}<REWARD REFLECTION HERE2>}|

|{\color{black}\bfseries Please adopt the following recommendations for the next iteration of reward function generation:}|
    |\textcolor{black}{\textbf{(1) If the standard deviation is less than 0.05 and the extreme deviation is less than 0.1, it is recommended to stop the iteration and instead encourage exploration of new combinations of reward terms.}}|
    |\textcolor{black}{\textbf{(2) For all reward items with higher scores it is recommended to keep them and for those with lower scores it is recommended to remove them.}}|
    |\textcolor{black}{\textbf{(3) Only the combination of reward items and the content of the reward function need to be optimised, not the numerical optimisation.}}|

Please carefully analyze the policy feedback and provide a new, improved reward function that can better solve the task. Some helpful tips for analyzing the policy feedback:
    (1) If the success rates are always near zero, then you must rewrite the entire reward function
    (2) If the values for a certain reward component are near identical throughout, then this means RL is not able to optimize this component as it is written. You may consider
        (a) Changing its scale or the value of its temperature parameter
        (b) Re-writing the reward component 
        (c) Discarding the reward component
    (3) If some reward components' magnitude is significantly larger, then you must re-scale its value to a proper range
Please analyze each existing reward component in the suggested manner above first, and then write the reward function code. 
\end{minted}
}
\begin{center}
Prompt 3: Code formatting tip
\end{center}
\vspace{-1.2em} 
{\color{gray!60!black}
\begin{minted}[
  fontsize=\scriptsize,
  bgcolor=lightgraybg,
  escapeinside=||,                  
  frame=single,           
  breaklines=true,
  breaksymbolleft={},
  tabsize=2
]{text}
The output of the reward function should consist of two items:
    (1) the total reward,
    (2) a dictionary of each individual reward component.
The code output should be formatted as a python code string: "```python ... ```".

Some helpful tips for writing the reward function code:
    (1) You may find it helpful to normalize the reward to a fixed range by applying transformations like torch.exp to the overall reward or its components
    (2) If you choose to transform a reward component, then you must also introduce a temperature parameter inside the transformation function; this parameter must be a named variable in the reward function and it must not be an input variable. Each transformed reward component should have its own temperature variable
    (3) Make sure the type of each input variable is correctly specified; a float input variable should not be specified as torch.Tensor
    (4) Most importantly, the reward code's input variables must contain only attributes of the provided environment class definition (namely, variables that have prefix self.). Under no circumstance can you introduce new input variables. 
\end{minted}
}

\section{Benchmark Details}
\label{app:benchmark_details}

\subsection{An Introduction to the Benchmarks}

\textbf{Isaac}.The Isaac Gym benchmark includes a broad set of continuous control tasks, covering locomotion, balancing, aerial control, and dexterous manipulation. Robots range from low-DoF systems (e.g., Cartpole, Ball Balance) to complex agents such as Humanoid, Anymal, AllegroHand, and ShadowHand. Each task presents different control challenges, requiring precise joint coordination, stable gait generation, or fine-grained object interaction.
All tasks provide observations including joint positions, velocities, root orientation, and task-specific data such as object pose or goal location. The control mode varies by task—torque, velocity, or end-effector control—depending on the robot type. Randomization of initial states and physical parameters (e.g., mass, friction) is applied during training to improve robustness and generalization. The benchmark emphasizes both low-level motor control and high-level strategy in physics-rich environments.

\textbf{Dexterity}.The Dexterous benchmark focuses on dexterous manipulation using the 24-DoF ShadowHand across a wide range of object-centric tasks. These include stacking blocks, turning faucets, opening doors, rotating bottles, catching objects, and tool use. The tasks require precise finger control, contact-rich interactions, and adaptability to diverse object geometries and behaviors.
Each environment provides proprioceptive input (joint states, fingertip positions), as well as object-related observations (pose, velocity, goal state). Control is applied via joint position or velocity commands. The tasks involve significant variability in object placement, orientation, and physical properties, encouraging the development of general and robust manipulation policies. The benchmark highlights the challenge of high-dimensional motor coordination in real-world-like, unstructured settings. 

\textbf{Maniskill2}. In the ManiSkill2 environment, a 7-DoF Franka Panda robotic arm is used by default. For tasks focused on stationary manipulation—such as Lift Cube, Pick Cube, Turn Faucet, and Stack Cube—a fixed-base arm configuration is employed. In contrast, tasks involving mobility, such as Open Cabinet Door and Open Cabinet Drawer, utilize a single-arm robot mounted on a Sciurus17 mobile base. The Push Chair task is handled by a dual-arm system, also equipped with the Sciurus17 base.The observation space includes robot-centric data like joint angles, joint velocities, and the base's pose (position and orientation in the world frame), along with task-specific inputs such as goal coordinates and end-effector locations. Control is performed in end-effector delta pose mode, which directly manages changes in 3D translation and orientation, the latter expressed in axis-angle form relative to the end-effector's frame.Each task features variability in key parameters, including the initial and goal states of the object being manipulated, the robot’s starting joint configuration, and physical dynamics like friction and damping.

\subsection{ManiSkill2 Environment Details}
\label{app:maniskill2_details}
\textcolor{black}{
We provide detailed information about the ManiSkill2 environment in this section.  
Detailed information about the Isaac and Dexterity environments is the same as in the Eureka (see the content in the appendix of the paper~\cite{ma2024eureka}). 
For each environment, we list its observation and action dimensions, the original description of the task, and the task fitness function $F$.}

\begin{table}[h]
\centering
\begin{tabular}{@{}p{0.95\textwidth}@{}}
\toprule
\multicolumn{1}{c}{ManiSkill2 Environments} \\
\midrule

Environment (obs dim, action dim) \\
Task description  \\
Task fitness function $F$ \\
\midrule

\textcolor{black!70}{\texttt{PickCube-v0}} (51,7) \\
This class corresponds to the PickCube task in ManiSkill. This environment consists of a robot arm and a cube placed on the table. At the beginning, the cube appears at a random location and orientation. The agent must control the gripper to approach, grasp, and lift the cube above a threshold height. The challenge lies in object localization, precise control, and stable grasping. \\
\textcolor{black!70}{1[dist\_cube\_goal < 0.05]} \\
\midrule

\textcolor{black!70}{\texttt{LiftCube-v0}} (42,7) \\
This environment corresponds to the LiftCube task. The agent is required to grasp a cube and lift it vertically above a specific height threshold. The task emphasizes accurate vertical movement and stable grasping without disturbing the cube's pose.  \\
\textcolor{black!70}{1[cube\_height > 0.2]} \\
\midrule

\textcolor{black!70}{\texttt{TurnFaucet-v0}} (40, 7) \\
This class corresponds to the TurnFaucet task. A faucet handle is mounted on a wall, and the agent must rotate it clockwise or counterclockwise to a target angle. The challenge lies in establishing proper contact, applying sufficient torque, and maintaining stability during the turning motion.  \\
\textcolor{black!70}{1[rotation\_reward < 0.1]} \\
\midrule

\textcolor{black!70}{\texttt{OpenCabinetDoor-v1}} (75, 11) \\
This environment corresponds to the OpenCabinetDoor task. A cabinet with a side-hinged door is presented. The agent must locate and pull the door handle to open it. The task involves estimating the door's hinge axis, approaching from an appropriate angle, and applying a pulling force that aligns with the door's rotation. \\
\textcolor{black!70}{1[goal\_diff < 0.1 and is\_static]} \\
\midrule

\textcolor{black!70}{\texttt{OpenCabinetDrawer-v1}} (75, 11) \\
This class corresponds to the OpenCabinetDrawer task. The robot must open a drawer embedded in a cabinet by locating the handle and pulling it outward. The task requires both accurate handle grasping and force application along a linear trajectory, while avoiding excessive torque that could misalign the drawer. \\
\textcolor{black!70}{1[goal\_diff < 0.05 and is\_static]} \\
\midrule

\textcolor{black!70}{\texttt{PushChair-v1}} (131, 18) \\
This environment corresponds to the PushChair task. The robot must push a movable chair from its initial location to a designated target region. The chair is free to rotate and slide. The agent needs to make strategic contact with the chair body and adjust its pushing direction dynamically to avoid misalignment and ensure accurate placement. \\
\textcolor{black!70}{1[chair\_to\_target\_dist < 0.3 and chair\_tilt < 0.2]} \\

\bottomrule
\end{tabular}
\end{table}

\newpage
\section{Implementation Details}
\label{app:implementation_details}

\subsection{Implementation Details of Sampling and Uncertainty}
\label{app:impl_sampling}

\textcolor{black}{
When explicitly prompted to generate diverse outputs, LLMs inevitably produce varying textual expressions for semantically equivalent content - a phenomenon particularly evident in reward function generation where code implementations may differ lexically while encoding identical reward semantics. For instance, as demonstrated in Section~\ref{app:important_for_uncertainty}, two LLM-generated reward function samples might both incorporate velocity-based rewards while exhibiting completely different textual formulations. Failure to detect and eliminate such semantic redundancies leads to computationally expensive duplicate evaluations that cannot be effectively identified through surface-level text matching, necessitating deeper semantic analysis for accurate deduplication.}

\textcolor{black}{
Therefore, the \method utilizes the BGE-M3 model~\cite{xiao2024c} for the purpose of semantic similarity assessment, whereas Python's built-in SequenceMatcher is employed for text similarity assessment. 
The uncertainty quantification for both reward components and reward functions is implemented through similarity comparison. Specifically, the component uncertainty score ($U(r_{i,:})$) is computed by comparing a given reward component against all components generated within the same iteration. The reward function uncertainty ($U(R_i)$) score is derived through comparison with all functionally similar reward functions from the same iteration (referred to as a similarity group). From each similarity group, only one reward function is randomly selected for training, while the remaining ones are discarded, thereby filtering out redundant reward functions. See Alg.~\ref{alg: Ur_UR_B} for the pseudocode. We employ a similarity threshold of 0.95, where the final similarity metric is determined as the maximum value between semantic similarity and textual similarity.}

\begin{algorithm}[h]
\caption{Uncertainty Quantification in the \method}\label{alg: Ur_UR_B}
\KwIn{$K$ reward component samples \( \{r_{i,1}, r_{i,2}, \dots, r_{i,m}\}_{i \in K} \), text models \(S_{text}\) and semantic models \(S_{semantic}\).}
\KwOut{Reward components uncertainty \( \{U(r_{i,1}), \dots, U(r_{i,m}) \} \), reward functions uncertainty \( U(R_i) \) and similarity sample group \( B_i \).}
\ForEach{\(\ R_{i \in K} \) }{
    \ForEach{\(\ r_{i,1}, r_{i,2}, \dots, r_{i,m}\)}{
        \ForEach{\(\ r_{j,1}, r_{j,2}, \dots, r_{j\in k,m}, j>i\)}{
            \If{max(\(S_{text}(r_{i,m}, r_{j,m})\), \(S_{semantic}(r_{i,m},r_{j,m})\))>0.95}{
            count\(_m\)+1
            }
        }
        \(U(r_{i,m})\)=1-count\(_m\)/\(K\)
    }
    \(U(R_i)=U(r_i)/(U(r_i)+\dots+U(r_k))\) \\
    \ForEach{\(R_{j>i}\)}{
    \If{max(\(S_{text}(R_i, R_j)\), \(S_{semantic}(R_i, R_j)\))>0.95 and \(R_i\) not in other B}{
    Add \(R_j\) to \(B_i\)
    }
    }
}
\Return \( \{U(r_{i,1}), \dots, U(r_{i,m}) \} \), \( U(R_i) \),  \( B_{i,n=1} \).
\end{algorithm}

\subsection{Hyper-parameter Settings}
\label{app:Hyper-parameter_Settings}

\textcolor{black}{
All hyperparameters in \method are listed in Table~\ref{tab:hyper-parameter_urdp}. 
The reinforcement learning algorithms employed for validation maintain the default configurations specified for each respective environment, with all hyperparameters comprehensively documented in Tables \ref{tab:hyper-parameter_sac} and \ref{tab:hyper-parameter_ppo}. }

\begin{table}[htb]
    \centering
    \caption{Hyperparameters of \method. }
    \begin{tabular}{ll}
    \toprule
    \textbf{Hyper-parameter} & \textbf{Value} \\
    \midrule
    Quantity of the reward samples   $K$   & 16 \\
    Maximum \# of iterations $N_{outer}$   & 10 \\
    Baseline \# of iterations $N_{inner}$  & 10 \\
    maximum similarity $\omega$            & 0.95 \\
    \bottomrule
    \end{tabular}
    \label{tab:hyper-parameter_urdp}
\end{table}
    
\begin{table}[htb]
    \centering
    \caption{Hyperparameters of SAC algorithm applied to Maniskill2.}
    \begin{tabular}{ll}
    \toprule
    \textbf{Hyper-parameter} & \textbf{Value} \\
    \midrule
    Discount factor $\gamma$ & 0.95 \\
    Target update frequency & 1 \\
    Learning rate & $3 \times 10^{-4}$ \\
    Train frequency & 8  \\
    Soft update $\tau$ & $5 \times 10^{-3}$ \\
    Gradient steps & 4  \\
    Learning starts & 4000 \\
    Hidden units per layer & 256 \\
    Batch Size & 1024  \\
    \# of layers & 2  \\
    Initial temperature &  0.2  \\
    Rollout steps per episode &  200 \\
    \bottomrule
    \end{tabular}
    \label{tab:hyper-parameter_sac}
\end{table}
    
\vspace{1em}
    
\begin{table}[htb]
    \centering
    \caption{Hyperparameters of PPO algorithm applied to each task. }
    \begin{tabular}{ll}
    \toprule
    \textbf{Hyper-parameter} & \textbf{Value} \\
    \midrule
    Discount factor $\gamma$ & 0.99 (Isaac), 0.96 (Dexterity), 0.85 (ManiSkill2) \\
    \# of epochs per update & 8 (Isaac), 5 (Dexterity), 15 (ManiSkill2) \\
    Learning rate & $5 \times 10^{-4}$ (Isaac), $3 \times 10^{-4}$ (Dexterity, ManiSkill2) \\
    Batch size & 32768, 16384, 8192 (Isaac), 16384 (Dexterity), 400 (ManiSkill2) \\
    Target KL divergence & 0.008 (Isaac), 0.016 (Dexterity), 0.05 (ManiSkill2) \\
    \# of layers & 3 (Isaac, Dexterity), 2 (ManiSkill2) \\
    \# of steps per update & 16 (Isaac), 8 (Dexterity), 3200 (ManiSkill2) \\
    \bottomrule
    \end{tabular}
    \label{tab:hyper-parameter_ppo}
\end{table}

\section{Case Studies}
\label{app:examples}

\subsection{Case Study 1: LLMs in Numerical Optimization}
\label{app:example_llm_optimization}

This study employs two comparative examples to visualize the sdifferences between \method and Eureka in optimization processes. Example 1(a) and 1(b) present the respective design trajectories of \method and Eureka for the ShadowHand task, where red annotations denote reward components and \textcolor{black}{blue} text indicates reward intensity hyperparameters. 

Analysis of Example 1(a) demonstrates that during evolutionary search iterations, Eureka exclusively modifies reward intensity hyperparameters while failing to improve reward components. Despite multiple optimization attempts, this approach yields degraded performance. This finding reveals a critical limitation: when simultaneously optimizing both reward components and their strengths, LLMs cannot effectively utilize their inherent advantages in semantic correlation and autoregressive modeling, while their deficiencies in numerical optimization become particularly pronounced. 

In contrast, \method's decoupled alternating optimization demonstrates systematic effectiveness. As shown in Example 1(b), \method proactively explores novel reward components (e.g., replacing ``distance reward'' with ``velocity reward''), achieving a 70.3\% performance improvement (from 0.929 to 1.583) before hyperparameter optimization. Subsequent Bayesian optimization further enhances performance with an average 22.1\% gain, culminating in a 102.5\% total improvement over the initial configuration. Example 2 confirms this consistent advantage, providing compelling visual evidence of \method's superior optimization characteristics.

\begin{center}
\textcolor{black}{Example 1(a): Eureka reward functions on ShadowHand (Iteration 2, HNS: 1.045; Iteration 3, HNS: 0.898; Iteration 4, HNS: 1.019; Iteration 5 HNS: 0.955).}
\end{center}
\vspace{-0.5em} 
\begin{minted}[
  fontsize=\scriptsize,
  bgcolor=lightgraybg,
  escapeinside=@@,         
  frame=single,
  breaklines=true,
  breaksymbolleft={},
  linenos
]{python}
Iteration @\textcolor{black}{2}@:
def @\textcolor{black}{compute\_reward}@(object_rot: torch.Tensor, goal_rot: torch.Tensor, object_angvel: torch.Tensor) -> Tuple[torch.Tensor, Dict[str, torch.Tensor]]:
    # Temperature parameters for reward components
    @\textcolor{blue}{orientation\_temp: float = 5.0}@  # Increased sensitivity for orientation alignment
    @\textcolor{blue}{angular\_velocity\_temp: float = 0.5}@  # Adjusted sensitivity for angular velocity
    
    # Compute the difference in orientation between the object and the goal
    orientation_diff = @\textcolor{black}{1.0}@ - torch.abs(torch.sum(object_rot * goal_rot, dim=@\textcolor{black}{-1}@))  # Quaternion dot product
    
    # Reward for minimizing the orientation difference
    orientation_reward = torch.exp(-orientation_temp * orientation_diff)
    
    # Reward for spinning the object (encouraging non-zero angular velocity)
    angular_velocity_magnitude = torch.norm(object_angvel, dim=@\textcolor{black}{-1}@)
    angular_velocity_reward = torch.tanh(angular_velocity_temp * angular_velocity_magnitude)  # Rescaled to [-1, 1]
    
    # Combine the rewards with a focus on orientation alignment
    total_reward = orientation_reward * (@\textcolor{black}{1.0}@ + angular_velocity_reward)  # Weighted combination
    
    # Return the total reward and individual components
    reward_components = {
        "orientation_reward": @\textcolor{red}{orientation\_reward}@,
        "angular_velocity_reward": @\textcolor{red}{angular\_velocity\_reward}@
    }
    return total_reward, reward_components

########################################################################################################

Iteration @\textcolor{black}{3}@:
def @\textcolor{black}{compute\_reward}@(object_rot: torch.Tensor, goal_rot: torch.Tensor, object_angvel: torch.Tensor) -> Tuple[torch.Tensor, Dict[str, torch.Tensor]]:
    # Temperature parameters for reward components
    @\textcolor{blue}{orientation\_temp: float = 10.0}@  # Increased sensitivity for orientation alignment
    @\textcolor{blue}{angular\_velocity\_temp: float = 1.0}@  # Adjusted sensitivity for angular velocity
    
    # Compute the difference in orientation between the object and the goal
    orientation_diff = @\textcolor{black}{1.0}@ - torch.abs(torch.sum(object_rot * goal_rot, dim=@\textcolor{black}{-1}@))  # Quaternion dot product
    
    # Reward for minimizing the orientation difference
    orientation_reward = torch.exp(-orientation_temp * orientation_diff)
    
    # Reward for spinning the object (encouraging non-zero angular velocity)
    angular_velocity_magnitude = torch.norm(object_angvel, dim=@\textcolor{black}{-1}@)
    angular_velocity_reward = torch.tanh(angular_velocity_temp * angular_velocity_magnitude)  # Rescaled to [-1, 1]
    
    # Combine the rewards with a focus on orientation alignment
    total_reward = orientation_reward * (@\textcolor{black}{1.0}@ + angular_velocity_reward)  # Weighted combination
    
    # Return the total reward and individual components
    reward_components = {
        "orientation_reward": @\textcolor{red}{orientation\_reward}@,
        "angular_velocity_reward": @\textcolor{red}{angular\_velocity\_reward}@
    }
    return total_reward, reward_components

########################################################################################################

Iteration @\textcolor{black}{4}@:
def @\textcolor{black}{compute\_reward}@(object_rot: torch.Tensor, goal_rot: torch.Tensor, object_angvel: torch.Tensor) -> Tuple[torch.Tensor, Dict[str, torch.Tensor]]:
    # Temperature parameters for reward components
    @\textcolor{blue}{orientation\_temp: float = 10.0}@  # Increased sensitivity for orientation alignment
    @\textcolor{blue}{angular\_velocity\_temp: float = 2.0}@  # Adjusted sensitivity for angular velocity
    
    # Compute the difference in orientation between the object and the goal
    orientation_diff = @\textcolor{black}{1.0}@ - torch.abs(torch.sum(object_rot * goal_rot, dim=@\textcolor{black}{-1}@))  # Quaternion dot product
    
    # Reward for minimizing the orientation difference
    orientation_reward = torch.exp(-orientation_temp * orientation_diff)
    
    # Reward for spinning the object (encouraging non-zero angular velocity)
    angular_velocity_magnitude = torch.norm(object_angvel, dim=@\textcolor{black}{-1}@)
    angular_velocity_reward = torch.tanh(angular_velocity_temp * angular_velocity_magnitude)  # Rescaled to [-1, 1]
    
    # Combine the rewards with a focus on orientation alignment
    total_reward = orientation_reward * (@\textcolor{black}{1.0}@ + angular_velocity_reward)  # Weighted combination
    
    # Return the total reward and individual components
    reward_components = {
        "orientation_reward": @\textcolor{red}{orientation\_reward}@,
        "angular_velocity_reward": @\textcolor{red}{angular\_velocity\_reward}@
    }
    return total_reward, reward_components

########################################################################################################

Iteration @\textcolor{black}{5}@:
def @\textcolor{black}{compute\_reward}@(object_rot: torch.Tensor, goal_rot: torch.Tensor, object_angvel: torch.Tensor) -> Tuple[torch.Tensor, Dict[str, torch.Tensor]]:
    # Temperature parameters for reward components
    @\textcolor{blue}{orientation\_temp: float = 50.0}@  # Increased sensitivity for orientation alignment
    @\textcolor{blue}{angular\_velocity\_temp: float = 2.0}@  # Adjusted sensitivity for angular velocity
    
    # Compute the difference in orientation between the object and the goal
    orientation_diff = @\textcolor{black}{1.0}@ - torch.abs(torch.sum(object_rot * goal_rot, dim=@\textcolor{black}{-1}@))  # Quaternion dot product
    
    # Reward for minimizing the orientation difference
    orientation_reward = torch.exp(-orientation_temp * orientation_diff)
    
    # Reward for spinning the object (encouraging non-zero angular velocity)
    angular_velocity_magnitude = torch.norm(object_angvel, dim=@\textcolor{black}{-1}@)
    angular_velocity_reward = torch.tanh(angular_velocity_temp * angular_velocity_magnitude)  # Rescaled to [-1, 1]
    
    # Combine the rewards with a focus on orientation alignment
    total_reward = orientation_reward * (@\textcolor{black}{1.0}@ + angular_velocity_reward)  # Weighted combination
    
    # Return the total reward and individual components
    reward_components = {
        "orientation_reward": @\textcolor{red}{orientation\_reward}@,
        "angular_velocity_reward": @\textcolor{red}{angular\_velocity\_reward}@
    }
    return total_reward, reward_components
\end{minted}

\begin{center}
\textcolor{black}{Example 1(b): \method reward functions on ShadowHand. Iteration 1 (before UABO), HNS: 0.929; Iteration 1 (after UABO), HNS: 1.165; Iteration 2 (before UABO), HNS: 1.583; Iteration 2 (after UABO), HNS: 1.882. }
\end{center}
\vspace{-0.84em} 
\begin{minted}[
  fontsize=\scriptsize,
  bgcolor=lightgraybg,
  escapeinside=@@,       
  frame=single,
  breaklines=true,
  breaksymbolleft={},
  linenos
]{python}
Iteration @\textcolor{black}{1}@(before UABO):
def @\textcolor{black}{compute\_reward}@(object_rot: torch.Tensor, goal_rot: torch.Tensor) -> Tuple[torch.Tensor, Dict[str, torch.Tensor]]:
    # Temperature parameters
    @\textcolor{blue}{orientation\_temp = 1.0}@  # Reduced temperature for better sensitivity
    @\textcolor{blue}{distance\_temp = 1.0}@     # Temperature for distance-based reward

    # Compute the dot product between the object and goal quaternions
    dot_product = torch.sum(object_rot * goal_rot, dim=@\textcolor{black}{1}@)

    # Ensure the dot product is within the valid range [-1, 1]
    dot_product = torch.clamp(dot_product, @\textcolor{black}{-1.0, 1.0}@)

    # Compute the angle difference between the quaternions
    angle_diff = torch.acos(@\textcolor{black}{2.0}@ * dot_product**@\textcolor{black}{2}@ - @\textcolor{black}{1.0}@)

    # Orientation reward: exponential transformation of the angle difference
    orientation_reward = torch.exp(-orientation_temp * angle_diff)

    # Distance-based reward: encourages reducing the angle difference
    distance_reward = -angle_diff  # Negative because we want to minimize the difference

    # Success bonus: reward for achieving the target orientation
    success_threshold: float = @\textcolor{black}{0.05}@  # Easier threshold for success
    success_bonus = torch.where(angle_diff < success_threshold, @\textcolor{black}{100.0, 0.0}@)  # Larger bonus

    # Total reward: weighted sum of orientation reward, distance reward, and success bonus
    total_reward = orientation_reward + distance_reward + success_bonus

    # Dictionary of individual reward components
    reward_components = {
        "orientation_reward": @\textcolor{red}{orientation\_reward}@,
        "distance_reward": @\textcolor{red}{distance\_reward}@,
    }

    return total_reward, reward_components

########################################################################################################

Iteration @\textcolor{black}{1}@(after UABO):
def @\textcolor{black}{compute\_reward}@(object_rot: torch.Tensor, goal_rot: torch.Tensor) -> Tuple[torch.Tensor, Dict[str, torch.Tensor]]:
    # Temperature parameters
    @\textcolor{blue}{orientation\_temp = 1.1134}@  # Reduced temperature for better sensitivity
    @\textcolor{blue}{distance\_temp = 1.1134}@     # Temperature for distance-based reward

    # Compute the dot product between the object and goal quaternions
    dot_product = torch.sum(object_rot * goal_rot, dim=@\textcolor{black}{1}@)

    # Ensure the dot product is within the valid range [-1, 1]
    dot_product = torch.clamp(dot_product, @\textcolor{black}{-1.0, 1.0}@)

    # Compute the angle difference between the quaternions
    angle_diff = torch.acos(@\textcolor{black}{2.0}@ * dot_product**@\textcolor{black}{2}@ - @\textcolor{black}{1.0}@)

    # Orientation reward: exponential transformation of the angle difference
    orientation_reward = torch.exp(-orientation_temp * angle_diff)

    # Distance-based reward: encourages reducing the angle difference
    distance_reward = -angle_diff  # Negative because we want to minimize the difference

    # Success bonus: reward for achieving the target orientation
    success_threshold: float = @\textcolor{black}{0.05}@  # Easier threshold for success
    success_bonus = torch.where(angle_diff < success_threshold, @\textcolor{black}{100.0, 0.0}@)  # Larger bonus

    # Total reward: weighted sum of orientation reward, distance reward, and success bonus
    total_reward = orientation_reward + distance_reward + success_bonus

    # Dictionary of individual reward components
    reward_components = {
        "orientation_reward": @\textcolor{red}{orientation\_reward}@,
        "distance_reward": @\textcolor{red}{distance\_reward}@
    }

    return total_reward, reward_components

########################################################################################################

Iteration @\textcolor{black}{2}@(before UABO):
def @\textcolor{black}{compute\_reward}@(object_rot: torch.Tensor, goal_rot: torch.Tensor, object_angvel: torch.Tensor) -> Tuple[torch.Tensor, Dict[str, torch.Tensor]]:
    # Temperature parameters for reward components
    @\textcolor{blue}{orientation\_temp = 1}@
    @\textcolor{blue}{velocity\_temp = 0.1}@
    
    # Compute the difference in orientation between the object and the goal
    orientation_diff = torch.norm(object_rot - goal_rot, dim=@\textcolor{black}{-1}@)
    
    # Compute the angular velocity magnitude of the object
    angvel_magnitude = torch.norm(object_angvel, dim=@\textcolor{black}{-1}@)
    
    # Reward for minimizing the orientation difference
    orientation_reward = torch.exp(-orientation_temp * orientation_diff)
    
    # Reward for maintaining a high angular velocity (encourages spinning)
    velocity_reward = torch.exp(-velocity_temp * (@\textcolor{black}{1.0}@ / (angvel_magnitude + @\textcolor{black}{1e-6}@)))
    
    # Combine the rewards with appropriate weights
    total_reward = @\textcolor{black}{0.7}@ * orientation_reward + @\textcolor{black}{0.3}@ * velocity_reward
    
    # Dictionary of individual reward components for logging
    reward_dict = {
        "orientation_reward": @\textcolor{red}{orientation\_reward}@,
        "velocity_reward": @\textcolor{red}{velocity\_reward}@
    }
    
    return total_reward, reward_dict

########################################################################################################

Iteration @\textcolor{black}{2}@(after UABO):
def @\textcolor{black}{compute\_reward}@(object_rot: torch.Tensor, goal_rot: torch.Tensor, object_angvel: torch.Tensor) -> Tuple[torch.Tensor, Dict[str, torch.Tensor]]:
    # Temperature parameters for reward components
    @\textcolor{blue}{orientation\_temp = 1.2642}@
    @\textcolor{blue}{velocity\_temp = 0.3145}@
    
    # Compute the difference in orientation between the object and the goal
    orientation_diff = torch.norm(object_rot - goal_rot, dim=@\textcolor{black}{-1}@)
    
    # Compute the angular velocity magnitude of the object
    angvel_magnitude = torch.norm(object_angvel, dim=@\textcolor{black}{-1}@)
    
    # Reward for minimizing the orientation difference
    orientation_reward = torch.exp(-orientation_temp * orientation_diff)
    
    # Reward for maintaining a high angular velocity (encourages spinning)
    velocity_reward = torch.exp(-velocity_temp * (@\textcolor{black}{1.0}@ / (angvel_magnitude + @\textcolor{black}{1e-6}@)))
    
    # Combine the rewards with appropriate weights
    total_reward = @\textcolor{black}{0.7}@ * orientation_reward + @\textcolor{black}{0.3}@ * velocity_reward
    
    # Dictionary of individual reward components for logging
    reward_dict = {
        "orientation_reward": @\textcolor{red}{orientation\_reward}@,
        "velocity_reward": @\textcolor{red}{velocity\_reward}@
    }
    
    return total_reward, reward_dict
\end{minted}

\begin{center}
\textcolor{black}{Example 2 (a): Eureka reward functions on Humanoid (Iteration 1, HNS: 2.273; Iteration 2, score: 0.412; Iteration 3, HNS: 0.032; Iteration 4, HNS: 0.127).}
\end{center}
\vspace{-0.84em} 
\begin{minted}[
  fontsize=\scriptsize,
  bgcolor=lightgraybg,
  escapeinside=@@,          
  frame=single,
  breaklines=true,
  breaksymbolleft={},
  linenos
]{python}
Iteration @\textcolor{black}{1}@:
def @\textcolor{black}{compute\_reward}@(root_states: torch.Tensor, dt: float) -> Tuple[torch.Tensor, Dict[str, torch.Tensor]]:
    # Extract the velocity of the humanoid's torso (root_states[:, 7:10] contains the linear velocity)
    velocity = root_states[:, @\textcolor{black}{7:10}@]
    
    # Compute the forward speed (we assume the humanoid is moving along the x-axis)
    forward_speed = velocity[:, @\textcolor{black}{0}@]
    
    # Reward for forward speed (scaled to a smaller range)
    @\textcolor{blue}{speed\_temp = 1.0}@  # Reduced temperature for better scaling
    speed_reward = forward_speed * speed_temp
    
    # Reward for consistency (encourage maintaining high speed)
    @\textcolor{blue}{consistency\_temp = 1.0}@
    consistency_reward = torch.exp(-consistency_temp * torch.abs(forward_speed - torch.mean(forward_speed)))
    
    # Total reward combines speed reward and consistency reward
    reward = speed_reward + consistency_reward
    
    # Dictionary of individual reward components
    reward_dict = {
        "speed_reward": @\textcolor{red}{speed\_reward}@,
        "consistency_reward": @\textcolor{red}{consistency\_reward}@,
    }
    
    return reward, reward_dict

########################################################################################################

Iteration @\textcolor{black}{2}@:
def @\textcolor{black}{compute\_reward}@(root_states: torch.Tensor, dt: float) -> Tuple[torch.Tensor, Dict[str, torch.Tensor]]:
    # Extract the velocity of the humanoid's torso (root_states[:, 7:10] contains the linear velocity)
    velocity = root_states[:, @\textcolor{black}{7:10}@]
    
    # Compute the forward speed (we assume the humanoid is moving along the x-axis)
    forward_speed = velocity[:, @\textcolor{black}{0}@]
    
    # Reward for forward speed (scaled to a smaller range)
    @\textcolor{blue}{speed\_temp = 0.2}@  # Further reduced temperature for better scaling
    speed_reward = forward_speed * speed_temp
    
    # Reward for consistency (encourage maintaining high speed over time)
    @\textcolor{blue}{consistency\_temp = 0.5}@  # Increased temperature for better scaling
    consistency_reward = torch.exp(-consistency_temp * torch.abs(forward_speed - torch.mean(forward_speed)))
    
    # Total reward combines speed reward and consistency reward
    reward = speed_reward + consistency_reward
    
    # Dictionary of individual reward components
    reward_dict = {
        "speed_reward": @\textcolor{red}{speed\_reward}@,
        "consistency_reward": @\textcolor{red}{consistency\_reward}@,
    }
    
    return reward, reward_dict

########################################################################################################

Iteration @\textcolor{black}{3}@:
def @\textcolor{black}{compute\_reward}@(root_states: torch.Tensor, dt: float) -> Tuple[torch.Tensor, Dict[str, torch.Tensor]]:
    # Extract the velocity of the humanoid's torso (root_states[:, 7:10] contains the linear velocity)
    velocity = root_states[:, @\textcolor{black}{7:10}@]
    
    # Compute the forward speed (we assume the humanoid is moving along the x-axis)
    forward_speed = velocity[:, @\textcolor{black}{0}@]
    
    # Reward for forward speed (scaled to a smaller range)
    @\textcolor{blue}{speed\_temp = 0.1}@  # Reduced temperature for better scaling
    speed_reward = forward_speed * speed_temp
    
    # Reward for consistency (encourage maintaining high speed)
    @\textcolor{blue}{consistency\_temp = 0.5}@  # Increased temperature for better scaling
    consistency_reward = torch.exp(-consistency_temp * torch.abs(forward_speed - torch.mean(forward_speed)))
    
    # Total reward combines speed reward and consistency reward
    reward = speed_reward + consistency_reward
    
    # Dictionary of individual reward components
    reward_dict = {
        "speed_reward": @\textcolor{red}{speed\_reward}@,
        "consistency_reward": @\textcolor{red}{consistency\_reward}@,
    }
    
    return reward, reward_dict

########################################################################################################

Iteration @\textcolor{black}{4}@:
def @\textcolor{black}{compute\_reward}@(root_states: torch.Tensor, dt: float) -> Tuple[torch.Tensor, Dict[str, torch.Tensor]]:
    # Extract the velocity of the humanoid's torso (root_states[:, 7:10] contains the linear velocity)
    velocity = root_states[:, @\textcolor{black}{7:10}@]
    
    # Compute the forward speed (we assume the humanoid is moving along the x-axis)
    forward_speed = velocity[:, @\textcolor{black}{0}@]
    
    # Reward for forward speed (scaled to a smaller range)
    @\textcolor{blue}{speed\_temp = 0.1}@  # Further reduced temperature for better scaling
    speed_reward = forward_speed * speed_temp
    
    # Reward for consistency (encourage maintaining high speed over time)
    @\textcolor{blue}{consistency\_temp = 10}@  # Increased temperature for better scaling
    consistency_reward = torch.exp(-consistency_temp * torch.abs(forward_speed - torch.mean(forward_speed)))
    
    # Total reward combines speed reward and consistency reward
    reward = speed_reward + consistency_reward
    
    # Dictionary of individual reward components
    reward_dict = {
        "speed_reward": @\textcolor{red}{speed\_reward}@,
        "consistency_reward": @\textcolor{red}{consistency\_reward}@,
    }
    
    return reward, reward_dict
\end{minted}

\begin{center}
Example 2 (b):  \method reward functions on Humanoid. Iteration 1 (before UABO), HNS: 1.428, Iteration 1 (after UABO), HNS: 2.132; Iteration 2 (before UABO), HNS: 2.481, Iteration 2 (after UABO), HNS: 2.646.
\end{center}
\vspace{-0.84em} 
\begin{minted}[
  fontsize=\scriptsize,
  bgcolor=lightgraybg,
  escapeinside=@@,          
  frame=single,
  breaklines=true,
  breaksymbolleft={},
  linenos
]{python}
Iteration @\textcolor{black}{1}@(before UABO):
def @\textcolor{black}{compute\_reward}@(root_states: torch.Tensor, dt: float) -> Tuple[torch.Tensor, Dict[str, torch.Tensor]]:
    # Extract the velocity of the humanoid's torso from the root states
    velocity = root_states[:, @\textcolor{black}{7:10}@]
    
    # Compute the forward speed (assuming x-axis is the forward direction)
    forward_speed = velocity[:, @\textcolor{black}{0}@]
    
    # Scale up the speed reward to provide stronger feedback
    @\textcolor{blue}{speed\_reward\_temp = 10}@  # Temperature parameter for speed reward
    speed_reward = torch.exp(speed_reward_temp * forward_speed * dt)
    
    # Optional: Add a reward for sustained running
    @\textcolor{blue}{sustained\_speed\_temp = 0.1}@  # Temperature parameter for sustained speed reward
    sustained_speed_reward = torch.exp(sustained_speed_temp * forward_speed * dt)
    
    # Combine the rewards
    total_reward = speed_reward + sustained_speed_reward
    
    # Return the total reward and the individual reward components
    reward_dict = {
        "speed_reward": @\textcolor{red}{speed\_reward}@,
        "sustained_speed_reward": @\textcolor{red}{sustained\_speed\_reward}@
    }
    
    return total_reward, reward_dict

########################################################################################################

Iteration @\textcolor{black}{1}@(after UABO):
def @\textcolor{black}{compute\_reward}@(root_states: torch.Tensor, dt: float) -> Tuple[torch.Tensor, Dict[str, torch.Tensor]]:
    # Extract the velocity of the humanoid's torso from the root states
    velocity = root_states[:, @\textcolor{black}{7:10}@]
    
    # Compute the forward speed (assuming x-axis is the forward direction)
    forward_speed = velocity[:, @\textcolor{black}{0}@]
    
    # Scale up the speed reward to provide stronger feedback
    @\textcolor{blue}{speed\_reward\_temp = 10.4424}@  # Temperature parameter for speed reward
    speed_reward = torch.exp(speed_reward_temp * forward_speed * dt)
    
    # Optional: Add a reward for sustained running
    @\textcolor{blue}{sustained\_speed\_temp = 0.3606}@  # Temperature parameter for sustained speed reward
    sustained_speed_reward = torch.exp(sustained_speed_temp * forward_speed * dt)
    
    # Combine the rewards
    total_reward = speed_reward + sustained_speed_reward
    
    # Return the total reward and the individual reward components
    reward_dict = {
        "speed_reward": @\textcolor{red}{speed\_reward}@,
        "sustained_speed_reward": @\textcolor{red}{sustained\_speed\_reward}@
    }
    
    return total_reward, reward_dict

########################################################################################################

Iteration @\textcolor{black}{2}@(before UABO):
def @\textcolor{black}{compute\_reward}@(root_states: torch.Tensor, dt: float) -> Tuple[torch.Tensor, Dict[str, torch.Tensor]]:
    # Extract the velocity of the humanoid's torso from the root states
    velocity = root_states[:, @\textcolor{black}{7:10}@]
    torso_height = root_states[:, @\textcolor{black}{2}@]  # Torso height (z-axis)
    
    # Compute the forward speed (assuming x-axis is the forward direction)
    forward_speed = velocity[:, 0]
    
    # Apply exponential transformation to the speed reward with a temperature parameter
    @\textcolor{blue}{speed\_temp = 1}@  # Temperature parameter for speed reward
    speed_reward = torch.exp(speed_temp * forward_speed) * dt
    
    # Penalty for falling (torso height below a threshold)
    fall_threshold = @\textcolor{black}{0.8}@  # Minimum torso height to avoid falling
    fall_penalty = @\textcolor{black}{-10.0}@ * (torso_height < fall_threshold).float() * dt
    
    # Combine the rewards
    total_reward = speed_reward + fall_penalty
    
    # Return the total reward and the individual reward components
    reward_dict = {
        "speed_reward": @\textcolor{red}{speed\_reward}@
    }
    
    return total_reward, reward_dict

########################################################################################################

Iteration @\textcolor{black}{2}@(after UABO):
def @\textcolor{black}{compute\_reward}@(root_states: torch.Tensor, dt: float) -> Tuple[torch.Tensor, Dict[str, torch.Tensor]]:
    # Extract the velocity of the humanoid's torso from the root states
    velocity = root_states[:, @\textcolor{black}{7:10}@]
    torso_height = root_states[:, @\textcolor{black}{2}@]  # Torso height (z-axis)
    
    # Compute the forward speed (assuming x-axis is the forward direction)
    forward_speed = velocity[:, @\textcolor{black}{0}@]
    
    # Apply exponential transformation to the speed reward with a temperature parameter
    @\textcolor{blue}{speed\_temp = 0.7111}@  # Temperature parameter for speed reward
    speed_reward = torch.exp(speed_temp * forward_speed) * dt
    
    # Penalty for falling (torso height below a threshold)
    fall_threshold = @\textcolor{black}{0.8}@  # Minimum torso height to avoid falling
    fall_penalty = @\textcolor{black}{-10.0}@ * (torso_height < fall_threshold).float() * dt
    
    # Combine the rewards
    total_reward = speed_reward + fall_penalty
    
    # Return the total reward and the individual reward components
    reward_dict = {
        "speed_reward": @\textcolor{red}{speed\_reward}@
    }
    
    return total_reward, reward_dict
\end{minted}

\subsection{Case Study 2: the Redundancy in Reward Function Samples} 
\label{app:important_for_uncertainty}

\textcolor{black}{
This case study analyzes two representative examples from Eureka within a single iteration, visually demonstrating a major source of computational redundancy. 
Our analysis reveals that Eureka generates multiple semantically equivalent but syntactically varied reward functions within a single iteration, all sharing identical reward intensities.}

For instance, in Example 3, although the two reward functions (Sample $\#9$ and $\#15$) from iteration 1 exhibit different textual expressions (``\textit{forward\_velocity\_reward}'' vs. ``\textit{velocity\_reward}''), their underlying reward objectives and logic are fundamentally identical. This observation suggests that the apparent diversity among Eureka-generated samples may be primarily lexical rather than semantic. Effective filtering of such pseudo-diversity is therefore essential to eliminate redundant and computationally inefficient evaluations. Example 4 is a similar example.

\begin{center}
\textcolor{black}{Example 3: The reward function codes of the task Ant after Iteration 1. The scores of Sample $\#9$ and Sample  $\#15$ are both 2.012.}
\end{center}
\vspace{-0.85em} 
\begin{minted}[
  fontsize=\scriptsize,
  bgcolor=lightgraybg,
  escapeinside=@@,           
  frame=single,
  breaklines=true,
  breaksymbolleft={},
  linenos
]{python}
Iteration @\textcolor{black}{1}@ (sample #9):
def @\textcolor{black}{compute\_reward}@(root_states: torch.Tensor, actions: torch.Tensor) -> Tuple[torch.Tensor, Dict[str, torch.Tensor]]:

    # Define weight parameters
    @\textcolor{blue}{forward\_velocity\_temp: float = 1.0}@
    @\textcolor{blue}{action\_penalty\_temp: float = 0.01}@
    
    # Extract forward velocity (x-axis velocity in the world frame)
    forward_velocity = root_states[:, @\textcolor{black}{7}@]  # Velocity along the x-axis
    
    # Reward for forward velocity
    forward_velocity_reward = forward_velocity * forward_velocity_temp
    
    # Penalty for large actions to encourage energy efficiency
    action_penalty = -torch.sum(torch.square(actions), dim=@\textcolor{black}{-1}@) * action_penalty_temp
    
    # Total reward
    reward = forward_velocity_reward + action_penalty
    
    # Individual reward components
    reward_dict = {
        "forward_velocity_reward": @\textcolor{red}{forward\_velocity\_reward}@,
        "action_penalty": @\textcolor{red}{action\_penalty}@,
    }
    
    return reward, reward_dict

########################################################################################################

Iteration @\textcolor{black}{1}@ (sample #15):
def @\textcolor{black}{compute\_reward}@(root_states: torch.Tensor, actions: torch.Tensor) -> Tuple[torch.Tensor, Dict[str, torch.Tensor]]:

    # Define weight parameters
    @\textcolor{blue}{velocity\_temp: float = 1.0}@
    @\textcolor{blue}{action\_penalty\_temp: float = 0.01}@
    
    # Extract the forward velocity (x-axis velocity)
    forward_velocity = root_states[:, @\textcolor{black}{7}@]  # x-axis velocity is at index 7
    
    # Reward for moving forward fast
    velocity_reward = forward_velocity * velocity_temp
    
    # Penalty for large actions to encourage energy efficiency
    action_penalty = -torch.sum(torch.square(actions), dim=@\textcolor{black}{-1}@) * action_penalty_temp
    
    # Total reward
    total_reward = velocity_reward + action_penalty
    
    # Individual reward components for debugging and analysis
    reward_components = {
        "velocity_reward": @\textcolor{red}{velocity\_reward}@,
        "action_penalty": @\textcolor{red}{action\_penalty}@,
    }
    
    return total_reward, reward_components
\end{minted}

\begin{center}
\textcolor{black}{Example 4: The reward function codes of the task Ant after Iteration 1. The scores of Sample $\#3$ and Sample $\#14$ are both 0.059.} 
\end{center}
\vspace{-0.85em} 
\begin{minted}[
  fontsize=\scriptsize,
  bgcolor=lightgraybg,
  escapeinside=@@,             
  frame=single,
  breaklines=true,
  breaksymbolleft={},
  linenos
]{python}
Iteration @\textcolor{black}{1}@ (sample #3):
def @\textcolor{black}{compute\_reward}@(object_rot: torch.Tensor, goal_rot: torch.Tensor, object_angvel: torch.Tensor) -> Tuple[torch.Tensor, Dict[str, torch.Tensor]]:
    # Temperature parameters for reward components
    @\textcolor{blue}{orientation\_temp: float = 1.0}@
    @\textcolor{blue}{angular\_vel\_temp: float = 0.1}@

    # Compute the difference in orientation using quaternion distance
    quat_diff = quat_mul(object_rot, quat_conjugate(goal_rot))
    orientation_error = torch.norm(quat_diff[:, @\textcolor{black}{1:4}@], dim=@\textcolor{black}{1}@)  # Ignore the scalar part for distance
    orientation_reward = torch.exp(-orientation_temp * orientation_error)

    # Penalize excessive angular velocity
    angular_vel_penalty = torch.exp(-angular_vel_temp * angular_vel_magnitude)

    # Combine rewards
    total_reward = orientation_reward * angular_vel_penalty

    # Return the total reward and individual components for debugging
    reward_components = {
        "orientation_reward": @\textcolor{red}{orientation\_reward}@,
        "angular_vel_penalty": @\textcolor{red}{angular\_vel\_penalty}@
    }
    return total_reward, reward_components

########################################################################################################

Iteration @\textcolor{black}{1}@ (sample #14):
def @\textcolor{black}{compute\_reward}@(object_rot: torch.Tensor, goal_rot: torch.Tensor, object_angvel: torch.Tensor) -> Tuple[torch.Tensor, Dict[str, torch.Tensor]]:
    # Temperature parameters for reward components
    @\textcolor{blue}{orientation\_temp = 1.0}@
    @\textcolor{blue}{angular\_velocity\_temp = 0.1}@

    # Compute the difference in orientation using quaternion distance
    quat_diff = quat_mul(object_rot, quat_conjugate(goal_rot))
    orientation_error = torch.norm(quat_diff[:, @\textcolor{black}{1:4}@], dim=@\textcolor{black}{1}@)  # Ignore the scalar part for distance
    orientation_reward = torch.exp(-orientation_temp * orientation_error)

    # Penalize excessive angular velocity
    angular_velocity_magnitude = torch.norm(object_angvel, dim=@\textcolor{black}{1}@)
    angular_velocity_penalty = torch.exp(-angular_velocity_temp * angular_velocity_magnitude)

    # Combine rewards
    total_reward = orientation_reward * angular_velocity_penalty

    # Return the total reward and individual components for debugging
    reward_components = {
        "orientation_reward": @\textcolor{red}{orientation\_reward}@,
        "angular_velocity_penalty": @\textcolor{red}{angular\_velocity\_penalty}@
    }
    return total_reward, reward_components
\end{minted}

\section{Detailed Results}
\label{app:detailed_desults}

\subsection{Evaluation on Efficiency}

Table~\ref{tab:sota_all_vertical_nohuman} presents the comprehensive evaluation results across all tasks in the three benchmarks. The comparative analysis demonstrates that while achieving comparable SR or HNS to Eureka, \method requires fewer simulation training episodes and LLM invocations in 92\% of the experimental tasks, indicating superior sample efficiency and computational economy. 

\begin{table}[htbp]
  \caption{\method vs. SOTA with efficiency. \method performed best in 92\% of tasks in terms of NOE and NLC (bolded parts in the table).}
  \label{tab:sota_all_vertical_nohuman}
  \centering
  \scriptsize
  \setlength{\tabcolsep}{3pt}
  \begin{tabular}{llccccccccc}
    \toprule
    Benchmark & Environment 
    & \multicolumn{3}{c}{Text2Reward}
    & \multicolumn{3}{c}{Eureka} 
    & \multicolumn{3}{c}{\method} \\
    \cmidrule(lr){3-5} \cmidrule(lr){6-8} \cmidrule(lr){9-11}
    & & HNS & NOE$\downarrow$ & NLC$\downarrow$
      & HNS & NOE$\downarrow$ & NLC$\downarrow$
      & HNS & NOE$\downarrow$ & NLC$\downarrow$ \\
    \midrule
    \multirow{9}{*}{Isaac} 
    & Ant             & 1.543 & 112 & 7 & 1.527   & 112 & 7 & 1.556    & \textbf{48} & \textbf{3} \\
    & Cartpole        & 1 & 16 & 1 & 1     & 16 & 1 & 1      & \textbf{15} & \textbf{1} \\
    & BallBalance     & 1 & 16 & 1 & 1     & 16 & 1 & 1      & \textbf{16} & \textbf{1} \\
    & Quadcopter      & 1.678 & 82 & 6 & 1.667  & 70 & 5 & 1.818  & \textbf{41} & \textbf{2} \\
    & FrankaCabinet   & 16.95 & 95 & 7 & 17    & 97 & 7 & 17.130     & \textbf{57} & \textbf{4} \\
    & Humanoid        & 2.305 & 16 & 1 & 2.306   & 16 & 1 & 2.646     & 52 & 3 \\
    & Anymal          & 1.095 & 87 & 6 & 1.113  & 91 & 6 & 1.2    & \textbf{47} & \textbf{2} \\
    & AllegroHand     & 2.176 & 121 & 8 & 2.162 & 95 & 6 & 2.182  & \textbf{40} & \textbf{4} \\
    & ShadowHand      & 1.805 & 111 & 8 & 1.786   & 105 & 7 & 1.817   & \textbf{33} & \textbf{2} \\
    \midrule
    &  & SR  &NOE$\downarrow$ &NLC$\downarrow$ & SR &NOE$\downarrow$ &NLC$\downarrow$ & SR &NOE$\downarrow$ &NLC$\downarrow$ \\
    \midrule
    \multirow{20}{*}{Dexterity}
    & BlockStack        & 0.67 & 112 & 7 & 0.67   & 112 & 7 & 0.68   & \textbf{53} & \textbf{1} \\
    & HandKettle            & 0.89 & 85 & 6 & 0.89   & 72 & 5 & 0.89    & \textbf{78} & \textbf{5} \\
    & HandDoorCloseOutward & 0.96 & 83 & 6 & 0.9   & 72 & 5 & 0.97   & \textbf{37} & \textbf{2} \\
    & DoorCloseInward   & 1 & 71 & 5 & 1.0    & 78 & 5 & 1.0     & \textbf{34} & \textbf{2} \\
    & SwingCup          & 0.84 & 93 & 7 & 0.84   & 87 & 6 & 0.84    & \textbf{51} & \textbf{3} \\
    & Switch            & 0 & 58 & 5 & 0.0    & 58 & 5 & 0.02    & 76 & \textbf{5} \\
    & TwoCatchUnderarm  & 0 & 62 & 5 & 0.0    & 62 & 5 & 0.0     & \textbf{62} & \textbf{4} \\
    & CatchUnderarm     & 0.72 & 95 & 7 & 0.73   & 89 & 6 & 0.73    & \textbf{67} & \textbf{4} \\
    & CatchAbreast      & 0.66 & 88 & 6 & 0.66   & 83 & 6 & 0.67    & \textbf{54} & \textbf{4} \\
    & DoorOpenInward    & 0.04 & 73 & 5 & 0.04   & 69 & 5 & 0.06    & \textbf{67} & \textbf{4} \\
    & PushBlock         & 0.14 & 97 & 7 & 0.14   & 92 & 6 & 0.15    & \textbf{49} & \textbf{3} \\
    & BottleCap         & 0.88 & 110 & 8 & 0.88   & 96 & 7 & 0.89    & \textbf{67} & \textbf{4} \\
    & ReOrientation     & 0.32 & 75 & 6 & 0.31   & 66 & 5 & 0.33    & \textbf{58} & \textbf{4} \\
    & CatchOver2Underarm& 0.91 & 77 & 6 & 0.9    & 74 & 5 & 0.93    & \textbf{57} & \textbf{3} \\
    & LiftUnderarm      &0.89  & 88 & 6 & 0.89   & 86 & 6 & 0.89    & \textbf{78} & \textbf{4} \\
    & Over              & 0.92 & 82 & 6 & 0.92   & 71 & 5 & 0.92    & \textbf{41} & \textbf{3} \\
    & Pen               & 0.85 & 95 & 7 & 0.85   & 97 & 7 & 0.85    & \textbf{78} & \textbf{4} \\
    & DoorOpenOutward   & 1 & 76 & 5 & 1.0    & 75 & 5 & 1.0     & \textbf{46} & \textbf{3} \\
    & Scissors          & 1 & 76 & 5 & 1.0    & 77 & 5 & 1     & \textbf{44} & \textbf{3} \\
    & GraspAndPlace     & 0.75 & 93 & 6 & 0.75   & 85 & 6 & 0.77    & \textbf{59} & \textbf{4} \\
    \midrule
    &  & SR  &NOE$\downarrow$ &NLC$\downarrow$ & SR &NOE$\downarrow$ &NLC$\downarrow$ & SR &NOE$\downarrow$ &NLC$\downarrow$ \\
    \midrule
    \multirow{6}{*}{ManiSkill2}
    & LiftCube     & 0.906 & 112 & 7 & 0.905   & 96 & 6 & 0.906   & \textbf{15} & \textbf{1} \\
    & PickCube     & 0.879 & 128 & 8 & 0.884   & 112 & 7 & 0.885  & \textbf{20} & \textbf{1} \\
    & TurnFaucet   & 0.799 & 96 & 6 & 0.800   & 96 & 6 & 0.801   & \textbf{34} & \textbf{2} \\
    & OpenCabinetDoor   & 0.865 & 96 & 6 & 0.861   & 96 & 6 & 0.866   & \textbf{31} & \textbf{2} \\
    & OpenCabinetDrawer & 0.633 & 112 & 7 & 0.632    & 96 & 6 & 0.638    & \textbf{43} & \textbf{2} \\
    & PushChair         & 0.657 & 96 & 6 & 0.654   & 96 & 6 & 0.657   & \textbf{51} & \textbf{2} \\
    \bottomrule
  \end{tabular}
\end{table}

\subsection{Evaluation on the Performance of the Reward Function}

Table~\ref{tab:sota_all_vertical} presents a comprehensive performance comparison of different methods across all tasks in three benchmarks. The results demonstrate that \method consistently outperforms the baseline approaches while maintaining comparable or reduced requirements for both simulation training episodes and LLM invocations. Notably, \method achieves superior performance to human-designed reward functions in 89\% of the experimental tasks, highlighting the significant potential of automated reward design methodologies. 

\begin{table}[htbp]
  \scriptsize
  \caption{ \textcolor{black}{Task-wise comparison of \method with other methods. \method outperforms the compared methods on 89\% of the tasks.} } ‌
  \label{tab:sota_all_vertical}
  \centering
  \setlength{\tabcolsep}{3pt}
  \begin{tabular}{llcccccccccccccc}
    \toprule
    Benchmark & Environment 
    & Sparse & Human 
    & \multicolumn{3}{c}{Text2Reward}
    & \multicolumn{3}{c}{Eureka}  
    & \multicolumn{3}{c}{\method} \\
    \cmidrule(lr){3-3} \cmidrule(lr){4-4} 
    \cmidrule(lr){5-7} \cmidrule(lr){8-10} \cmidrule(lr){11-13}
    & & HNS$\uparrow$ & HNS$\uparrow$ 
    & HNS$\uparrow$ & NOE & NLC 
    & HNS$\uparrow$ & NOE & NLC 
    & HNS$\uparrow$ & NOE & NLC \\
    \midrule
    \multirow{9}{*}{Isaac} 
    & Ant           & 0 & 1     & 0.772  & 48 & 3   & 0.828  & 48 & 3   & \textbf{1.556}   & 48 & 3 \\
    & Cartpole      & 0 & 1     & 1      & 15 & 1   & 1      & 15 & 1   & \textbf{1}       & 15 & 1 \\
    & BallBalance   & 0 & 1     & 1      & 16 & 1   & 1      & 16 & 1   & \textbf{1}       & 16 & 1 \\
    & Quadcopter    & 0 & 1     & 1.041   & 41 & 3   & 1.25   & 41 & 3   & \textbf{1.818}   & 41 & 2 \\
    & FrankaCabinet & 0 & 1     & 5.4    & 57 & 4   & 4.8    & 57 & 4   & \textbf{17.130}  & 57 & 4 \\
    & Humanoid      & 0 & 1     & 2.217  & 52 & 4   & 2.306  & 52 & 4   & \textbf{2.646}   & 52 & 3 \\
    & Anymal        & 0 & 1     & 0.317  & 47 & 3   & 0.545  & 47 & 3   & \textbf{1.2}     & 47 & 2 \\
    & AllegroHand   & 0 & 1     & 1.196  & 40 & 3   & 1.594  & 40 & 3   & \textbf{2.182}   & 40 & 4 \\
    & ShadowHand    & 0 & 1     & 1.034  & 33 & 3   & 1.115  & 33 & 3   & \textbf{1.817}   & 33 & 2 \\
    \midrule
    & & SR$\uparrow$ & SR$\uparrow$ 
    & SR$\uparrow$ & NOE & NLC 
    & SR$\uparrow$ & NOE & NLC 
    & SR$\uparrow$ & NOE & NLC \\
    \midrule
    \multirow{20}{*}{Dexterity}
    & BlockStack        & 0 & 0.69  & 0.11   & 53 & 4 & 0.12   & 53 & 4 & 0.679  & 53 & 1 \\
    & Kettle            & 0 & 0.02  & 0.89   & 78 & 5 & 0.89   & 78 & 5 & \textbf{0.89}    & 72 & 5 \\
    & DoorCloseOutward  & 0.15 & 0.06  & 0.57   & 37 & 3 & 0.64   & 37 & 3 & \textbf{0.968}   & 37 & 2 \\
    & DoorCloseInward   & 0 & 1     & 0.74   & 34 & 3 & 0.83   & 34 & 3 & \textbf{1}       & 34 & 2 \\
    & SwingCup          & 0 & 0     & 0.62   & 51 & 4 & 0.53   & 51 & 4 & \textbf{0.84}    & 51 & 3 \\
    & Switch            & 0 & 0     & 0   & 76 & 5 & 0.01   & 76 & 5 & \textbf{0.02}    & 76 & 5 \\
    & TwoCatchUnderarm  & 0 & 0     & 0      & 62 & 5 & 0      & 62 & 5 & \textbf{0}       & 62 & 4 \\
    & CatchUnderarm     & 0 & 0.51  & 0.58   & 67 & 5 & 0.63   & 67 & 5 & \textbf{0.73}    & 67 & 4 \\
    & CatchAbreast      & 0 & 0.37  & 0.27   & 54 & 5 & 0.34   & 54 & 5 & \textbf{0.67}    & 54 & 4 \\
    & DoorOpenInward    & 0 & 0.03  & 0      & 67 & 5 & 0      & 67 & 5 & \textbf{0.06}    & 67 & 4 \\
    & PushBlock         & 0 & 0.01  & 0.05   & 49 & 4 & 0.05   & 49 & 4 & \textbf{0.15}    & 49 & 3 \\
    & BottleCap         & 0.91 & 0.91  & 0.21   & 67 & 5 & 0.25   & 67 & 5 & 0.89    & 67 & 4 \\
    & ReOrientation     & 0.01 & 0.02  & 0.25   & 58 & 4 & 0.28   & 58 & 4 & \textbf{0.33}    & 58 & 4 \\
    & CatchOver2Underarm& 0 & 0.87  & 0.81   & 57 & 4 & 0.81   & 57 & 4 & \textbf{0.93}    & 57 & 3 \\
    & LiftUnderarm      & 0 & 0.37  & 0.83   & 78 & 6 & 0.85   & 78 & 6 & \textbf{0.89}    & 78 & 4 \\
    & Over              & 0 & 0.9   & 0.54   & 41 & 4 & 0.61   & 41 & 4 & \textbf{0.92}    & 41 & 3 \\
    & Pen               & 0.01 & 0.74  & 0.67   & 78 & 6 & 0.63   & 78 & 6 & \textbf{0.85}    & 78 & 4 \\
    & DoorOpenOutward   & 0.02 & 0.85  & 0.76   & 46 & 3 & 0.87   & 46 & 3 & \textbf{1}       & 46 & 3 \\
    & Scissors          & 0.99 & 0.96  & 0.73   & 44 & 3 & 0.69   & 44 & 3 & \textbf{1}     & 44 & 3 \\
    & GraspAndPlace     & 0 & 0.87  & 0.41   & 59 & 4 & 0.43   & 59 & 4 & 0.77    & 59 & 4 \\
    \midrule
    & & SR$\uparrow$ & SR$\uparrow$ 
    & SR$\uparrow$ & NOE & NLC 
    & SR$\uparrow$ & NOE & NLC 
    & SR$\uparrow$ & NOE & NLC \\
    \midrule
    \multirow{6}{*}{ManiSkill2}
    & LiftCube          & 0.143 & 0.543  & 0.531  & 15 & 1 & 0.356  & 15 & 1 & \textbf{0.906}   & 15 & 1 \\
    & PickCube          & 0.131 & 0.479  & 0.497  & 20 & 2 & 0.434  & 20 & 2 & \textbf{0.885}   & 20 & 1 \\
    & TurnFaucet        & 0 & 0.598  & 0.631  & 34 & 3 & 0.516  & 34 & 3 & \textbf{0.801}   & 34 & 2 \\
    & OpenCabinetDoor   & 0.028 & 0.651  & 0.713  & 31 & 2 & 0.575  & 31 & 2 & \textbf{0.866}   & 31 & 2 \\
    & OpenCabinetDrawer & 0 & 0.37   & 0.519  & 43 & 3 & 0.478  & 43 & 3 & \textbf{0.64}    & 43 & 2 \\
    & PushChair         & 0 & 0.334  & 0.432  & 51 & 4 & 0.336  & 51 & 4 & \textbf{0.657}   & 51 & 2 \\
    \bottomrule
  \end{tabular}
\end{table}

\section{Proofs}
\label{app:proofs}

\subsection{Determination of the Kernel Function}

\begin{theorem}
    The kernel function \eqref{kernel-new} satisfied the properties of symmetry and positive semi-definiteness.
\end{theorem}
\begin{proof}
    The property of symmetry is obvious. Now we prove the positive semi-definiteness based on the properties of Matern kernel \eqref{matern2.5}. For any given finite set of sample points $\Tilde{p}^{(1)},\Tilde{p}^{(2)},\cdots,\Tilde{p}^{(n)}$, we denote the corresponding kernel matrix as 
    \begin{equation}
        \Tilde{K}_{ij} = \Tilde{k}(\Tilde{p}^{(i)},\Tilde{p}^{(j)}).
    \end{equation}
    By performing coordinate scaling transformation on the sample points, we obtain new sample points 
    \begin{equation}
        p^{(i)} = (\Tilde{p}^{(i)}_1/l_1, \cdots, \Tilde{p}^{(i)}_d/l_d),\: i = 1,2,\cdots,n.
    \end{equation}
    And the Matern kernel matrix is 
    \begin{equation}\label{mtr-martern}
        K_{ij} = k(p^{(i)},p^{(j)})
    \end{equation}
    Since Matern kernel \eqref{matern2.5} is positive semi-definite, the kernel matrix \eqref{mtr-martern} constructed from the transformed points with Matern kernel is positive semi-definite. Furthermore, the kernel matrix of \eqref{kernel-new} is essentially equivalent to the Martern kernel matrix \eqref{mtr-martern} computed on the transformed sample points, i.e.
    \begin{align}
        & r_{\text{new}}(\Tilde{p}^{(i)},\Tilde{p}^{(j)}) = r(p^{(i)},p^{(j)})\\
        &\Tilde{K}_{ij} 
        = \Tilde{k}(\Tilde{p}^{(i)},\Tilde{p}^{(j)})
        = f_{\nu}(r_{\text{new}}) = f_{\nu}(r) = K_{ij}.
    \end{align}
    Therefore, $\Tilde{K}$ is positive semi-definite.
\end{proof}
In the newly defined weighted distance \eqref{dist-new}, a larger $l_i$ in directions with more rapid variations can increase the possibility of exploration, while a smaller $l_i$ in directions with smoother variations will reduce exploration and emphasize the exploitation of information from previously sampled points.

\subsection{Convergence analysis of the Uncertainty-accelerated Expected Improvement (uEI)}

The basic definitions and theorems have been defined to analyze the convergence rate of Bayesian optimization~\cite{bull2011convergence}. 
Here, we briefly restate some of the key definitions required. 
Let $\mathcal{X} \subset \mathbb{R}^d$ be a compact set with non-empty interior. For a function $f : \mathcal{X} \to \mathbb{R}$ to be minimized, $K_\theta$ is the correlation kernel for function $f$ prior distribution $\pi$ with length-scales $\theta$. $\mathcal{H}_\theta(\mathcal{X})$ is the reproducing-kernel Hilbert space of $K_\theta$ on $\mathcal{X}$. Let $\mathbb{P}^u_f$ and $\mathbb{E}^u_f$ denote the probability and expectation operators when minimizing the fixed function $f$ using strategy $u$. The loss suffered over the ball $B_R$ in
$\mathcal{H}_\theta(\mathcal{X})$ after $n$ steps by a strategy $u$ is defined as,
\begin{equation}
L_n(u, \mathcal{H}_\theta(\mathcal{X}), R) := \sup_{\substack{f \in \mathcal{H}_\theta(\mathcal{X}) \\ \|f\|_{\mathcal{H}_\theta(\mathcal{X})} \leq R}} \mathbb{E}^u_f \left[ f(x^*_n) - \min f \right]
\end{equation}
where $x^*_n$ is the estimated minimum of $f$. \par
It is proved that the strategy expected improvement converges at least at rate $n^{-(\nu \wedge 1)/d}$, up to logarithmic factors, where $\nu$ is the parameter in Matern kernel~\cite{bull2011convergence}.
\begin{theorem}
    Assume that the function $f$ depends only on $m$ input variables, $m<d$, and remains constant along the other $d-m$ directions. Under such an assumption, with an appropriate choice of weighted parameters, the Uncertainty-accelerated Expected Improvement converges at least at rate $n^{-(\nu \wedge 1)/m}$, up to logarithmic factors, where $\nu$ is the parameter in Matern kernel.
\end{theorem}
\begin{proof}
    From the proof of expected improvement convergence rate~\cite{bull2011convergence}, we observe that the parameter $d$ in convergence rate estimation is actually the dimensionality of the sampling space. And the conclusion holds based on the condition that $\left\{x_n \right\}$ is a quasi-uniform sequence in a region of interest~\cite{Narcowich2003}. Without loss of generality, let us assume that the function $f$ depends on dimensions $i_1$ to $i_m$, and is invariant with respect to dimensions $i_{m+1}$ to $i_d$. For the uEI strategy, let 
    \begin{equation}
        \lambda_j = \begin{cases}
            0,\:&j=1,\cdots,m\\
            \infty,\: &j=m+1,\cdots,d
        \end{cases}
    \end{equation}
    Consequently, any exploration in the directions of dimensions $i_{m+1}$ to $i_d$ will be discouraged. The effective dimensionality of the sampling space decreases from $d$ to $m$, which leads to an improved convergence rate of at least $n^{-(\nu \wedge 1)/m}$.
\end{proof}
Although our theorem has focused on the limiting case in which $f$ is entirely independent of certain directions, it illustrates how applying weighted constraints allows for dimension-specific treatment within the sampling space, thus enhancing the efficiency of the algorithm.

\section{Additional Analysis}
\label{app:additional_analysis}

\subsection{Uncertainty and Reward Shaping}
\label{app:reward_shaping}

\textcolor{black}{To validate the role of high-uncertainty reward components, we conducted ablation studies by removing these components from the reward function. Figure~\ref{fig:multiI} presents comparative cases between the original reward functions ($R$) and its ablated counterparts ($R$ w.o. $r_{u\uparrow}$). Our analysis reveals two key findings: (1) the removal of high-uncertainty components leads to significant performance degradation, with respective decreases of 19\%, 83\%, and 51\% in HNS/SR metrics; and (2) reward functions retaining these components demonstrate accelerated discovery of critical states during early RL training phases, effectively reducing inefficient exploration. These results collectively demonstrate the crucial function of high-uncertainty components in both final performance and training efficiency. These results suggest that high-uncertainty reward components contribute positively to reward shaping and play an essential role in guiding effective policy learning. } 

\begin{figure}[htbp]
  \centering
  \captionsetup[subfigure]{aboveskip=0pt, belowskip=-6pt}
  
  \begin{subfigure}[b]{0.31\textwidth}
    \includegraphics[width=\linewidth]{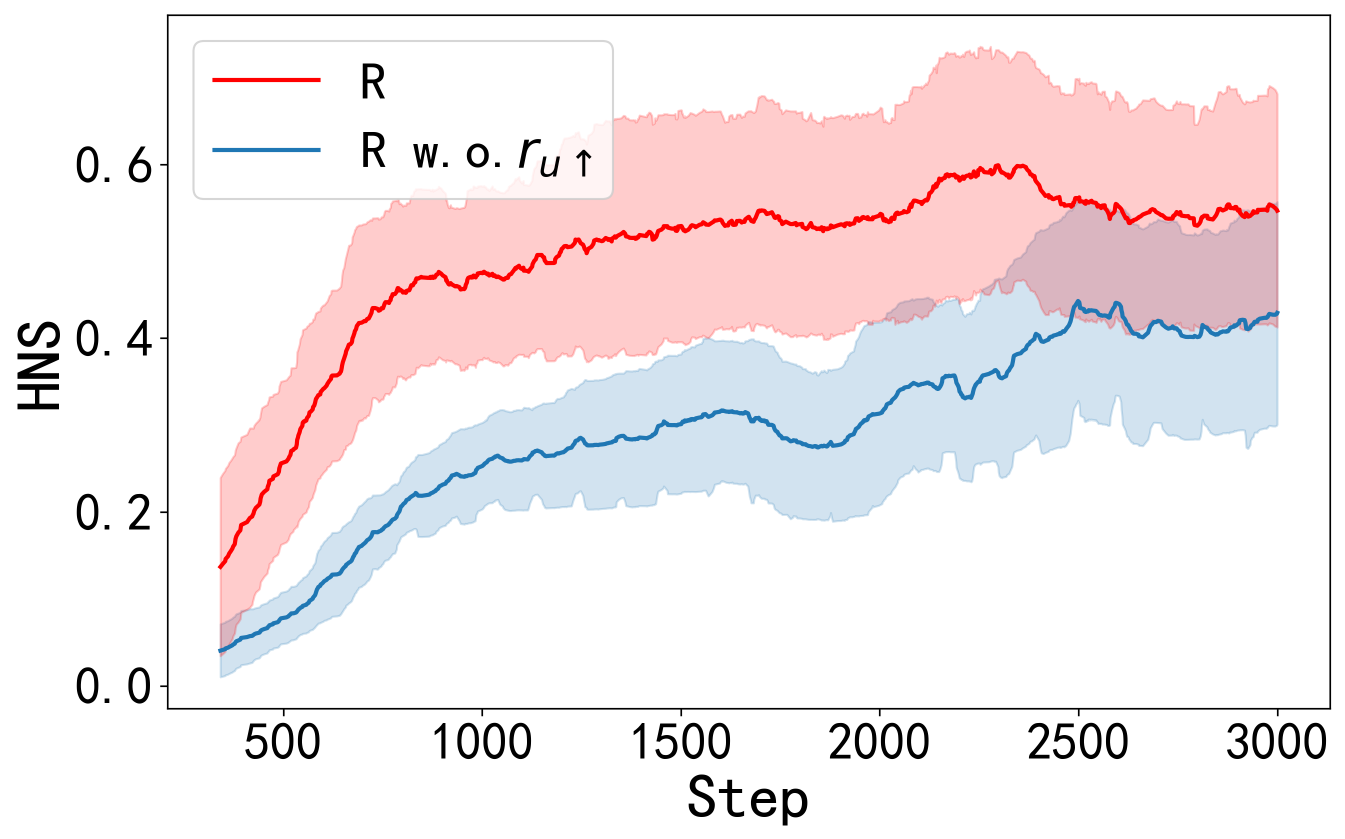}
    \caption{Anymal}
    \label{fig:subI1}
  \end{subfigure}
  \hspace{0.01\textwidth}
  \begin{subfigure}[b]{0.31\textwidth}
    \includegraphics[width=\linewidth]{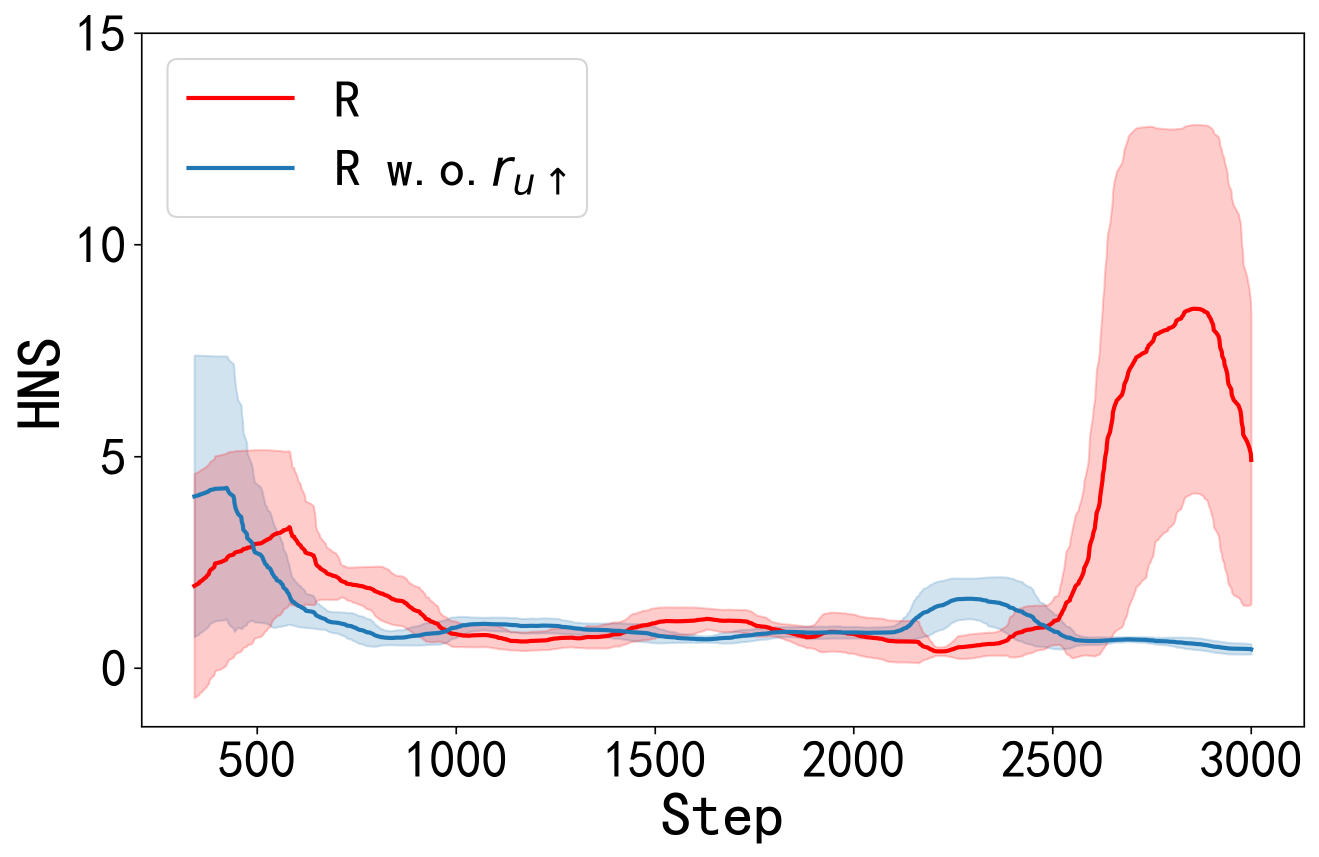}
    \caption{FrankaCabinet}
    \label{fig:subI2}
  \end{subfigure}
  \hspace{0.01\textwidth}
  \begin{subfigure}[b]{0.31\textwidth}
    \includegraphics[width=\linewidth]{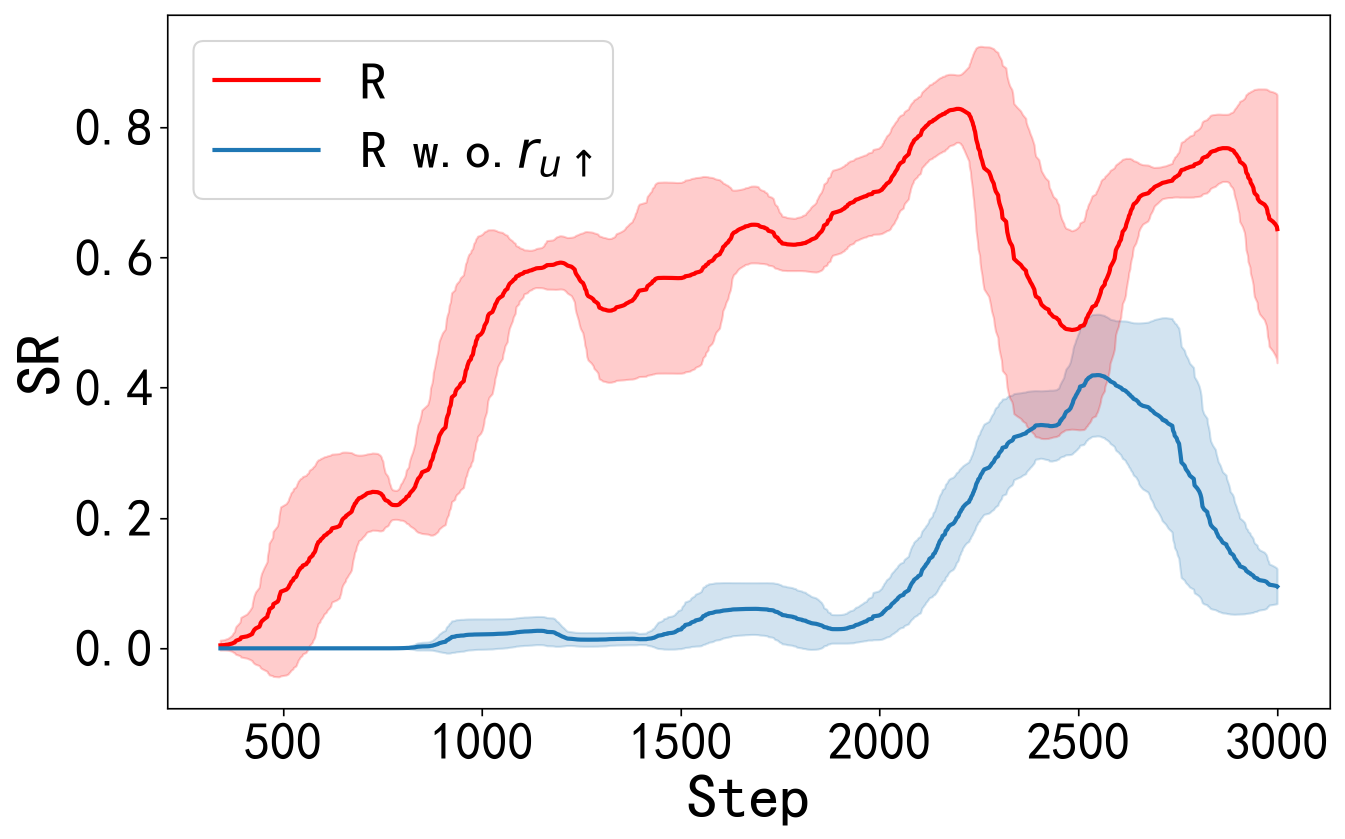}
    \caption{LiftUnderarm}
    \label{fig:subI3}
  \end{subfigure}
  \caption{\textcolor{black}{The comparison between $R$ and $R$ w.o.$r_{u\uparrow}$ suggests that the high-uncertainty reward components ($r_{u\uparrow}$) contributes to reward shaping during the policy learning.} }  
  \label{fig:multiI}
\end{figure}

\subsection{LLM Alternatives}
\label{app:llm}

\textbf{\method with Qwen2.5}. In Fig.~\ref{fig:urdp_vs_qwen}, we compare the performance of \method with DeepSeek-v3-241226 (the results reported in the paper) and \method with Qwen2.5 (qwen-max-0919)~\cite{qwen2025qwen25technicalreport}. These results demonstrate the consistency of the effect of \method on different LLMs and eliminate concerns that the differences in the capabilities of LLMs themselves may affect the results.  

\begin{figure}[htbp]
  \centering
  \begin{subfigure}[b]{\linewidth}
    \centering
    \includegraphics[width=\linewidth]{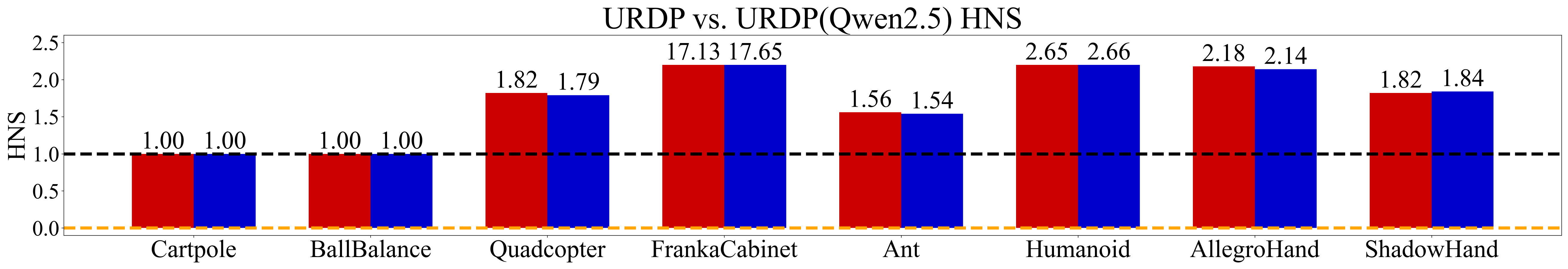}
  \end{subfigure}

  \vspace{-3pt} 
  
  \begin{subfigure}[b]{0.5\linewidth}
    \centering
    \includegraphics[width=\linewidth]{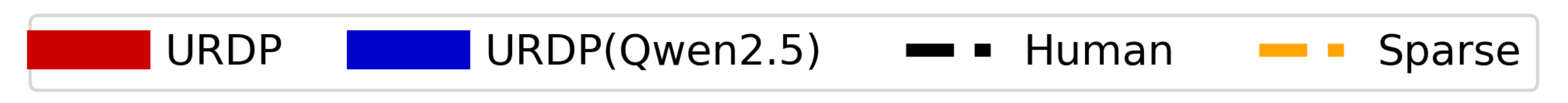}
  \end{subfigure}
  
  \caption{\textcolor{black}{\method demonstrates consistent performance across different LLMs.}  }
  \label{fig:urdp_vs_qwen}
\end{figure}

\section{Limitation and Discussion}

\textcolor{black}{
In this work, we investigate efficient automated reward design methodologies based on large language models (LLMs). However, constrained by inherent limitations of LLMs in spatial reasoning capabilities, our approach, like other comparable methods, faces challenges in addressing scenario-specific constraints during reward formulation. A representative case emerges in ``grasping'' tasks where environmental obstacles may restrict robotic manipulation paths, constraints that should ideally be reflected in reward design. While providing detailed environmental descriptions in prompts may partially mitigate this issue, a more fundamental solution would involve integrating video-language models (VLMs) into the reward design framework. VLMs demonstrate superior spatial perception capabilities that could enrich the understanding of RL task objectives, environmental constraints, and reward composition. Nevertheless, incorporating VLMs introduces new challenges regarding computational scalability during reward design and tuning processes. We therefore identify this as a critical yet underexplored research direction worthy of systematic investigation.}

\textcolor{black}{
To maintain simplicity in presenting our work, we employ the base capabilities of large language models without sophisticated inference-time enhancement techniques (e.g., chain-of-thought, test-time training). However, advanced reasoning techniques have demonstrated significant improvements in handling complex logical tasks, as evidenced in code generation and mathematical reasoning domains. We posit these methods would similarly enhance reward function code design. Ultimately, substantial exploration potential remains for large language model techniques in automated reward design.}

\section{Author contributions}
Miao Xin and Xiaolu Zhou contributed to the conception and design of the approach. Yang Yang and Bosong Ding implemented the code and performed the experiments. Miao Xin organized the entire research program and wrote the paper. All authors read and approved the submitted version.

\end{document}